\documentclass{article}
\PassOptionsToPackage{numbers, compress}{natbib}

\usepackage[main,preprint]{neurips_2026}

\usepackage[utf8]{inputenc} 
\usepackage[T1]{fontenc}    
\usepackage{hyperref}       
\usepackage{url}            
\usepackage{booktabs}       
\usepackage{amsfonts}       
\usepackage{nicefrac}       
\usepackage{microtype}      
\usepackage{xcolor}         
\usepackage{graphicx}
\usepackage{booktabs}
\usepackage{titletoc}
\usepackage{multirow}
\usepackage{subcaption}
\usepackage{wrapfig}
\usepackage{float}

\usepackage{algorithm}
\usepackage{algpseudocode}
\usepackage[normalem]{ulem}
\useunder{\uline}{\ul}{}

\usepackage[accsupp]{axessibility}  

\usepackage{amssymb}
\usepackage{pifont}
\newcommand{\cmark}{\ding{51}}%
\newcommand{\xmark}{\ding{55}}%

\usepackage{mathtools}

\DeclarePairedDelimiterX{\infdivx}[2]{(}{)}{%
  #1\;\delimsize\|\;#2%
}

\usepackage{xcolor}
\newcommand{\mcal}[1]{\mathcal{#1}}
\newcommand{\dom}[1]{\mathcal{D}_{\mathrm{#1}}}
\newcommand{\dblend}{\mathcal{D}_{\mathrm{blend}}}

\DeclareMathOperator{\dsim}{dsim}

\DeclareMathOperator{\acc}{acc}

\usepackage{cleveref}

\usepackage{amsmath, amssymb, amsthm}

\theoremstyle{plain}
\newtheorem{theorem}{Theorem}[section]
\newtheorem{proposition}[theorem]{Proposition}

\newtheorem{corollary}[theorem]{Corollary}

\theoremstyle{definition}

\newtheorem{assumption}[theorem]{Assumption}

\theoremstyle{remark}
\newtheorem{remark}[theorem]{Remark}

\begin{document}

\title{RADAR: Relative Angular Divergence Across Representations}

\author{
\textbf{Xavier Cadet}\thanks{Equal contribution}\\
 Dartmouth College\\
 {\tt\small xavier.fjf.cadet@dartmouth.edu}
 \and
 \textbf{Mateusz Nowak}$^*$\\
 Dartmouth College\\
 {\tt\small mateusz.m.nowak.th@dartmouth.edu}
 \and \textbf{Peter Chin}\\
 Dartmouth College\\
 {\tt\small peter.chin@dartmouth.edu}
}

\maketitle

\begin{abstract}
Machine learning methods rely on data. However, gathering suitable data can be challenging due to availability constraints, cost, or the need for domain expertise. Expanding datasets with additional sources is a common response to limited data, yet this practice does not always improve downstream performance and can sometimes lead to a loss of performance, known as negative transfer. We propose RADAR, a simple, geometrically grounded metric for estimating cross-domain transferability in foundation models. RADAR analyzes the layer-wise evolution of representations by measuring angular alignments and relative changes in distance along layer-to-layer displacement trajectories, and by comparing empirical distributions of within-domain and cross-domain dynamics. We hypothesize that domain transferability is related to the divergence between these trajectory distributions. We evaluate the metric across multiple modalities, including cross-lingual sentiment classification with text embedding models and cross-domain image classification with foundation vision models. Across several settings, RADAR provides competitive predictive performance relative to existing transferability metrics on several vision and text benchmarks, with particularly strong results when domain transitions are smooth or cleanly separated. Our ablations further suggest that the effectiveness of transferability estimation depends on the geometry of the model's internal representation space, with different modalities favoring different topological formulations.
\end{abstract}

\section{Introduction}
\label{section:introduction}

In modern machine learning, expanding training datasets with auxiliary sources is a standard strategy to address data constraints. However, this practice does not guarantee improved downstream performance and can even lead to negative transfer. Identifying which source domains will yield positive transfer remains a difficult task, as empirically evaluating each candidate impact requires costly full-model retraining. To address this challenge, we propose RADAR (\textbf{R}elative \textbf{A}ngular \textbf{D}ivergence \textbf{A}cross \textbf{R}epresentations), a simple, geometrically grounded metric designed to efficiently estimate cross-domain transferability in foundation models using only frozen features.

Unlike baseline metrics that require pre-training on the target dataset or evaluate only the static output of a single layer, RADAR analyzes the dynamic, layer-wise structural evolution of representations using frozen foundation models. We provide a high-level overview of the metric in Figure~\ref{fig:radar_pipeline}. The framework measures divergence by extracting intermediate feature representations from the model to map layer-to-layer displacement trajectories. These trajectories are then characterized by computing relative distance changes and angles between samples within the same domain versus samples across different domains. Based on these angles and relative distances, we fit high-dimensional Gaussian Mixture Models (GMMs) to model the continuous probability distributions of the within-domain and cross-domain dynamics. Finally, given these two distributions, we measure the overall structural shift between domains by computing the symmetric Kullback-Leibler (KL) divergence between them.

\begin{figure}[ht!]
    \centering
    \includegraphics[width=\linewidth]{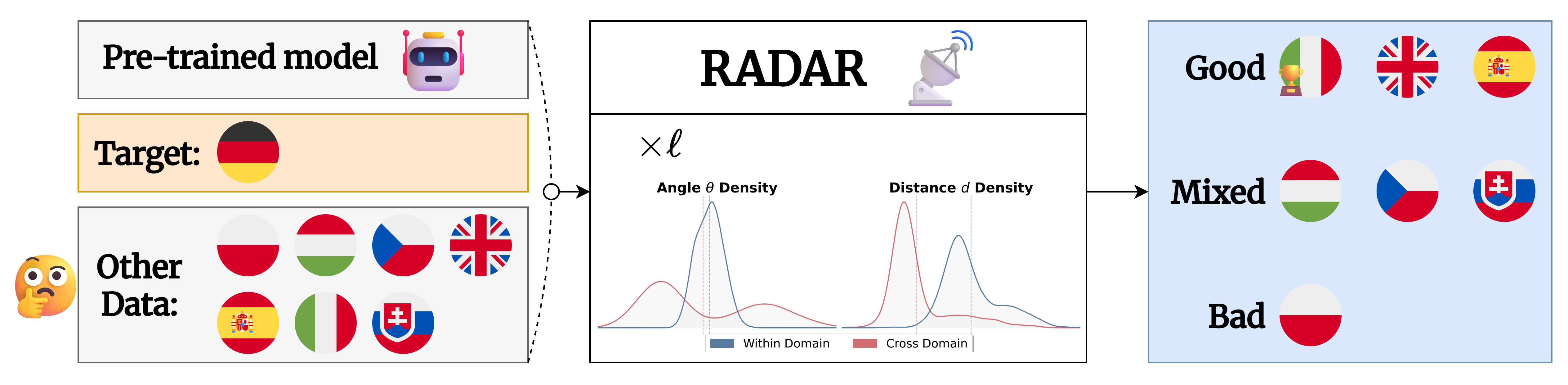}
    \caption{\textbf{Overview of the RADAR framework.} RADAR identifies optimal source data for transfer learning by extracting layer-wise ($\ell$) geometric trajectories from pre-trained models. By computing the divergence in angle ($\theta$) and relative distance ($d$) densities between within- and cross-domain pairs, it efficiently ranks auxiliary candidates without requiring computationally expensive fine-tuning.}
    \label{fig:radar_pipeline}
\end{figure}

We evaluate RADAR across both vision modalities (using CLIP \cite{radford2021learning} and DINOv3 \cite{simoni2025dinov3} architectures) and text modalities (using Qwen3-Embedding \cite{zhang2025qwen3} and EmbeddingGemma \cite{vera2025embeddinggemma}) on benchmark datasets such as DomainNet \cite{peng2019moment}, OfficeHome \cite{venkateswara2017deep}, PACS \cite{li2017deeperbroaderartierdomain} for the vision modality, AmazonReview \cite{nielsen2025encoder}, and EuroEval \cite{nielsen2025encoder} for the text modality, comparing it against state-of-the-art transferability metrics. We show that RADAR achieves top-three performance in seven out of ten configurations, securing the absolute best performance in six of them. Furthermore, we thoroughly evaluate the design choices of our algorithm—specifically focusing on domain subsampling, density estimation, and alternative optimal transport algorithms—demonstrating the overall robustness of our method. Our core contributions and findings are summarized as follows:

\begin{itemize}
    \item \textbf{Dynamic, Layer-Wise Structural Evolution:} We propose RADAR---a geometrically grounded metric that evaluates dynamic representation shifts between different domains. By measuring the divergence between these dynamic, continuous distributions, we effectively capture both smooth domain transitions and distinctly separated domains.
    \item \textbf{Highly Competitive Predictive Performance:} RADAR consistently ranks as the best or second-best transferability metric against established baselines (such as LEEP \cite{nguyen2020leep}, LogME \cite{you2021logme}, NCE \cite{tran2019transferability}, TransRate \cite{huang2022frustratingly}, and S-OTDD \cite{nguyen2025lightspeed}) on text and vision modalities.
    \item \textbf{Robust and Efficient Estimation:} We demonstrate the robustness and high computational efficiency of our method through an extensive ablation study, focusing on hyperparameter choices, density estimation techniques, representation subsampling, and alternative optimal transport algorithms.
\end{itemize}

Ultimately, RADAR offers a competitive, cross-modal transferability metric that reliably predicts transfer gain rankings across diverse vision and text benchmarks --- particularly when domain transitions are smooth or cleanly separated --- while remaining computationally efficient relative to empirical- and optimal transport-based metrics.

\section{Problem Statement}
\label{sec:problem}

\paragraph{Domains and shared label space.}
Let $\mcal{D} = \{\dom{1}, \ldots, \dom{K}\}$ be a collection of $K$ domains, each containing labeled data $\dom{i} = \{(x_j, y_j)\}_{j=1}^{N_i}$ drawn from a domain-specific joint distribution $P_i(X, Y)$. We assume all domains share a common label space $\mcal{Y}$, with every class represented in every domain. One domain $\dom{\mathrm{tgt}} \in \mcal{D}$ is designated the \emph{target}; the remaining $K-1$ domains form the pool of \emph{source} candidates, $\mcal{D}_{\mathrm{src}} = \mcal{D} \setminus \{\dom{\mathrm{tgt}}\}$. For instance, assuming we have three domains \{real, paintings, sketches\} dataset, we would like to know whether we should use paintings, and sketches to improve our performance on real images.

\paragraph{Blended training set.}
A \emph{blend} is a subset $\mcal{S} \subseteq \mcal{D}_{\mathrm{src}}$;
the corresponding blended dataset is the union of its member domains,
\begin{equation}
  \dblend\left(\mcal{S}\right) = \bigcup_{\dom{i} \in \mcal{S}} \dom{i}.
  \label{eq:blend}
\end{equation}
There are $2^{K-1}$ possible blends. For instance, if our target domain is real, should we train with \{\{\}, \{paintings\}, \{sketches\}, \{sketches, paintings\}\}.
Given a blend $\dblend\left(\mcal{S}\right)$, we train a classifier on the union of the target data and the source blend, $\dom{\mathrm{tgt}} \cup \dblend(\mcal{S})$, and evaluate it on $\dom{\mathrm{tgt}}$. We can now naturally evaluate the target-only baseline as $\mcal{S} = \varnothing$, yielding $A(\varnothing) = \acc\bigl(f_{\dom{\mathrm{tgt}}}, \dom{\mathrm{tgt}}\bigr)$.
We record two quantities as ground truth:
\begin{align}
  A(\mcal{S}) &= \acc\bigl(f_{\dom{\mathrm{tgt}} \cup \dblend\left(\mcal{S}\right)}, \dom{\mathrm{tgt}}\bigr),
  \label{eq:abs_acc} \\
  \Delta(\mcal{S}) &= A(\mcal{S}) - A(\varnothing).
  \label{eq:transfer_gain}
\end{align}
A positive $\Delta(\mcal{S})$ indicates that blending in the sources $\mcal{S}$ improves over training on target data alone; a negative value indicates negative transfer.

\paragraph{Source selection problem.}
Computing $A(\mcal{S})$ for every subset $\mcal{S}$ requires $2^{K-1}$ full
training runs, which is intractable for large $K$.
Our goal is to predict $\Delta(\mcal{S})$ --- and thereby rank blend configurations --- from \emph{frozen features alone}, without any model retraining. Concretely, we seek a \emph{blend similarity function}
\begin{equation}
  \dsim\bigl(\dblend(\mcal{S}), \dom{\mathrm{tgt}}\bigr) \in \mathbb{R}
  \label{eq:dsim_def}
\end{equation}
such that its ranking over all $2^{K-1}-1$ non-empty blends is consistent with
the ranking induced by $\Delta(\mcal{S})$:
\begin{equation}
  \dsim\!\bigl(\dblend(\mcal{S}),\, \dom{\mathrm{tgt}}\bigr) <
  \dsim\!\bigl(\dblend(\mcal{S}'),\, \dom{\mathrm{tgt}}\bigr)
  \implies
  \Delta(\mcal{S}) > \Delta(\mcal{S}').
  \label{eq:dsim_goal}
\end{equation}
We evaluate the quality of a similarity function by its Spearman rank correlation $\rho$ with the ground-truth $\Delta$ values across all blend configurations.

\section{Relative Angular Divergence Across Representations as a Metric}
\label{sec:method}

\paragraph{Feature extraction.}
Let $\Phi$ denote a pre-trained foundation model, such as a text embedding or vision architecture. For an input sample $x_i$ drawn from a specific domain $\dom{d} \in \mcal{D}$, we define its intermediate feature representation extracted at layer $l$ of $\Phi$ as $\mathbf{h}^{(l)}(x_i) \in \mathbb{R}^{H_l}$, where $H_l$ is the dimensionality of the representation at layer $l$.

\begin{figure}[ht!]
  \centering
  \subfloat[Within-domain dynamic]{\includegraphics[width=0.49\textwidth]{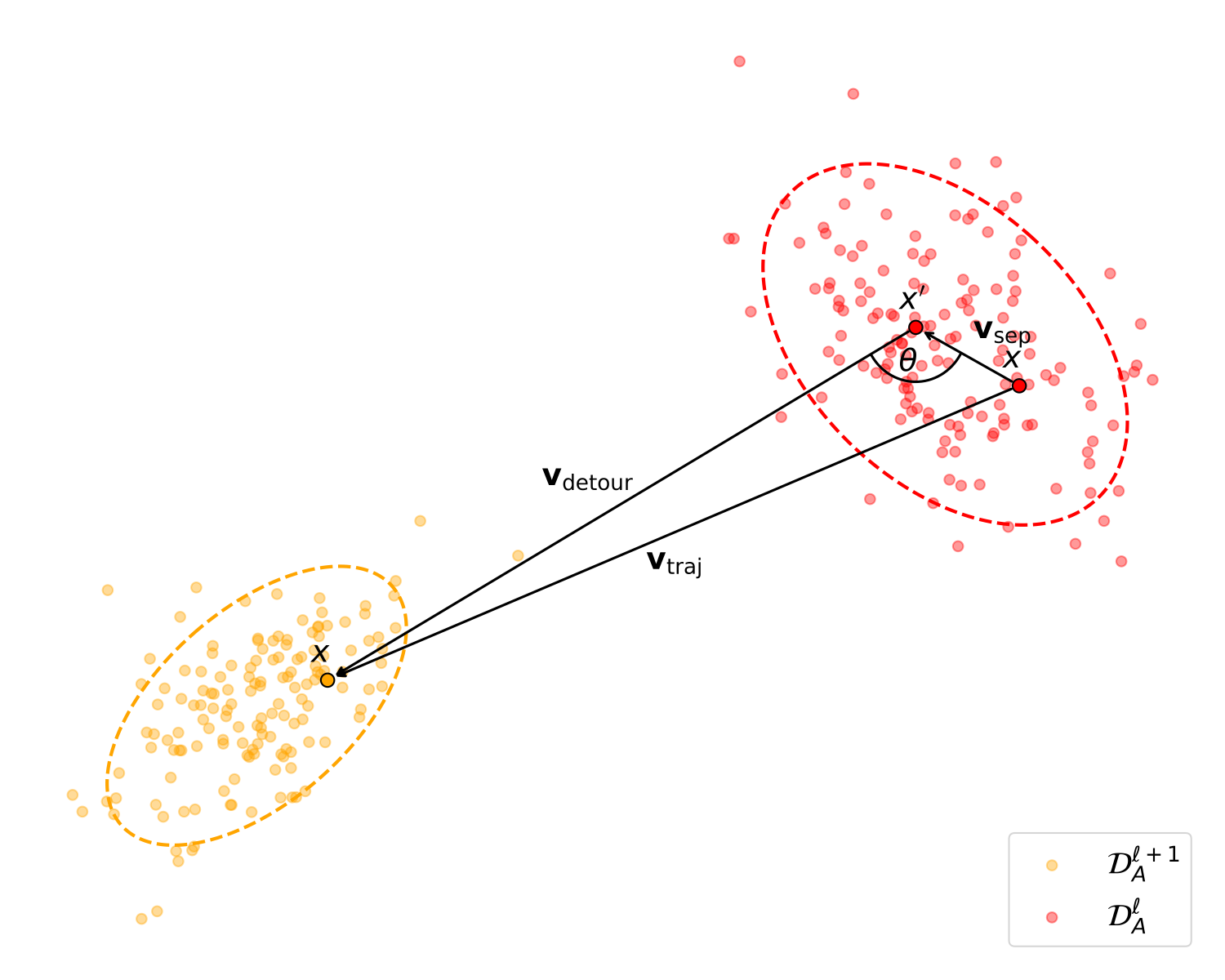}}\label{fig:traid_within}
  \hfill
  \subfloat[Across-domain dynamic]{\includegraphics[width=0.49\textwidth]{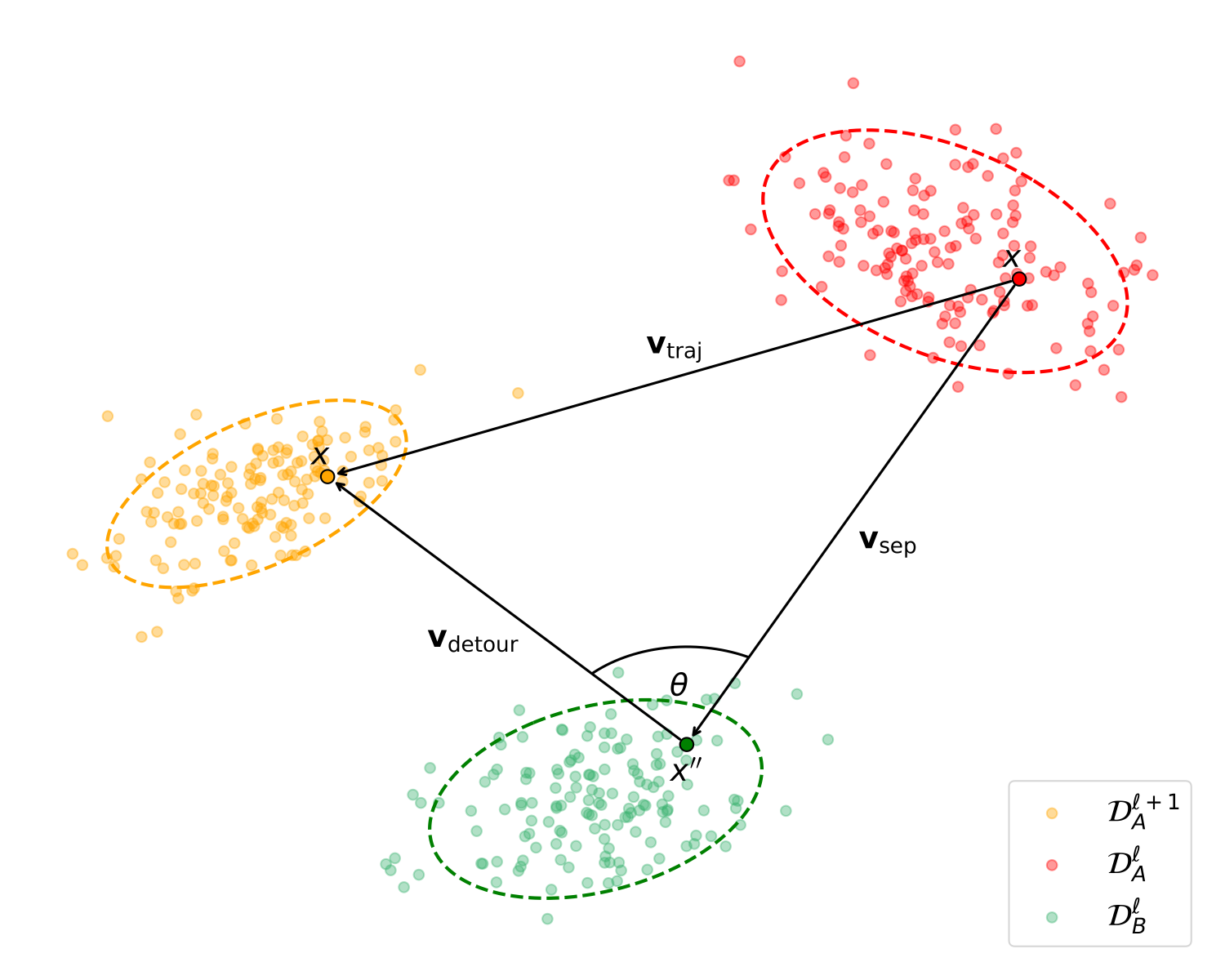}}\label{fig:triad_cross}
  \caption{Geometric triangle of displacement vectors spanning layers $l$ and $l+1$}
  \label{fig:triad}
\end{figure}

\paragraph{Angle and relative distance.}
Given two domains, $\dom{A}$ and $\dom{B}$, we compute empirical within-domain and cross-domain angle and distance distributions. For a given anchor sample $x \in \dom{A}$ with label $y$, we draw a secondary intra-domain sample $x' \in \dom{A} \setminus \{x\}$ and a cross-domain sample $x'' \in \dom{B}$. To capture within-domain dynamics, we define a geometric triangle of displacement vectors spanning layers $l$ and $l+1$ (see Figure~\ref{fig:triad}):
\begin{align}
\mathbf{v}_{\text{sep}}^{(l)}(x, x') &= \mathbf{h}^{(l)}(x') - \mathbf{h}^{(l)}(x), \label{eq:sep} \\
\mathbf{v}_{\text{detour}}^{(l)}(x, x') &= \mathbf{h}^{(l+1)}(x) - \mathbf{h}^{(l)}(x'), \label{eq:detour} \\
\mathbf{v}_{\text{traj}}^{(l)}(x) &= \mathbf{h}^{(l+1)}(x) - \mathbf{h}^{(l)}(x). \label{eq:traj}
\end{align}
 
This triad establishes the local spread --- measuring the spatial separation between samples at a single layer ($\mathbf{v}_{\text{sep}}$) and the detour required if we substituted $x$ with $x'$ before transitioning to the next layer ($\mathbf{v}_{\text{detour}}$). Analogously, we compute the corresponding cross-domain triad to capture inter-domain divergence: the cross-domain separation $\mathbf{v}_{\text{sep}}^{(l)}(x, x'')$, the cross-domain detour $\mathbf{v}_{\text{detour}}^{(l)}(x, x'')$, and the direct trajectory of the anchor sample $\mathbf{v}_{\text{traj}}^{(l)}(x)$.

To quantify the structural shift at layer $l$, we compute the angular alignment $\theta^{(l)}$ between the separation and detour vectors using cosine similarity, as well as the relative distance $d^{(l)}$. For the within-domain pair within Euclidean space, these are given by:

\begin{align}
    \theta^{(l)}(x, x') &= \arccos \left( \frac{\mathbf{v}_{\text{sep}}^{(l)}(x, x') \cdot \mathbf{v}_{\text{detour}}^{(l)}(x, x')}{\max(\|\mathbf{v}_{\text{sep}}^{(l)}(x, x')\| \|\mathbf{v}_{\text{detour}}^{(l)}(x, x')\|,\epsilon)} \right), \\
    d^{(l)}(x, x') &=  \frac{\|\mathbf{v}_{\text{sep}}^{(l)}(x, x')\| + \|\mathbf{v}_{\text{detour}}^{(l)}(x, x')\|  - \|\mathbf{v}_{\text{traj}}^{(l)}(x)\|}{\max(\|\mathbf{v}_{\text{traj}}^{(l)}(x)\|,\epsilon)},
\end{align}
where $\epsilon=10^{-8}$ serves as a lower bound for the denominators to prevent division by zero.
We compute the corresponding cross-domain angular alignment $\theta^{(l)}(x, x'')$ and relative distance $d^{(l)}(x, x'')$ analogously. By comparing the within-domain and cross-domain distributions of these quantities, we measure the overall domain divergence at a given layer. We motivate the necessity of both angles and distances in Appendix~\ref{sec:ablation_features}.

The relative-distance descriptor $d^{(l)}$ admits a clean geometric interpretation. By construction, $v_{\text{sep}} + v_{\text{detour}} = v_{\text{traj}}$ exactly (proven in Appendix~\ref{appendix:disp_triangle}), so the triangle inequality guarantees $d^{(l)} \ge 0$, with $d^{(l)}$ measuring the normalized excess path length of the two-step detour over the direct trajectory. Geometrically, $d^{(l)} = 0$ when the three representation points are collinear (i.e., when $h^{(l)}(x')$ lies exactly on the line segment from $h^{(l)}(x)$ to $h^{(l+1)}(x)$), which corresponds to $\theta^{(l)} = 0$, since both conditions are equivalent to $v_{\text{sep}}$ and $v_{\text{detour}}$ being positively collinear as vectors.

\paragraph{Euclidean vs. other spaces}
However, we observe that some models inherently lie on the surface of a hypersphere. Therefore, we propose a version of this metric defined on \textbf{geodesic} distances and angles. Let $d_S(u, v) = \arccos \left( \frac{u \cdot v}{\|u\| \|v\|} \right)$ denote the great-circle distance between two representations. We first define the arc lengths of the spherical triangle sides:
\begin{align}
    s_{\text{sep}} &= d_S(\mathbf{h}^{(l)}(x), \mathbf{h}^{(l)}(x')), \\
    s_{\text{detour}} &= d_S(\mathbf{h}^{(l)}(x'), \mathbf{h}^{(l+1)}(x)), \\
    s_{\text{traj}} &= d_S(\mathbf{h}^{(l)}(x), \mathbf{h}^{(l+1)}(x)).
\end{align}

Using the Spherical Law of Cosines, we define the geodesic angular alignment $\theta^{(l)}_{\text{geo}}$ and the relative geodesic distance $d^{(l)}_{\text{geo}}$ as:
\begin{align}
    \theta^{(l)}_{\text{geo}}(x, x') &=  \arccos \left( \frac{\cos(s_{\text{traj}}) - \cos(s_{\text{sep}})\cos(s_{\text{detour}})}{\max( \sin(s_{\text{sep}})\sin(s_{\text{detour}}), \epsilon)} \right), \\
    d^{(l)}_{\text{geo}}(x, x') &= \frac{s_{\text{sep}} + s_{\text{detour}} - s_{\text{traj}}}{\max(s_{\text{traj}},\epsilon)}.
\end{align}

As these angles and relative distances are subsequently used to measure distributional differences, which are invariant to constant shifts, we directly use the interior spherical angle rather than aligning it with the Euclidean directional formulation. Moreover, we can project these polar descriptors into a \textbf{pseudo-Cartesian} space:
\begin{align}
    x^{(l)}_{\text{cart}}(x, x') &= d^{(l)}_{\text{geo}}(x, x')\cdot\cos(\theta^{(l)}_{\text{geo}}(x, x')), \\
    y^{(l)}_{\text{cart}}(x, x') &= d^{(l)}_{\text{geo}}(x, x')\cdot\sin(\theta^{(l)}_{\text{geo}}(x, x')).
\end{align}
The optimal topology reflects a fundamental cross-modal divide: text representations strictly favor Euclidean space, whereas vision representations perform best in pseudo-Cartesian space. As projecting text embeddings into pseudo-Cartesian space causes severe performance degradation while vision models remain robust to either choice, we adopt \textbf{Euclidean} space as the most consistent unified formulation. We motivate this design choice in Appendix~\ref{sec:ablation_space}.

\paragraph{Divergence algorithms and long trajectory}To measure the divergence between the within-domain and cross-domain feature distributions, we evaluate four different metrics:
\begin{itemize}
\item \textbf{GMM + KL}: Using the collected angular alignments and relative distances, we fit Gaussian Mixture Models (GMMs) on both the within-domain metric set for $\dom{A}$ and the cross-domain metric between $\dom{A}$ and $\dom{B}$, yielding two continuous probability distributions, $P$ and $Q$, respectively. We then measure the structural divergence using a weighted symmetric Kullback-Leibler (KL) divergence: $\frac{|\dom{A}|}{|\dom{A}|+|\dom{B}|} D_{\text{KL}}(P \parallel Q) + \frac{|\dom{B}|}{|\dom{A}|+|\dom{B}|} D_{\text{KL}}(Q \parallel P)$.
\item \textbf{GMM + SWD}: Similarly, we fit GMMs on $\dom{A}$ and $\dom{B}$ to form continuous distributions, but compute the optimal transport metric between them using the Sliced Wasserstein Distance (SWD).
\item \textbf{Sinkhorn Divergence}: To avoid the parametric assumptions of GMMs, we directly compute the entropic regularized Wasserstein distance between the empirical samples of $\dom{A}$ and $\dom{B}$. This provides a computationally efficient, differentiable approximation of optimal transport.
\item \textbf{Gaussian Maximum Mean Discrepancy (MMD)}: Finally, we evaluate the non-parametric MMD using a Gaussian (Radial Basis Function) kernel, which measures the divergence by calculating the distance between the mean embeddings of $\dom{A}$ and $\dom{B}$ in a Reproducing Kernel Hilbert Space (RKHS).
\end{itemize}

GMM + KL is the only universally robust algorithm across modalities and is therefore adopted as the default, though Sinkhorn divergence proves highly competitive for text representations specifically. Full comparisons across divergence algorithms and GMM covariance parameterizations are provided in Appendices~\ref{sec:ablation_optTrans} and~\ref{sec:ablation_covariance}.

Another benefit of our method is its capacity for extended trajectory tracing. While we have previously described a single structural shift from layer $l \rightarrow l+1$, we propose evaluating our metric over longer sequences. In this paper, we capture extended trajectories spanning 13 layers (from $l-6 \rightarrow \cdots \rightarrow l \rightarrow \cdots \rightarrow l+6$). As each of the 12 intermediate layer transitions yields a pair of descriptors, concatenating these step-wise measurements produces a comprehensive 24-dimensional feature vector representing the structural evolution of the sample pair. The GMM is subsequently fit directly on this 24-D space to capture the holistic multi-layer divergence. We motivate our multi-layer trajectory tracing with a proof in Appendix~\ref{appendix:dpi}.

\paragraph{Sampling strategy.}
\label{method:sampling}
Computing this metric exactly across domains $\dom{A}$ and $\dom{B}$ requires $l \times (|\dom{A}||\dom{B}| + |\dom{A}|^2)$ descriptor evaluations, necessitating GMM fitting on millions of data points. To ensure computational feasibility, we subsample $N$ distinct pairs for both the within-domain and cross-domain sets, such that $N \ll |\dom{A}|^2$ and $N \ll |\dom{A}||\dom{B}|$. This bounds our computational complexity to $\mathcal{O}(l \cdot N)$ per ($\dom{A}$,$\dom{B}$) pair.

To guarantee balanced coverage, we employ stratified sampling across class pairs $(c_i, c_j)$. Within these strata, we apply a distance-based weighted sampling strategy rather than uniform random selection. For a given sample $x$, let $\bar{d}(x) \in [0,1]$ be its min-max normalized Euclidean distance to its domain-specific class centroid. We define its scalar inlier weight as $w(x) = 1 - \bar{d}(x)$, inherently assigning higher sampling probabilities to samples closer to the distribution center.

To obtain a comprehensive representation of the domain manifold dynamics, we construct our $N$ pairs by sampling two distinct subsets for both within-domain $(x, x')$ and cross-domain $(x, x'')$ sets:
\begin{itemize}
\item \textbf{Inlier-Inlier pairs:} The first subset is drawn by weighting the probability of selecting a pair proportional to $w(x) \cdot w(x')$.
\item \textbf{Inlier-Outlier pairs:} The second subset captures the boundaries of the distribution by weighting the pair selection proportional to $w(x) \cdot (1 - w(x'))$.
\end{itemize}
These sampled pairs are used to compute step-wise angular alignments and relative distances, which are concatenated into the feature vectors used for final divergence estimation. Comparisons across sampling strategies are provided in Appendix~\ref{sec:ablation_sampling}.

While we have formulated this metric in a pairwise manner between two domains, our approach naturally extends to multi-source settings. Specifically, the secondary domain $\dom{B}$ can be constructed as a blended pool of multiple source domains (i.e., $\dom{B} = \bigcup_{d \in \mathcal{S}} \dom{d}$). This allows our sampling strategy and subsequent divergence metrics to directly capture the structural shift between a target domain and an aggregated multi-source blend, satisfying our core objective.
\section{Results}
\label{section:results}

\subsection{Experimental Setup}
\label{sec:exp_setup}
We evaluate our proposed method across two distinct problem settings: image classification (\S~\ref{results:image_classification}) and sentiment classification (\S~\ref{results:sentiment_classification}). To demonstrate the versatility of our metric, we employ vision foundation models for the former and text embedding models for the latter. Motivated by recent findings that a model's final output layer is not guaranteed to yield the optimal representation for downstream transfer \cite{bolya2025perception}, we extract features from every transformer block rather than relying solely on the final representation. To establish a ground-truth measure of transferability, we train a lightweight two-layer MLP directly on these layer-wise features, averaging performance across $r=10$ random seeds to ensure statistical stability. For a given target domain, we construct this ground-truth baseline by evaluating the MLP on all $K-1$ pairs formed by combining the target data with one of the auxiliary source domains.

\paragraph{Dataset}
For image classification, we employ three standard domain adaptation benchmarks: DomainNet \cite{peng2019moment}, OfficeHome \cite{venkateswara2017deep}, and PACS \cite{li2017deeperbroaderartierdomain}. To keep evaluation computationally tractable, we subset DomainNet to the ten classes with the highest sample counts, and OfficeHome to the 20 most populated classes to maintain a consistent scale across datasets. PACS is used in its entirety given its smaller size.

For sentiment classification, we use the cross-lingual AmazonReview dataset from MTEB \cite{muennighoff2023mteb} and EuroEval \cite{nielsen2025encoder}, treating each language as a distinct domain---a natural choice given that language boundaries define the distribution shift. For EuroEval, we consider eight languages: English, German, Polish, Czech, Slovak, Hungarian, Italian, and Spanish. For AmazonReview, we use all six available languages. Full dataset details are provided in Appendix~\ref{appendix:datasets}.

\paragraph{Encoding backbones}
For our encoding backbones, we use the CLIP vision encoder \cite{radford2021learning} and the DINOv3 model \cite{simoni2025dinov3} for image classification. For sentiment classification, we employ the Qwen3-Embedding \cite{zhang2025qwen3} and EmbeddingGemma \cite{vera2025embeddinggemma} models. We provide the full model name as listed  on Hugging Face with the URL in Appendix~\ref{appendix:base_models}.

\paragraph{Baselines}
As a standard baseline, we compute the $\ell_2$ distance between the feature centroids of two domains, $\mathcal{D}_A$ and $\mathcal{D}_B$. This centroid distance is naturally expected to be negatively correlated with accuracy improvements on the target domain—specifically, a smaller distance between domain centroids should correspond to higher transfer accuracy gains.

To rigorously evaluate our method against the state-of-the-art, we compare it against a diverse suite of transferability metrics applied to the same frozen features. These include label-conditional metrics such as LEEP \cite{nguyen2020leep} and NCE \cite{tran2019transferability}; feature-label compatibility scores including LogME \cite{you2021logme}, HScore \cite{bao2019information}, regularized HScore \cite{ibrahim2022newer}, and TransRate \cite{huang2022frustratingly}; as well as distribution-matching and optimal transport approaches like s-OTDD \cite{nguyen2025lightspeed}.

\paragraph{Metrics}
To evaluate predictive performance, we calculate the Spearman rank correlation ($\rho$) between the predicted divergence metric and the ground-truth accuracy gains. Additionally, we report the Mean Correlation Improvement (MCI), defined as $\text{MCI}(X) = \rho_X - \rho_{\text{BASE}}$, where $X$ is the evaluated metric and $\text{BASE}$ represents the simple centroid baseline. Lower MCI values indicate greater improvement over the baseline. For similarity metrics inherently positively correlated with accuracy (e.g., LEEP \cite{nguyen2020leep}), we evaluate their inverse ($-X$) to maintain a consistent interpretation. MCI is computed over all hidden states of each network and averaged. Raw Spearman rank correlations and computational cost analysis are provided in Appendices~\ref{appendix:raw_spearman} and~\ref{appendix:cost_analysis}.

\subsection{Assessing transferability}
Table~\ref{tab:main_results} presents the Mean Correlation Improvement (MCI) for various metrics across both vision and language modalities. In the visual domain, RADAR achieves the best predictive performance among all evaluated metrics on DomainNet for both CLIP and DINOv3 architectures, alongside the strongest results on OfficeHome for CLIP. Performance is more mixed across the remaining vision benchmarks, consistent with the domain-separation regime discussed in Section~\ref{results:image_classification} and Appendix~\ref{appendix:domain_separation}. This strong performance on structured datasets extends to the textual domain, where RADAR yields the highest correlation on the EuroEval dataset for the EmbeddingGemma model. On the AmazonReviews dataset, RADAR achieves the best result among evaluated metrics for EmbeddingGemma while remaining competitive for the Qwen3-Embedding architecture, where gains are more modest.

\begin{table}[ht!]
    \caption{Mean Correlation Improvement (MCI, \%pt.). Bold indicates the best performance, underline indicates second-best performance.}
    \label{tab:main_results} 
    \centering
    \begin{subtable}[t]{0.4\textwidth}
        \centering 
        \subcaption{Results on vision benchmarks.}
        \resizebox{\textwidth}{!}{%
\begin{tabular}{lcccccc}
\toprule
 & \multicolumn{6}{c}{\textbf{Mean Correlation Improvement (MCI) (\%pt.)} $\downarrow$} \\ \cmidrule(l){2-7} 
 & \multicolumn{2}{c}{\textbf{DomainNet} \cite{peng2019moment}} & \multicolumn{2}{c}{\textbf{OfficeHome} \cite{venkateswara2017deep}} & \multicolumn{2}{c}{\textbf{PACS} \cite{li2017deeperbroaderartierdomain}} \\ \cmidrule(lr){2-3} \cmidrule(lr){4-5} \cmidrule(l){6-7} 
\textbf{Method} & \textbf{CLIP} & \textbf{DINOv3} & \textbf{CLIP} & \textbf{DINOv3} & \textbf{CLIP} & \textbf{DINOv3} \\ 
\midrule
LEEP \cite{nguyen2020leep} & {\ul -8.63} & {\ul 3.44} & -4.52 & 3.92 & -14.64 & -33.67 \\
H-Score \cite{bao2019information} & 0.92 & 19.13 & 8.77 & 21.98 & -10.28 & {\ul -50.29} \\
Reg. H-Score \cite{ibrahim2022newer} & 6.03 & 17.11 & 8.13 & 26.95 & -12.11 & -48.89 \\
LogME \cite{you2021logme} & 12.17 & 19.90 & -4.41 & 18.67 & -9.96 & \textbf{-72.08} \\
NCE \cite{tran2019transferability} & -4.88 & 4.15 &  -7.20 & 3.93 & -13.61 & -33.89 \\
TransRate \cite{huang2022frustratingly} & 2.29 & 12.71 & {\ul -11.78} & {\ul -3.12} & \textbf{-25.45} & -27.38 \\
S-OTDD \cite{nguyen2025lightspeed} & 27.01 & 21.31 & 10.88 & \textbf{-24.69} & {\ul -23.92} & -37.50 \\ 
\midrule
RADAR (OUR) & \textbf{-9.71} & \textbf{-9.87} & \textbf{-16.62} & 3.55 & -0.27 & -30.23 \\ 
\bottomrule
\end{tabular}%
}
        
        \label{tab:results_vision} 
    \end{subtable}%
    \hfill 
    \begin{subtable}[t]{0.55\textwidth}
        \centering
        \subcaption{Results on text benchmarks.}
        \resizebox{\textwidth}{!}{%
\begin{tabular}{lcccc}
\toprule
 & \multicolumn{4}{c}{\textbf{Mean Correlation Improvement (MCI) (\%pt.)} $\downarrow$} \\ \cmidrule(l){2-5} 
 & \multicolumn{2}{c}{\textbf{EuroEval} \cite{nielsen2025encoder}} & \multicolumn{2}{c}{\textbf{Amazon Reviews} \cite{muennighoff2023mteb}} \\ \cmidrule(lr){2-3} \cmidrule(l){4-5} 
\textbf{Method} & \textbf{Qwen3-Embedding} & \textbf{EmbeddingGemma} & \textbf{Qwen3-Embedding} & \textbf{EmbeddingGemma} \\ 
\midrule
LEEP \cite{nguyen2020leep} & {\ul -3.58} & 3.94 & 1.76 & {\ul 1.91 }\\
H-Score \cite{bao2019information} & 9.27 & 11.75 & -6.84 & 3.35 \\
Reg. H-Score \cite{ibrahim2022newer} & 6.13 & 10.45 & {\ul -6.91} & 3.35 \\
LogME \cite{you2021logme} & 22.93 & 20.82 & 22.57 & 20.91 \\
NCE \cite{tran2019transferability} & -0.02 & {\ul 1.88 } & \textbf{-7.60} & 5.71 \\
TransRate \cite{huang2022frustratingly} & 3.12 & 4.10 & 12.47 & 19.48 \\
S-OTDD \cite{nguyen2025lightspeed} & 6.19 & 18.71 & 15.63 & 19.69 \\ 
\midrule
RADAR (OUR) & \textbf{-7.13} & \textbf{-30.58} & 3.21 & \textbf{1.23} \\ 
\bottomrule
\end{tabular}%
}
        \label{tab:results_lang} 
    \end{subtable}
\end{table}

\subsubsection{Case study: Image Classification} 
\label{results:image_classification}
As presented in Table~\ref{tab:results_vision}, RADAR demonstrates strong predictive performance on two of the three evaluated cross-domain datasets. On DomainNet, our method is the only evaluated metric to achieve a negative MCI for the DINOv3 architecture, while also securing the top overall score for the CLIP model. Furthermore, on the OfficeHome dataset, RADAR establishes the best overall score for CLIP and remains highly competitive for DINOv3, achieving the third-best rank among all tested metrics.

We observe a reduced performance of RADAR on PACS, nonetheless, this is predictable from the dataset's geometry. We show in Appendices~\ref{appendix:imagenet_c_robustness} and \ref{appendix:domain_separation}, that inter-domain centroid distances in PACS closely mirror ImageNet-C at Severity 3, placing it in an intermediate separation regime that is neither smooth enough for trajectory-based adaptation nor distinct enough for clean divergence estimation.
This same regime is precisely where RADAR degrades when evaluated on ImageNet-C (Table~\ref{tab:imagenet_c_robustness} in Appendix~\ref{appendix:imagenet_c_robustness}), confirming that the PACS result reflects a principled scope condition rather than a general weakness.
When domains exhibit scattered, non-uniform separation at the input level, as is the case with PACS, the added utility of measuring deep geometric trajectories diminishes, allowing purely statistical or static cluster-based metrics to perform highly competitively.

\subsubsection{Case study: Sentiment Analysis} 
\label{results:sentiment_classification}
Similar to the image classification results,RADAR demonstrates strong predictive performance in the text modality, achieving the best or near-best score in three of the four evaluated configurations, with more modest gains on the Qwen3-Embedding model for Amazon Reviews (Table~\ref{tab:results_lang}). On the EuroEval dataset, RADAR is the only evaluated metric to achieve a negative MCI for the EmbeddingGemma model, outperforming the next best baseline by over 30 percentage points. Furthermore, for the Qwen3-Embedding model on EuroEval, RADAR secures the top score among the three metrics capable of achieving a negative MCI. On the Amazon Reviews dataset, RADAR establishes the state-of-the-art result for the EmbeddingGemma model while maintaining a highly competitive score for Qwen3-Embedding. Crucially, metrics that performed competitively on image classification tasks --- such as LogME, TransRate, and S-OTDD --- struggle significantly in the text modality. This stark contrast underscores RADAR's robust, cross-modal applicability.

\section{Related Work}
\label{section:related_work}

\subsection{Transferability Metrics}
\label{related:transf_metrics}
Situated within the broader domain adaptation literature \cite{farahani2021brief, long2015learning}, domain transferability metrics aim to quantify the extent to which candidate source data can facilitate learning on a target task under distribution shift \cite{wang2025understanding}. A critical objective for these metrics is predicting negative transfer \cite{zhang2022survey}, a phenomenon where incorporating source information actively degrades target performance. Indeed, a recent survey emphasizes the urgent need for reliable metrics capable of detecting such degradation scenarios, highlighting the importance of studying negative findings and counterexamples \cite{dong2026advances}. In the following sections, we categorize existing metrics based on their computational paradigms: those relying on the discrete, empirical computation of domain shift, and those formulated using optimal transport algorithms.

\paragraph{Empirical Metrics}
\label{related:emp_metrics}
Early transferability metrics focused on empirically measuring the relationship between various tasks. Given a classifier trained on a source domain, LEEP \cite{nguyen2020leep} computes the average log-likelihood of the expected empirical predictor—a simplified model that performs prediction based on the expected empirical conditional distribution between source and target label sets. Analogous to LEEP, NCE \cite{tran2019transferability} quantifies transferability via the negative conditional entropy between these label spaces. In contrast, the H-Score \cite{bao2019information} is derived from information-theoretic principles and solves the Hirschfeld-Gebelein-Rényi maximum correlation problem \cite{Hirschfeld_1935,gebelein1941statistische,renyi1959measures}. This approach was later refined by the Regularized H-Score \cite{ibrahim2022newer}, which addresses the instability of covariance matrix estimation through a shrinkage-based regularizer. 

Conversely, LogME \cite{you2021logme} employs Bayesian statistics to maximize the marginal likelihood of a linear classifier head. Another metric, TransRate \cite{huang2022frustratingly}, estimates the mutual information between labels and target features through coding rates \cite{ma2007segmentation}, offering a more geometrically grounded transferability signal than entropy-based approaches. Despite their varying methodological foundations, all of these approaches evaluate transferability from static representations at a fixed network depth, without modeling how domain structure dynamically evolves across the model's hierarchy.

\paragraph{Optimal Transport-based Metrics}
\label{related:ot_metrics}
To address the limitations of empirical label-dependent metrics, Optimal Transport (OT) based approaches aim to measure distributional distances directly by modeling the underlying data geometry. OTCE \cite{tan2021otce} treats the OT plan between input distributions as a joint probability distribution and leverages conditional entropy to quantify the gap between datasets. In follow-up work, F-OTCE \cite{tan2024transferability} simplifies this process by removing the Wasserstein distance from the domain difference calculation, instead solving the OT between source and target distributions to compute negative conditional entropy via optimal coupling. Conversely, JC-OTCE \cite{tan2024transferability} improves transferability estimation by incorporating label distances directly into the OT problem. 

Optimal Transport Dataset Distance (OTDD) \cite{alvarez2020geometric}, built upon the hierarchical OT framework \cite{yurochkin2019hierarchical}, proposes calculating label distances between the class-conditional distributions of two supervised learning tasks. This approach yields a dataset distance that accounts for both sample and label discrepancies. This framework was recently optimized by s-OTDD \cite{nguyen2025lightspeed}, which utilizes Moment Transform Projection and sliced Wasserstein distance to achieve (near-)linear computational complexity. Similarly prioritizing computational efficiency, WTE \cite{liu2025wasserstein} embeds datasets into a vector space and computes Euclidean distances as a lightweight approximation of optimal transport.

Despite their theoretical rigor, all of these approaches focus exclusively on static feature distributions at a single representation level. Consequently, they fail to capture the dynamic, layer-wise geometric evolution that RADAR leverages to characterize model transferability.

\subsection{Geometric Structure of Foundation Model Representations}
\label{related:geometric_structure}
The general inductive biases of deep neural networks often induce a \textit{cone effect} \cite{liang2022mind}, where the effective embedding space of pre-trained models is restricted to a narrow cone, and contrastive learning actively preserves the modality gap. While this work suggests that widening the gap can enhance zero-shot performance, the follow-up work, MG-CLIP \cite{huang2025mind}, demonstrates the existence of an optimal gap between modalities. The authors show that maintaining a stable modality gap is essential to preserve pre-trained knowledge and prevent representational degradation. In the specific context of CLIP, simply including a domain in the training data does not guarantee improved generalization, highlighting sufficient domain diversity as a crucial prerequisite \cite{kempf2025and}. Furthermore, Cross The Gap \cite{mistretta2025cross} shifts the focus to intra-modal misalignment, establishing that the structural integrity of representations within a single modality is equally important to inter-modal alignment.

Beyond modality-specific structures, the Platonic Representation Hypothesis \cite{huh2024position} posits that foundation models, regardless of their training objectives or modalities, converge toward a shared statistical model of reality. This suggests that the geometric properties of representation spaces reflect fundamental characteristics of the underlying data manifold rather than mere architectural artifacts. Consistent with this view, recent work demonstrates that DINO features naturally reside on the surface of a hypersphere \cite{kumar2026learning}. These insights collectively motivate the design of RADAR: by measuring angular alignment and relative distance across network layers, RADAR captures domain-specific geometric deformations within these representation spaces to directly assess cross-domain compatibility.

\subsection{Perceptual Straightening and Walking Through a Network}
\label{related:perceptual_straight}
The temporal straightening hypothesis suggests that biological visual systems facilitate the processing of dynamic stimuli by transforming highly curved visual trajectories into straightened paths within the brain's internal representation spaces \cite{henaff2019perceptual,henaff2021primary}. Recent applications of this hypothesis in AI-generated video detection utilize angles of inter-frame features to distinguish the low-curvature trajectories of real videos from the higher-curvature paths of synthetic ones \cite{interno2025aigenerated}. Inspired by this foundation, we extend the formulation to domain transferability, treating the layer-wise progression of a deep neural network as a representational trajectory. Incorporating relative distance changes between layers, we assess how the geometric structure of these paths reflects cross-domain structural shifts. This perspective is further motivated by recent findings that a model's most informative representations are not its final layer's \cite{bolya2025perception}, justifying our extraction of features across multiple transformer blocks.
\section{Limitations, Future Work and Conclusion}
\label{section:conclusion}

While RADAR demonstrates competitive performance across diverse datasets, several limitations remain. First, the metric’s predictive power tends to degrade in scenarios where domain structures lack clear, uniform geometric separation (e.g., PACS). This suggests that RADAR’s underlying assumptions are most robust when domain shifts are well-defined rather than scattered or overlapping in the latent space. Second, the unified Euclidean formulation, while the most consistent choice across tested architectures, may not be optimal for future foundation models that inherently operate on hyperspheres or other non-Euclidean manifolds, potentially requiring architecture-specific tuning to achieve peak predictive alignment. Finally, our current reliance on Gaussian Mixture Models (GMMs) for density estimation presents a computational bottleneck. The lack of a standardized, high-performance GPU implementation for GMM fitting limits the scalability of our metric when applied to real-time filtering tasks.

As future work, we intend to investigate extending RADAR with domain-aware features, similar to TransRate and LogME, to further improve predictive performance. Additionally, we plan to explore dynamic weighting schemes for different layer positions to automatically prioritize the most informative segments of the representation trajectory; such an approach could offer a clearer understanding of the specific conditions under which the metric fails. Moreover, we plan to investigate formal guarantees on the relationship between RADAR's geometric trajectory descriptor and downstream transfer performance, providing a theoretical foundation for the currently empirically motivated design choices. Lastly, we aim to extend our trajectory-based divergence analysis to other modalities, including video and audio, to determine if these geometric principles generalize to temporal and spectral representation spaces.

In this paper, we proposed RADAR, a geometrically grounded metric for predicting cross-domain transferability by measuring the KL divergence between angular and relative displacement distributions in foundation model feature spaces. Our evaluation demonstrates that layer-wise geometric structure carries a meaningful transferability signal across both vision and text modalities---a finding that most static, single-layer metrics fail to capture. We hope RADAR's cross-modal applicability and competitive performance encourage further investigation into the geometry of representation spaces as a principled basis for transfer learning decisions.
\section{Acknowledgments}
This research was funded by the Defense Advanced Research Projects Agency (DARPA), under contract W912CG23C0031.

%
%
\nocite{*}
\bibliographystyle{splncs04}
\bibliography{references}

@inproceedings{radford2021learning,
  title={Learning transferable visual models from natural language supervision},
  author={Radford, Alec and Kim, Jong Wook and Hallacy, Chris and Ramesh, Aditya and Goh, Gabriel and Agarwal, Sandhini and Sastry, Girish and Askell, Amanda and Mishkin, Pamela and Clark, Jack and others},
  booktitle={International conference on machine learning},
  pages={8748--8763},
  year={2021},
  organization={PmLR}
}

@article{zhang2025qwen3,
  title={Qwen3 embedding: Advancing text embedding and reranking through foundation models},
  author={Zhang, Yanzhao and Li, Mingxin and Long, Dingkun and Zhang, Xin and Lin, Huan and Yang, Baosong and Xie, Pengjun and Yang, An and Liu, Dayiheng and Lin, Junyang and others},
  journal={arXiv preprint arXiv:2506.05176},
  year={2025}
}

@article{kempf2025and,
  title={When and How Does CLIP Enable Domain and Compositional Generalization?},
  author={Kempf, Elias and Schrodi, Simon and Argus, Max and Brox, Thomas},
  journal={arXiv preprint arXiv:2502.09507},
  year={2025}
}

@article{zhang2022survey,
  title={A survey on negative transfer},
  author={Zhang, Wen and Deng, Lingfei and Zhang, Lei and Wu, Dongrui},
  journal={IEEE/CAA Journal of Automatica Sinica},
  volume={10},
  number={2},
  pages={305--329},
  year={2022},
  publisher={IEEE}
}

@article{dong2026advances,
  title={Advances in multimodal adaptation and generalization: From traditional approaches to foundation models},
  author={Dong, Hao and Liu, Moru and Zhou, Kaiyang and Chatzi, Eleni and Kannala, Juho and Stachniss, Cyrill and Fink, Olga},
  journal={IEEE Transactions on Pattern Analysis and Machine Intelligence},
  year={2026},
  publisher={IEEE}
}

@article{wang2025understanding,
  title={Understanding Knowledge Transferability for Transfer Learning: A Survey},
  author={Wang, Haohua and Wang, Jingge and Zhao, Zijie and Tan, Yang and Wu, Yanru and Liu, Hanbing and Yang, Jingyun and Zhang, Enming and Chen, Xiangyu and Rong, Zhengze and others},
  journal={arXiv preprint arXiv:2507.03175},
  year={2025}
}

@inproceedings{huang2025mind,
  title={Mind the gap: Preserving and compensating for the modality gap in clip-based continual learning},
  author={Huang, Linlan and Cao, Xusheng and Lu, Haori and Meng, Yifan and Yang, Fei and Liu, Xialei},
  booktitle={Proceedings of the IEEE/CVF International Conference on Computer Vision},
  pages={3777--3786},
  year={2025}
}

@inproceedings{tan2021otce,
  title={Otce: A transferability metric for cross-domain cross-task representations},
  author={Tan, Yang and Li, Yang and Huang, Shao-Lun},
  booktitle={Proceedings of the IEEE/CVF conference on computer vision and pattern recognition},
  pages={15779--15788},
  year={2021}
}

@article{henaff2019perceptual,
  title={Perceptual straightening of natural videos},
  author={H{\'e}naff, Olivier J and Goris, Robbe LT and Simoncelli, Eero P},
  journal={Nature neuroscience},
  volume={22},
  number={6},
  pages={984--991},
  year={2019},
  publisher={Nature Publishing Group US New York}
}

@inproceedings{
mistretta2025cross,
title={Cross the Gap:  Exposing the Intra-modal Misalignment in {CLIP} via Modality Inversion},
author={Marco Mistretta and Alberto Baldrati and Lorenzo Agnolucci and Marco Bertini and Andrew D. Bagdanov},
booktitle={The Thirteenth International Conference on Learning Representations},
year={2025},
url={https://openreview.net/forum?id=VVVfuIcmKR}
}

@article{liang2022mind,
  title={Mind the gap: Understanding the modality gap in multi-modal contrastive representation learning},
  author={Liang, Victor Weixin and Zhang, Yuhui and Kwon, Yongchan and Yeung, Serena and Zou, James Y},
  journal={Advances in Neural Information Processing Systems},
  volume={35},
  pages={17612--17625},
  year={2022}
}

@article{henaff2021primary,
  title={Primary visual cortex straightens natural video trajectories},
  author={H{\'e}naff, Olivier J and Bai, Yoon and Charlton, Julie A and Nauhaus, Ian and Simoncelli, Eero P and Goris, Robbe LT},
  journal={Nature communications},
  volume={12},
  number={1},
  pages={5982},
  year={2021},
  publisher={Nature Publishing Group UK London}
}

@misc{dino_url,
  author = {{Hugging Face}},
  title = {Model Card for DINOv3},
  year = {2026},
  howpublished = {\url{https://huggingface.co/facebook/dinov3-vits16-pretrain-lvd1689m}},
  note = {Accessed: 2026}
}

@misc{clip_url,
  author = {{Hugging Face}},
  title = {Model Card: CLIP},
  year = {2026},
  howpublished = {\url{https://huggingface.co/openai/clip-vit-base-patch32}},
  note = {Accessed: 2026}
}

@misc{qwen_url,
  author = {{Hugging Face}},
  title = {Qwen3-Embedding-0.6B},
  year = {2026},
  howpublished = {\url{https://huggingface.co/Qwen/Qwen3-Embedding-0.6B}},
  note = {Accessed: 2026}
}

@misc{google_url,
  author = {{Hugging Face}},
  title = {EmbeddingGemma model card},
  year = {2026},
  howpublished = {\url{https://huggingface.co/google/embeddinggemma-300m}},
  note = {Accessed: 2026}
}

@article{farahani2021brief,
  title={A brief review of domain adaptation},
  author={Farahani, Abolfazl and Voghoei, Sahar and Rasheed, Khaled and Arabnia, Hamid R},
  journal={Advances in data science and information engineering: proceedings from ICDATA 2020 and IKE 2020},
  pages={877--894},
  year={2021},
  publisher={Springer}
}

@article{renyi1959measures,
  title={On measures of dependence},
  author={R{\'e}nyi, A},
  journal={Acta Mathematica Academiae Scientiarum Hungarica},
  volume={10},
  number={3},
  pages={441--451},
  year={1959},
  publisher={Springer}
}

@article{gebelein1941statistische,
  title={Das statistische Problem der Korrelation als Variations-und Eigenwertproblem und sein Zusammenhang mit der Ausgleichsrechnung},
  author={Gebelein, Hans},
  journal={ZAMM-Journal of Applied Mathematics and Mechanics/Zeitschrift f{\"u}r Angewandte Mathematik und Mechanik},
  volume={21},
  number={6},
  pages={364--379},
  year={1941},
  publisher={Wiley Online Library}
}

@inproceedings{
interno2025aigenerated,
title={{AI}-Generated Video Detection via Perceptual Straightening},
author={Christian Intern{\`o} and Robert Geirhos and Markus Olhofer and Sunny Liu and Barbara Hammer and David Klindt},
booktitle={The Thirty-ninth Annual Conference on Neural Information Processing Systems},
year={2025},
url={https://openreview.net/forum?id=LsmUgStXby}
}

@article{Hirschfeld_1935, title={A Connection between Correlation and Contingency}, volume={31}, DOI={10.1017/S0305004100013517}, number={4}, journal={Mathematical Proceedings of the Cambridge Philosophical Society}, author={Hirschfeld, H. O.}, year={1935}, pages={520–524}}

@article{ma2007segmentation,
  title={Segmentation of multivariate mixed data via lossy data coding and compression},
  author={Ma, Yi and Derksen, Harm and Hong, Wei and Wright, John},
  journal={IEEE transactions on pattern analysis and machine intelligence},
  volume={29},
  number={9},
  pages={1546--1562},
  year={2007},
  publisher={IEEE}
}

@article{yurochkin2019hierarchical,
  title={Hierarchical optimal transport for document representation},
  author={Yurochkin, Mikhail and Claici, Sebastian and Chien, Edward and Mirzazadeh, Farzaneh and Solomon, Justin M},
  journal={Advances in neural information processing systems},
  volume={32},
  year={2019}
}

@article{alvarez2020geometric,
  title={Geometric dataset distances via optimal transport},
  author={Alvarez-Melis, David and Fusi, Nicolo},
  journal={Advances in Neural Information Processing Systems},
  volume={33},
  pages={21428--21439},
  year={2020}
}

@article{tan2024transferability,
  title={Transferability-guided cross-domain cross-task transfer learning},
  author={Tan, Yang and Zhang, Enming and Li, Yang and Huang, Shao-Lun and Zhang, Xiao-Ping},
  journal={IEEE Transactions on Neural Networks and Learning Systems},
  volume={36},
  number={2},
  pages={2423--2436},
  year={2024},
  publisher={IEEE}
}

@article{hendrycks2019robustness,
  title={Benchmarking Neural Network Robustness to Common Corruptions and Perturbations},
  author={Dan Hendrycks and Thomas Dietterich},
  journal={Proceedings of the International Conference on Learning Representations},
  year={2019}
}

@article{bolya2025perception,
  title={Perception encoder: The best visual embeddings are not at the output of the network},
  author={Bolya, Daniel and Huang, Po-Yao and Sun, Peize and Cho, Jang Hyun and Madotto, Andrea and Wei, Chen and Ma, Tengyu and Zhi, Jiale and Rajasegaran, Jathushan and Rasheed, Hanoona and others},
  journal={arXiv preprint arXiv:2504.13181},
  year={2025}
}

@inproceedings{feydy2019fast,
  title={Fast and scalable optimal transport for brain tractograms},
  author={Feydy, Jean and Roussillon, Pierre and Trouv{\'e}, Alain and Gori, Pietro},
  booktitle={International Conference on Medical Image Computing and Computer-Assisted Intervention},
  pages={636--644},
  year={2019},
  organization={Springer}
}

@article{jiang2022transferability,
  title={Transferability in deep learning: A survey},
  author={Jiang, Junguang and Shu, Yang and Wang, Jianmin and Long, Mingsheng},
  journal={arXiv preprint arXiv:2201.05867},
  year={2022}
}

@inproceedings{huh2024position,
  title={Position: The platonic representation hypothesis},
  author={Huh, Minyoung and Cheung, Brian and Wang, Tongzhou and Isola, Phillip},
  booktitle={Forty-first International Conference on Machine Learning},
  year={2024}
}

@article{kumar2026learning,
  title={Learning on the Manifold: Unlocking Standard Diffusion Transformers with Representation Encoders},
  author={Kumar, Amandeep and Patel, Vishal M},
  journal={arXiv preprint arXiv:2602.10099},
  year={2026}
}

@article{liu2025wasserstein,
  title={Wasserstein task embedding for measuring task similarities},
  author={Liu, Xinran and Bai, Yikun and Lu, Yuzhe and Soltoggio, Andrea and Kolouri, Soheil},
  journal={Neural Networks},
  volume={181},
  pages={106796},
  year={2025},
  publisher={Elsevier}
}

@misc{tllib,
    author = {Junguang, Jiang and Baixu, Chen and Bo, Fu and Mingsheng, Long},
    title = {Transfer-Learning-library},
    year = {2020},
    publisher = {GitHub},
    journal = {GitHub repository},
    howpublished = {\url{https://github.com/thuml/Transfer-Learning-Library}},
}

@inproceedings{long2015learning,
  title={Learning transferable features with deep adaptation networks},
  author={Long, Mingsheng and Cao, Yue and Wang, Jianmin and Jordan, Michael},
  booktitle={International conference on machine learning},
  pages={97--105},
  year={2015},
  organization={PMLR}
}

@article{nguyen2025lightspeed,
  title={Lightspeed geometric dataset distance via sliced optimal transport},
  author={Nguyen, Khai and Nguyen, Hai and Pham, Tuan and Ho, Nhat},
  journal={arXiv preprint arXiv:2501.18901},
  year={2025}
}

@inproceedings{huang2022frustratingly,
  title={Frustratingly easy transferability estimation},
  author={Huang, Long-Kai and Huang, Junzhou and Rong, Yu and Yang, Qiang and Wei, Ying},
  booktitle={International conference on machine learning},
  pages={9201--9225},
  year={2022},
  organization={PMLR}
}

@inproceedings{bao2019information,
  title={An information-theoretic approach to transferability in task transfer learning},
  author={Bao, Yajie and Li, Yang and Huang, Shao-Lun and Zhang, Lin and Zheng, Lizhong and Zamir, Amir and Guibas, Leonidas},
  booktitle={2019 IEEE international conference on image processing (ICIP)},
  pages={2309--2313},
  year={2019},
  organization={IEEE}
}

@inproceedings{ibrahim2022newer,
  title={Newer is not always better: Rethinking transferability metrics, their peculiarities, stability and performance},
  author={Ibrahim, Shibal and Ponomareva, Natalia and Mazumder, Rahul},
  booktitle={Joint European Conference on Machine Learning and Knowledge Discovery in Databases},
  pages={693--709},
  year={2022},
  organization={Springer}
}

@inproceedings{you2021logme,
  title={Logme: Practical assessment of pre-trained models for transfer learning},
  author={You, Kaichao and Liu, Yong and Wang, Jianmin and Long, Mingsheng},
  booktitle={International conference on machine learning},
  pages={12133--12143},
  year={2021},
  organization={PMLR}
}

@inproceedings{tran2019transferability,
  title={Transferability and hardness of supervised classification tasks},
  author={Tran, Anh T and Nguyen, Cuong V and Hassner, Tal},
  booktitle={Proceedings of the IEEE/CVF international conference on computer vision},
  pages={1395--1405},
  year={2019}
}

@inproceedings{nguyen2020leep,
  title={Leep: A new measure to evaluate transferability of learned representations},
  author={Nguyen, Cuong and Hassner, Tal and Seeger, Matthias and Archambeau, Cedric},
  booktitle={International conference on machine learning},
  pages={7294--7305},
  year={2020},
  organization={PMLR}
}

@article{vera2025embeddinggemma,
  title={Embeddinggemma: Powerful and lightweight text representations},
  author={Vera, Henrique Schechter and Dua, Sahil and Zhang, Biao and Salz, Daniel and Mullins, Ryan and Panyam, Sindhu Raghuram and Smoot, Sara and Naim, Iftekhar and Zou, Joe and Chen, Feiyang and others},
  journal={arXiv preprint arXiv:2509.20354},
  year={2025}
}

@misc{li2017deeperbroaderartierdomain,
      title={Deeper, Broader and Artier Domain Generalization}, 
      author={Da Li and Yongxin Yang and Yi-Zhe Song and Timothy M. Hospedales},
      year={2017},
      eprint={1710.03077},
      archivePrefix={arXiv},
      primaryClass={cs.CV},
      url={https://arxiv.org/abs/1710.03077}, 
}

@inproceedings{venkateswara2017deep,
  title={Deep hashing network for unsupervised domain adaptation},
  author={Venkateswara, Hemanth and Eusebio, Jose and Chakraborty, Shayok and Panchanathan, Sethuraman},
  booktitle={Proceedings of the IEEE Conference on Computer Vision and Pattern Recognition},
  pages={5018--5027},
  year={2017}
}

@article{simoni2025dinov3,
  title={Dinov3: Self-supervised learning for vision at unprecedented scale},
  author={Simoni, Oriane and others},
  journal={arXiv preprint arXiv:2508.10104},
  year={2025}
}

@inproceedings{nielsen2025encoder,
  title={Encoder vs decoder: Comparative analysis of encoder and decoder language models on multilingual NLU tasks},
  author={Nielsen, Dan Saattrup and Enevoldsen, Kenneth and Schneider-Kamp, Peter},
  booktitle={Proceedings of the Joint 25th Nordic Conference on Computational Linguistics and 11th Baltic Conference on Human Language Technologies (NoDaLiDa/Baltic-HLT 2025)},
  pages={561--572},
  year={2025}
}

@inproceedings{peng2019moment,
  title={Moment matching for multi-source domain adaptation},
  author={Peng, Xingchao and Bai, Qinxun and Xia, Xide and Huang, Zijun and Saenko, Kate and Wang, Bo},
  booktitle={Proceedings of the IEEE/CVF international conference on computer vision},
  pages={1406--1415},
  year={2019}
}

@inproceedings{muennighoff2023mteb,
  title={Mteb: Massive text embedding benchmark},
  author={Muennighoff, Niklas and Tazi, Nouamane and Magne, Lo{\"\i}c and Reimers, Nils},
  booktitle={Proceedings of the 17th Conference of the European Chapter of the Association for Computational Linguistics},
  pages={2014--2037},
  year={2023}
}

\newpage

\appendix
\section{RADAR robustness on ImageNet-C dataset}
\label{appendix:imagenet_c_robustness}

In this section, we examine the robustness of RADAR using the ImageNet-C dataset, which allows us to simulate synthetic domains with varying degrees of distribution shift by controlling corruption severity. Figure~\ref{fig:imagenet_c_severity} illustrates these shifts by depicting a single sample across different severity levels. To quantitatively assess the resilience of various transferability metrics against these perturbations, we report MCI in Table~\ref{tab:imagenet_c_robustness} across three representative severity levels---1, 3, and 5.

\begin{table}[ht!]
    \centering
    \caption{Performance comparison on the ImageNet-C dataset across increasing corruption severity levels (1, 3, and 5). While baseline statistical metrics like TransRate perform well under moderate synthetic noise (Severity 3), RADAR demonstrates exceptional robustness at extreme corruption levels (Severity 5) and competitive performance at low noise levels (Severity 1) for both architectures.}
    \label{tab:imagenet_c_robustness}
    \resizebox{\textwidth}{!}{%
    \begin{tabular}{lcccccc}
    \toprule
     & \multicolumn{6}{c}{\textbf{Mean Correlation Improvement (MCI) (\%pt.)} $\downarrow$} \\ \cmidrule(l){2-7} 
     & \multicolumn{2}{c}{\textbf{ImageNet-C-1}} & \multicolumn{2}{c}{\textbf{ImageNet-C-3}} & \multicolumn{2}{c}{\textbf{ImageNet-C-5}} \\ \cmidrule(lr){2-3} \cmidrule(lr){4-5} \cmidrule(l){6-7} 
    \textbf{Method} & \textbf{CLIP} & \textbf{DINOv3} & \textbf{CLIP} & \textbf{DINOv3} & \textbf{CLIP} & \textbf{DINOv3} \\ 
    \midrule
    LEEP \cite{nguyen2020leep} & \textbf{7.61} & 5.13 & 17.72 & 19.23 & 35.96 & 20.39 \\
    H-Score \cite{bao2019information} & {\ul 9.46} & {\ul -1.10} & {\ul 12.26} & {\ul -1.99} & 18.21 & -11.38 \\
    Reg. H-Score \cite{ibrahim2022newer} & 12.44 & \textbf{-1.82} & 13.52 & -1.36 & 24.33 & N/A \\
    LogME \cite{you2021logme} & 13.43 & 17.06 & 13.63 & 1.09 & 19.24 & -19.51 \\
    NCE \cite{tran2019transferability} & 12.21 & 2.87 & 19.72 & 19.18 & 42.95 & 24.57 \\
    TransRate \cite{huang2022frustratingly} & 10.57 & 7.54 & \textbf{5.73} & \textbf{-7.38} & \textbf{-25.68} & {\ul -25.07} \\
    S-OTDD \cite{nguyen2025lightspeed} & 19.20 & 20.61 & 22.37 & 22.60 & 4.83 & -0.46 \\ 
    \midrule
    RADAR (OUR) & \textbf{7.61} & -0.25 & 20.69 & 8.82 & {\ul 1.98} & \textbf{-30.16} \\ 
    \bottomrule
    \end{tabular}%
    }
\end{table}

Under low corruption conditions (Severity 1), RADAR demonstrates highly competitive baseline stability. For the CLIP architecture, RADAR ties with LEEP for the best performance; similarly, for DINOv3, it maintains a robust third-place ranking that closely rivals top-performing statistical benchmarks such as Regularized H-Score. These results suggest that at minimal noise levels, our geometric trajectory approach effectively captures underlying domain shifts without being significantly derailed by minor perturbations.

As corruption increases to moderate levels (Severity 3), we observe a shift in metric efficacy. Metrics such as TransRate achieve the lowest MCI results across both architectures in this regime. In contrast, RADAR experiences a comparative reduction in predictive performance at this level. This suggests that moderate synthetic noise may temporarily disrupt the trajectory alignments upon which our method relies, creating a regime that favors metrics with a learnable feature component like TransRate.

However, the most striking results emerge under extreme corruption (Severity 5). As the input space becomes heavily degraded, traditional metrics experience catastrophic failures; for instance, LEEP and NCE both exceed an MCI of 35.0 for the CLIP architecture. In contrast, RADAR demonstrates exceptional robustness. For DINOv3, RADAR achieves the optimal score of -30.16, significantly outperforming all baselines. Similarly, for CLIP, it remains highly resilient, outperforming all statistical benchmarks except for TransRate.

These findings suggest a compelling structural advantage: while moderate noise can perturb local trajectories, severe corruption uniformly distorts the embedding space, causing static metrics to collapse. Because RADAR measures the \textit{relative angular divergence} of how representations evolve across layers, it remains remarkably invariant to extreme spatial distortions, preserving its ability to accurately predict model transferability even when the underlying data is heavily corrupted. Nevertheless, we note that when the perturbations are moderate, the metric's performance degrades. This indicates that RADAR is highly competitive when domain shifts are either minimal (Severity 1) or extreme (Severity 5), but exhibits vulnerabilities to intermediate, unstructured perturbations (Severity 3).

\begin{figure}[ht!]
    \centering
    \includegraphics[width=1\linewidth]{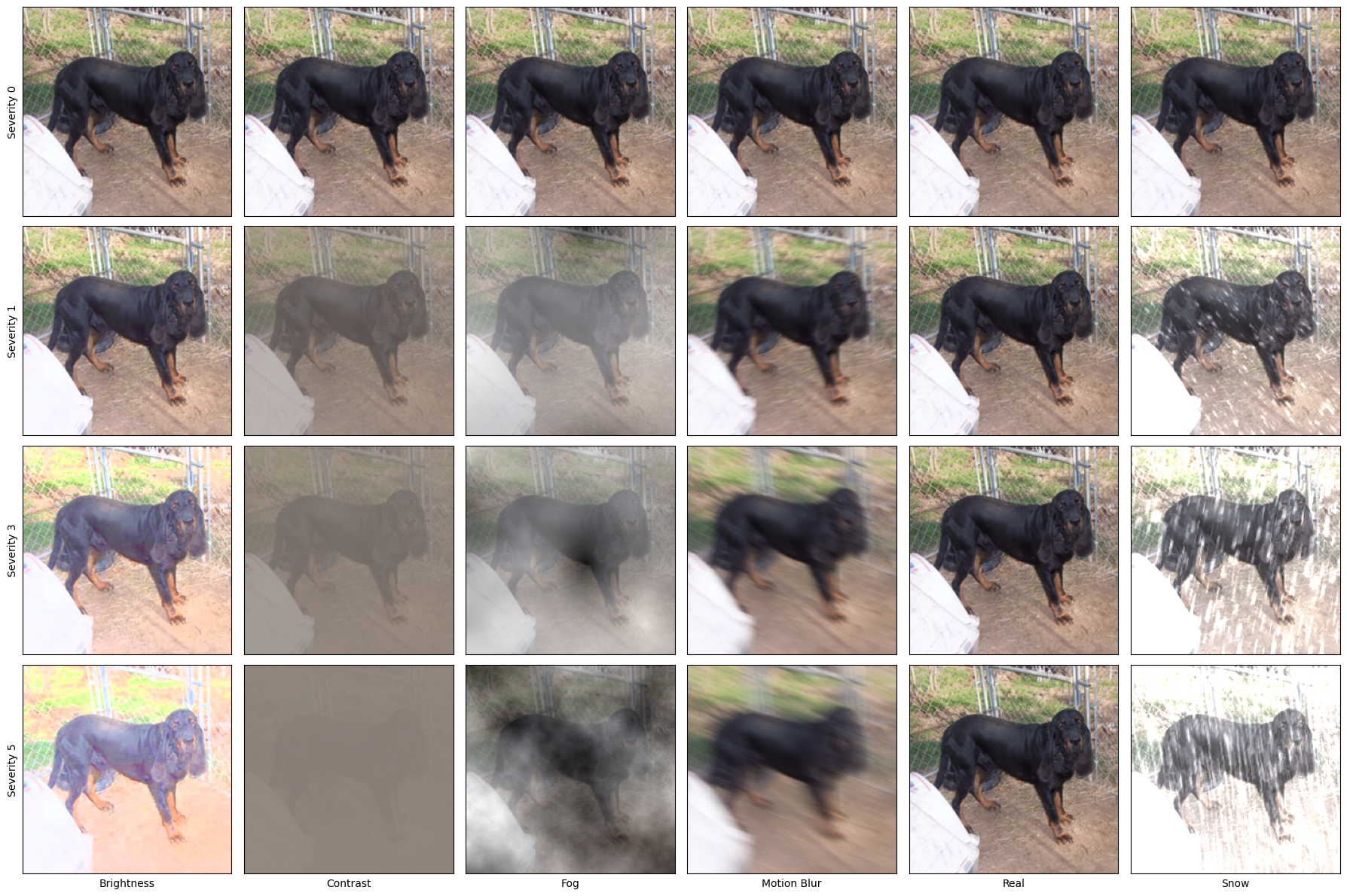}
    \caption{\textbf{Visualizing ImageNet-C synthetic corruptions.} Examples of various visual corruptions across increasing severities (0, 1, 3, and 5). The profound structural degradation at Severity 5 fundamentally scrambles patch-level token distributions in Vision Transformers. As shown in Table~\ref{tab:imagenet_c_robustness}, RADAR successfully tracks these chaotic geometric trajectories where most of the standard statistical metrics fail.}
    \label{fig:imagenet_c_severity}
\end{figure}

\section{Domain Separation --- Comparing ImageNet-C and other datasets}
\label{appendix:domain_separation}

Building upon the discussion in Appendix~\ref{appendix:imagenet_c_robustness}, we investigate the domain separation within the first-layer feature spaces of both CLIP and DINOv3. This analysis allows us to assess the domain separability inherent in standard benchmarks---DomainNet \cite{peng2019moment}, OfficeHome \cite{venkateswara2017deep}, and PACS \cite{li2017deeperbroaderartierdomain}---relative to the varying corruption severities of ImageNet-C. The resulting feature-space visualizations are presented in Figures~\ref{fig:layer0_vis_clip} and \ref{fig:layer0_vis_dino}.

\begin{figure}[ht!]
  \centering
  
  \subfloat[DomainNet]{
      \includegraphics[width=0.33\textwidth]{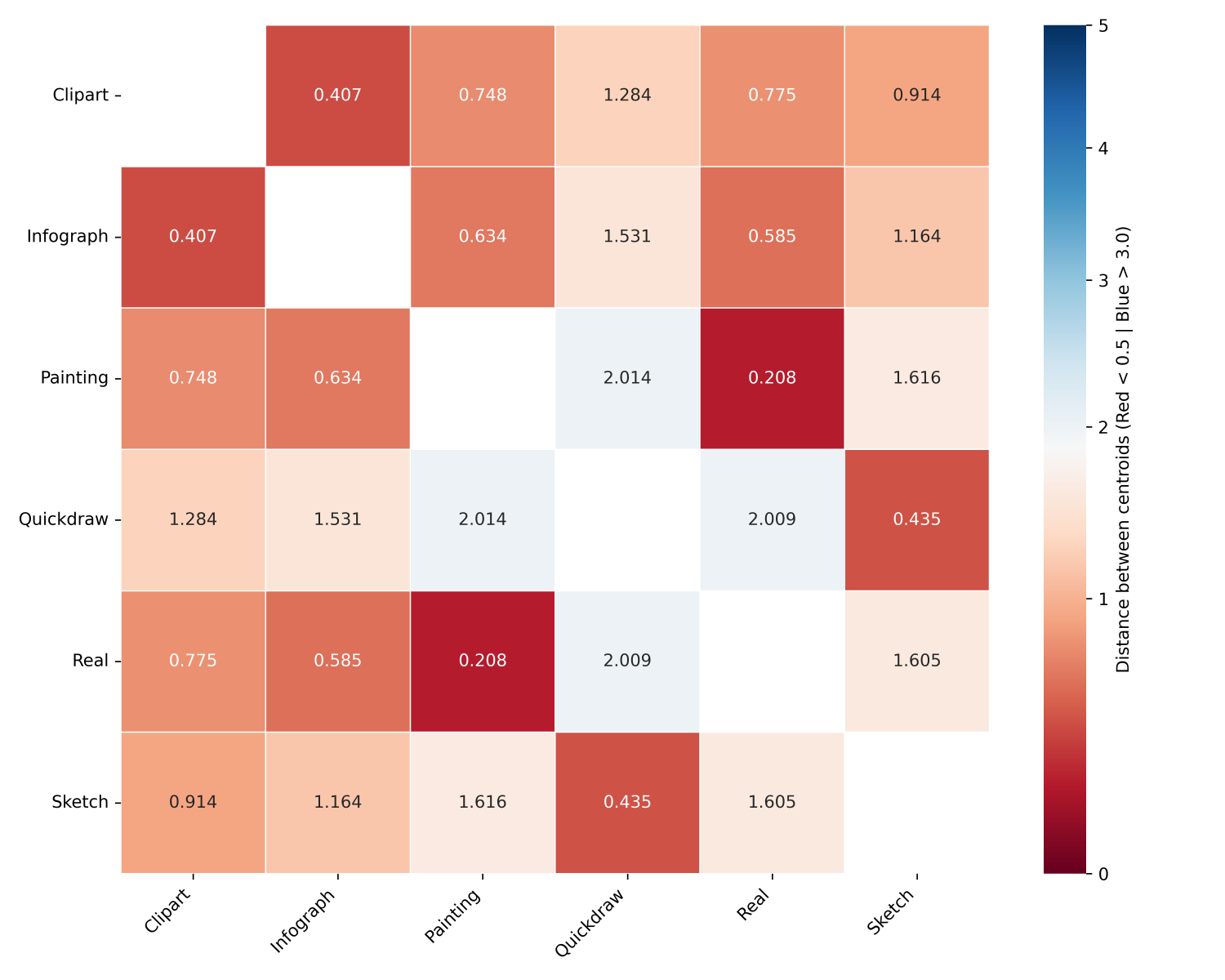}
      \label{fig:domainnet_l0} 
  }  
  \subfloat[OfficeHome]{
      \includegraphics[width=0.33\textwidth]{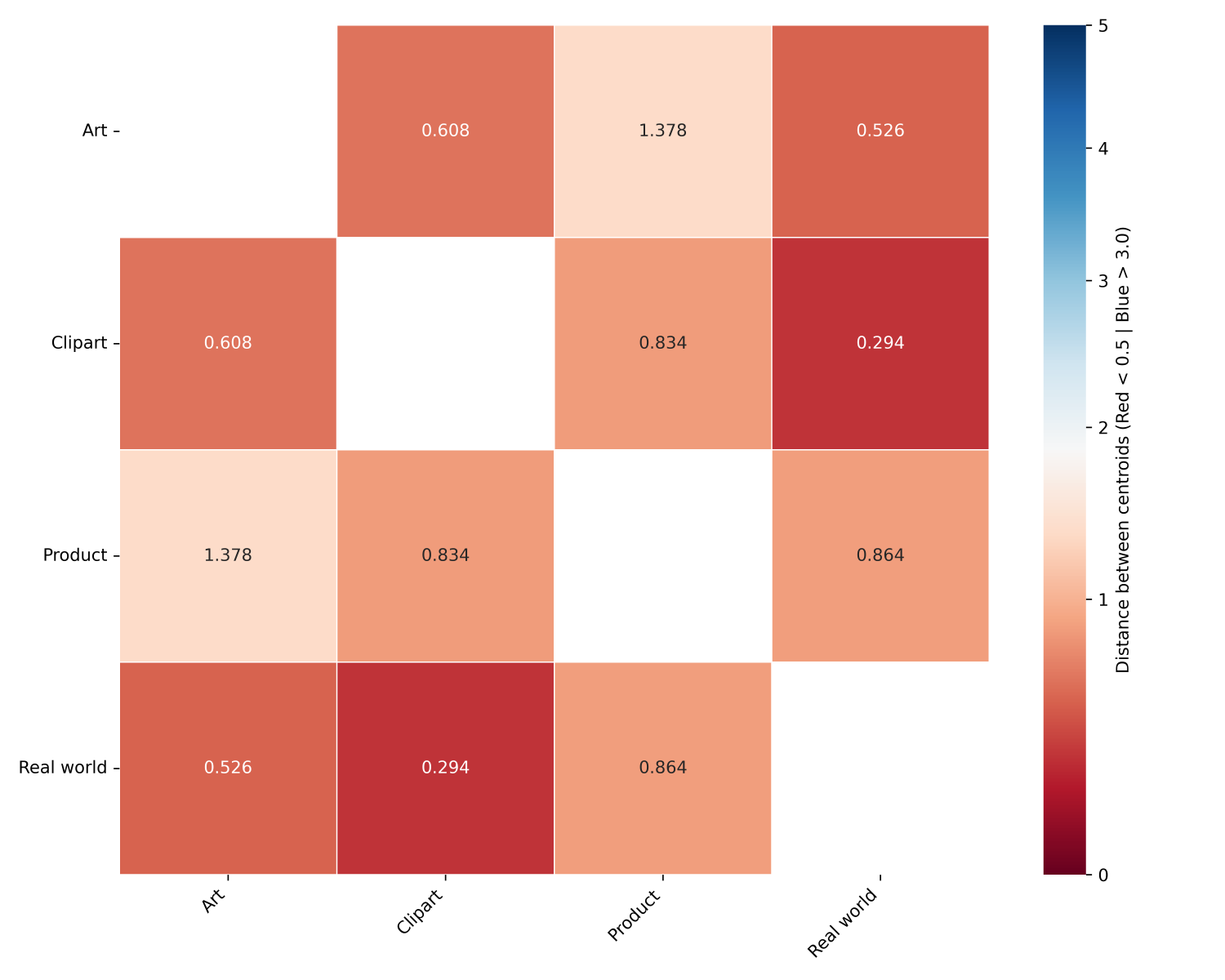}
      \label{fig:officehome_l0} 
  }
  \subfloat[PACS]{
      \includegraphics[width=0.33\textwidth]{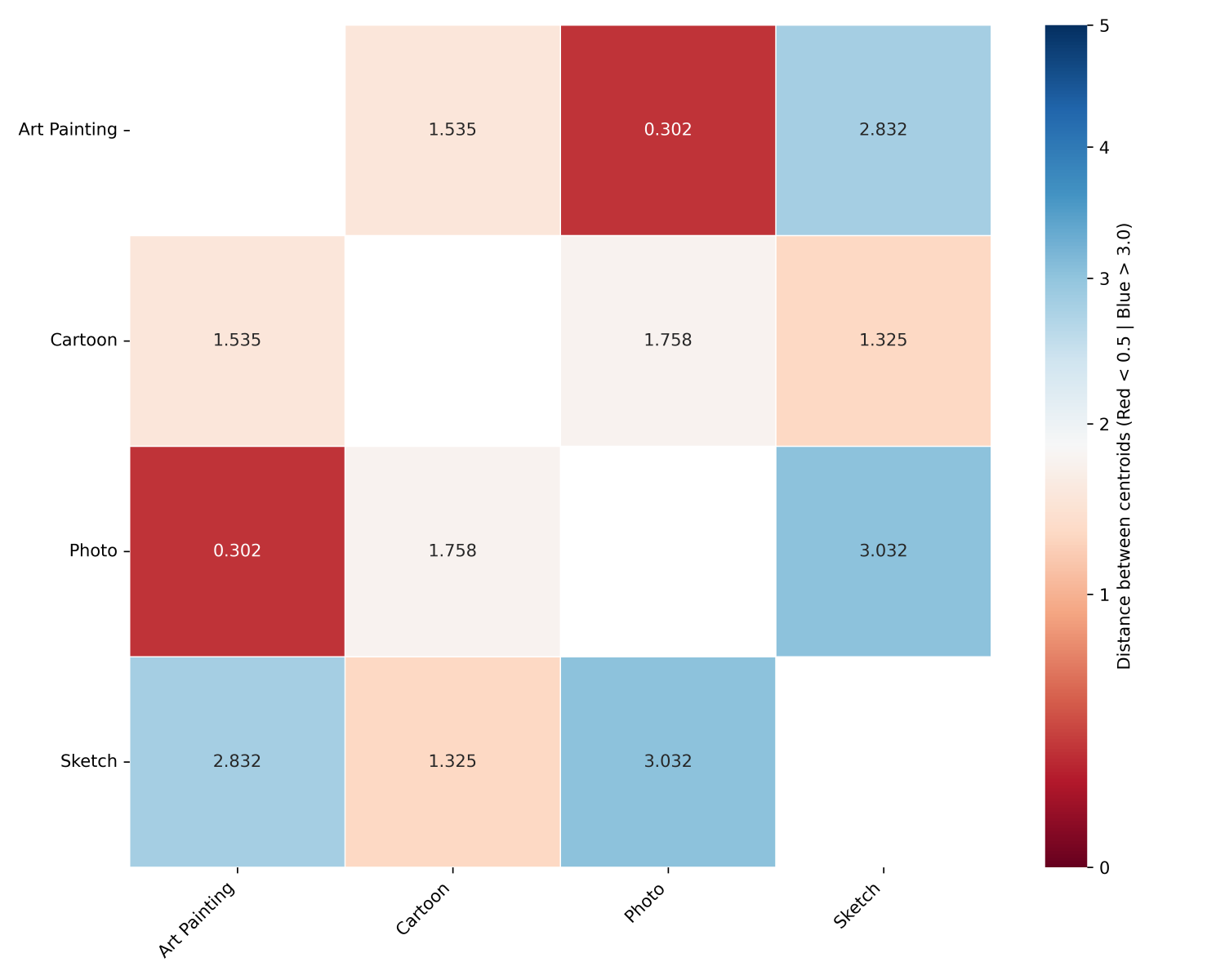}
      \label{fig:pacs_l0} 
  }
  
  \vspace{1em}
  
  \subfloat[ImageNet-C-1]{
      \includegraphics[width=0.33\textwidth]{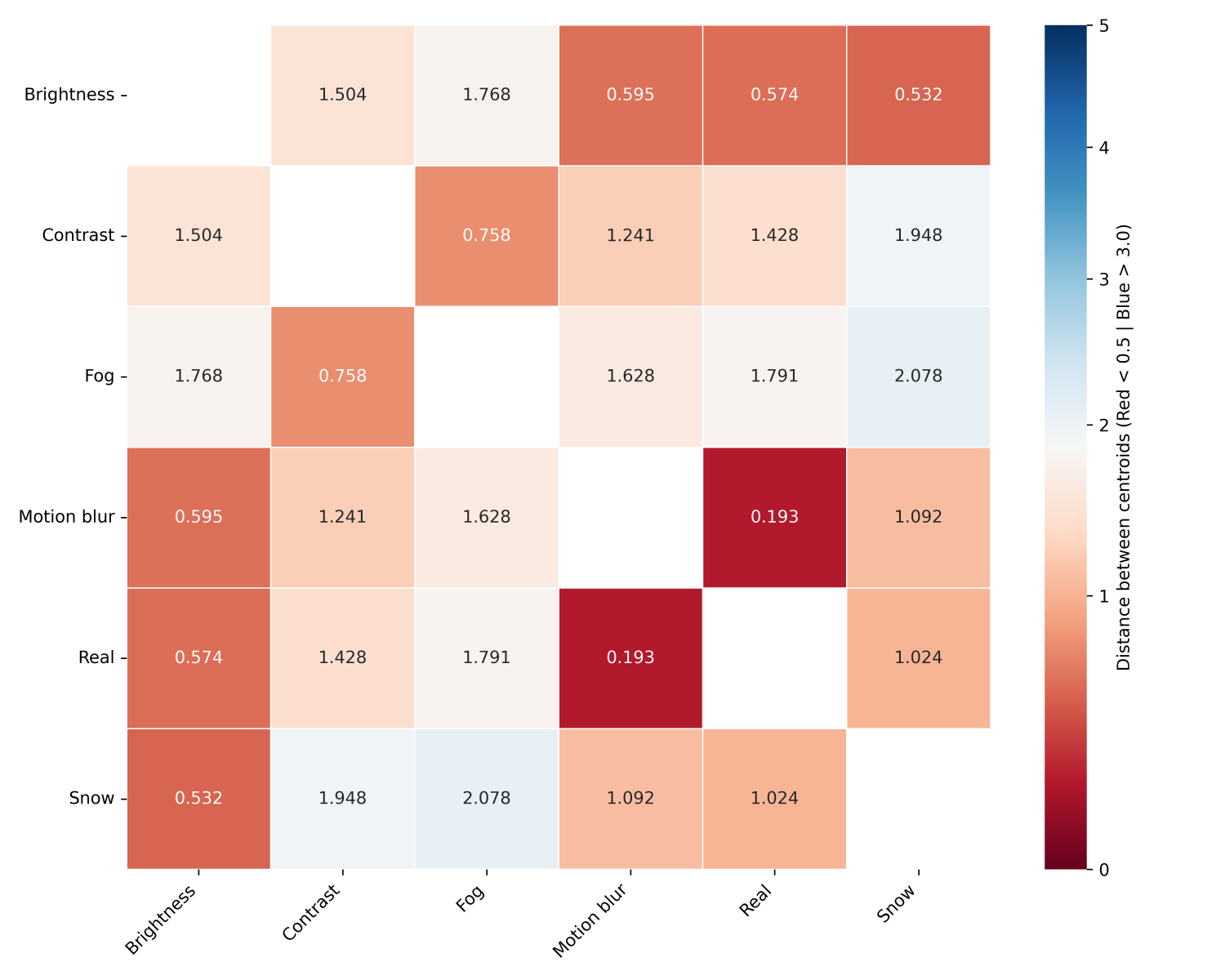}
      \label{fig:imagenet_c1_l0} 
  }  
  \subfloat[ImageNet-C-3]{
      \includegraphics[width=0.33\textwidth]{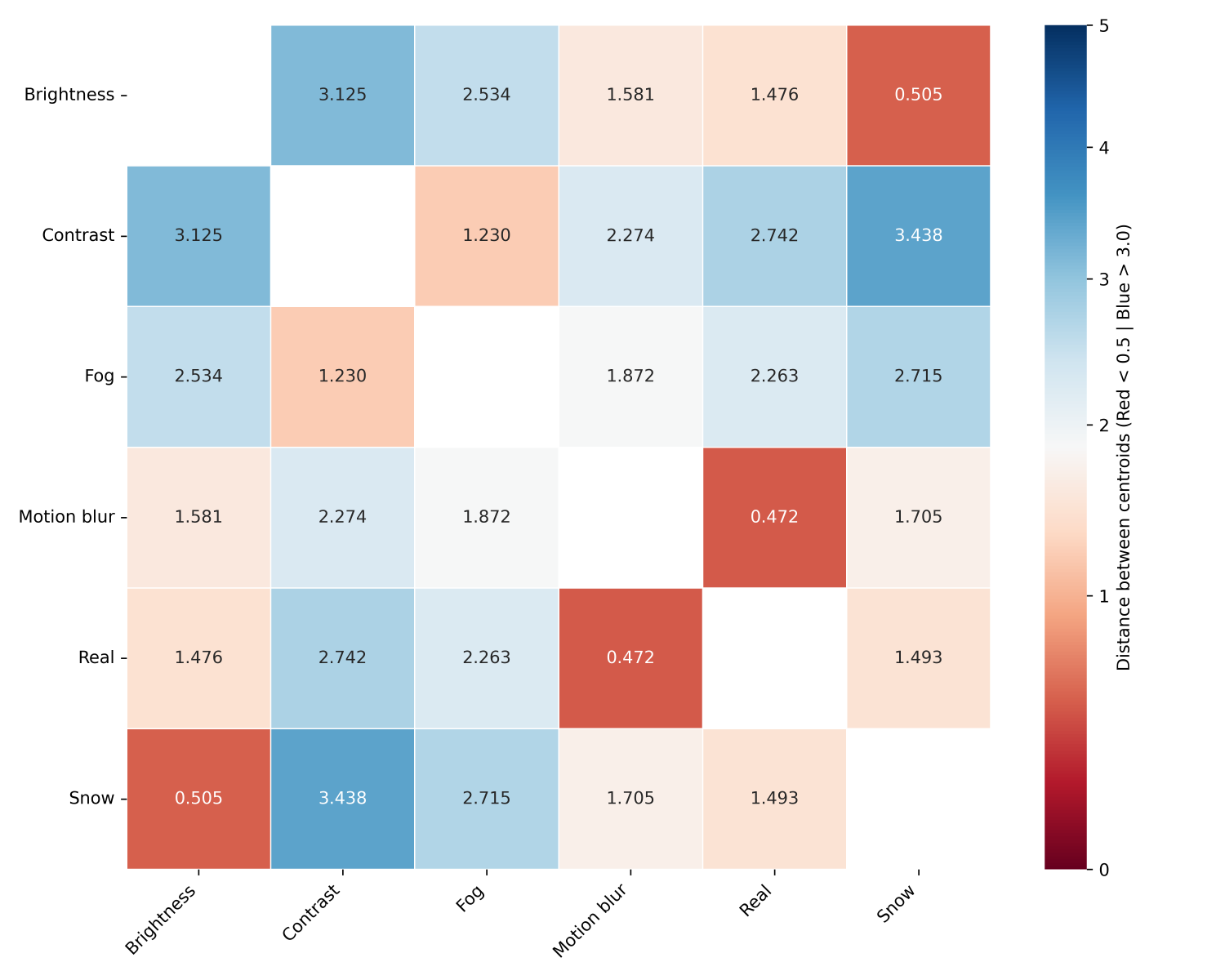}
      \label{fig:imagenet_c3_l0} 
  }
  \subfloat[ImageNet-C-5]{
      \includegraphics[width=0.33\textwidth]{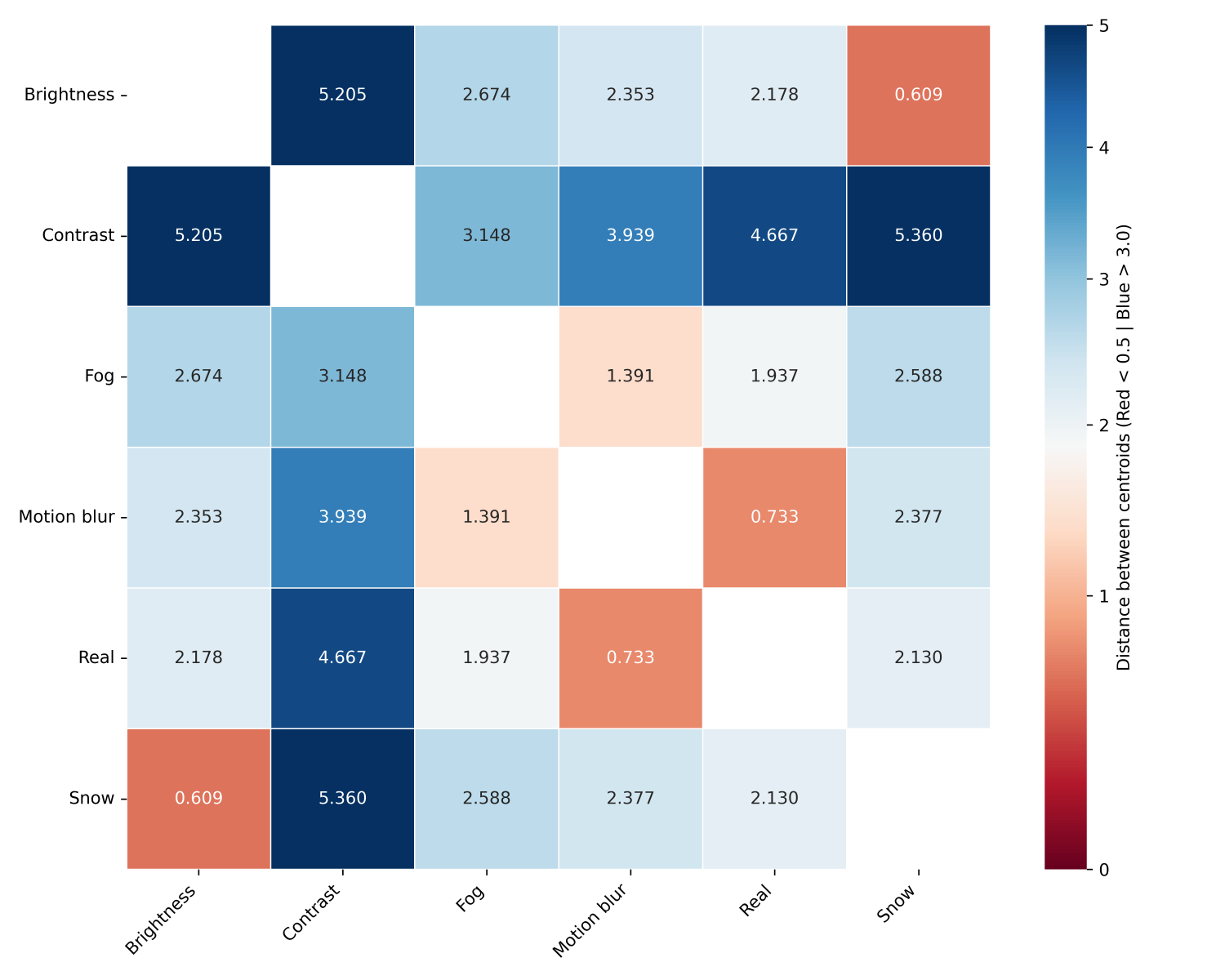}
      \label{fig:imagenet_c5_l0} 
  }
  
\caption{\textbf{Visualization of Layer 0 centroid distances across datasets for the CLIP model.} Distance matrices illustrate the spatial separation between domains right at the input embedding level. Warmer hues denote smaller centroid distances between domains, whereas cooler hues denote larger separations. At Layer 0, the domain distance distributions of DomainNet and OfficeHome closely mirror the low-severity corruptions of ImageNet-C-1. Conversely, PACS exhibits greater initial domain separation, more closely resembling the distribution seen in ImageNet-C-3.}
  \label{fig:layer0_vis_clip} 
\end{figure}

As demonstrated in Figure~\ref{fig:layer0_vis_clip}, the distances between domain feature centroids in DomainNet and OfficeHome closely resemble those of the low-severity corruptions in ImageNet-C (Severity 1). This indicates that the feature representations of these domains exhibit smooth transitions and modest distribution shifts. A notable exception is observed in the distances between the Quickdraw--Painting and Quickdraw--Real domain pairs; these demonstrate severe feature separation, mirroring the more pronounced spatial divergences characteristic of ImageNet-C-3 and ImageNet-C-5.

The PACS dataset, however, exhibits a markedly different pattern. Its inter-domain distances closely mirror those of ImageNet-C-3, indicating a moderate degree of domain separation. This offers a principled explanation for the reduced efficacy of our method on PACS: as the ImageNet-C experiments demonstrate, our approach performs optimally when domains either transition gradually or are highly separated in feature space. PACS appears to fall into an intermediate regime---analogous to ImageNet-C-3---characterized by partial overlap that is neither structured enough for smooth trajectory alignment nor distinct enough for clear separation.

\begin{figure}[ht!]
  \centering
  
  \subfloat[DomainNet]{
      \includegraphics[width=0.33\textwidth]{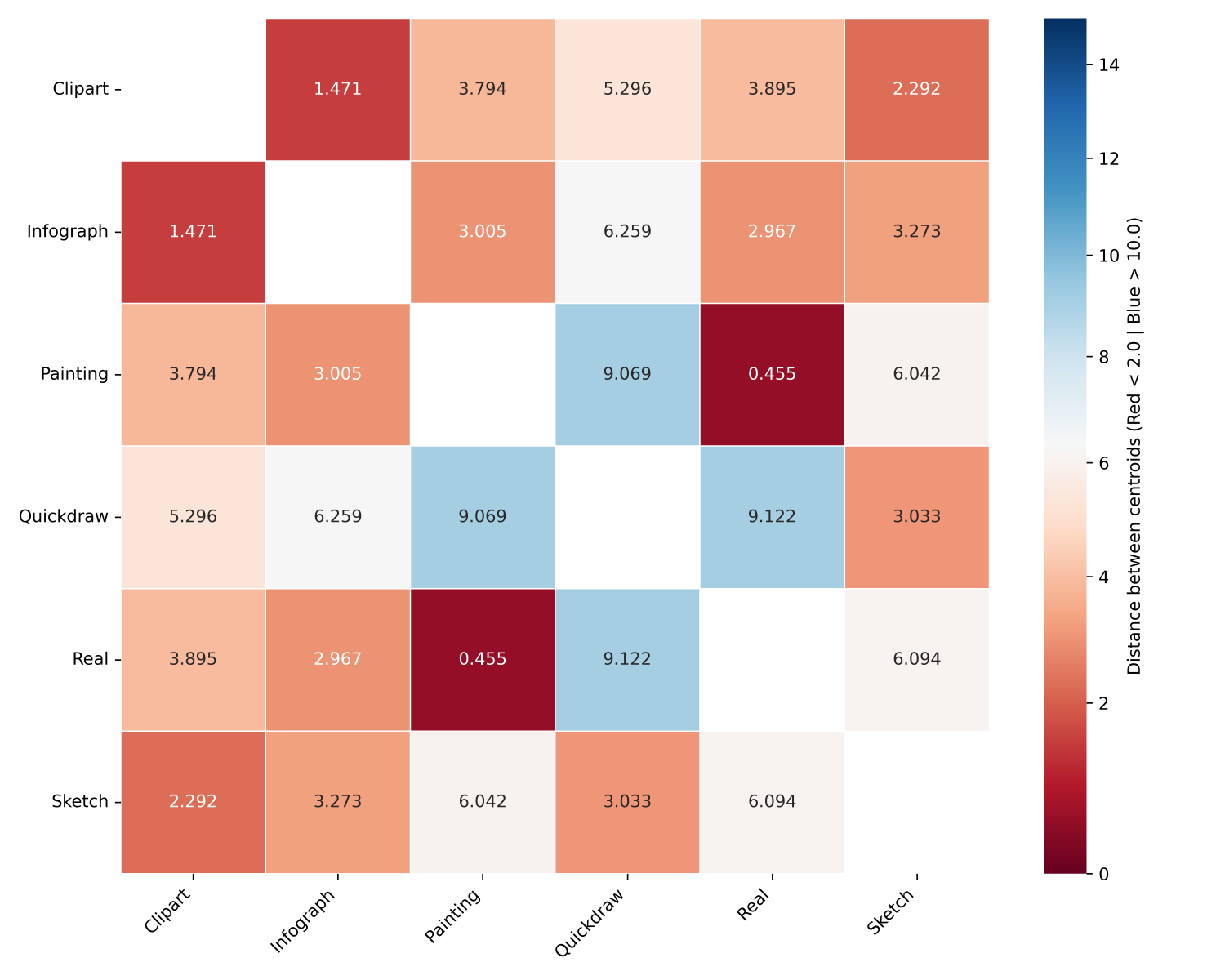}
      \label{fig:domainnet_l0_dino} 
  }  
  \subfloat[OfficeHome]{
      \includegraphics[width=0.33\textwidth]{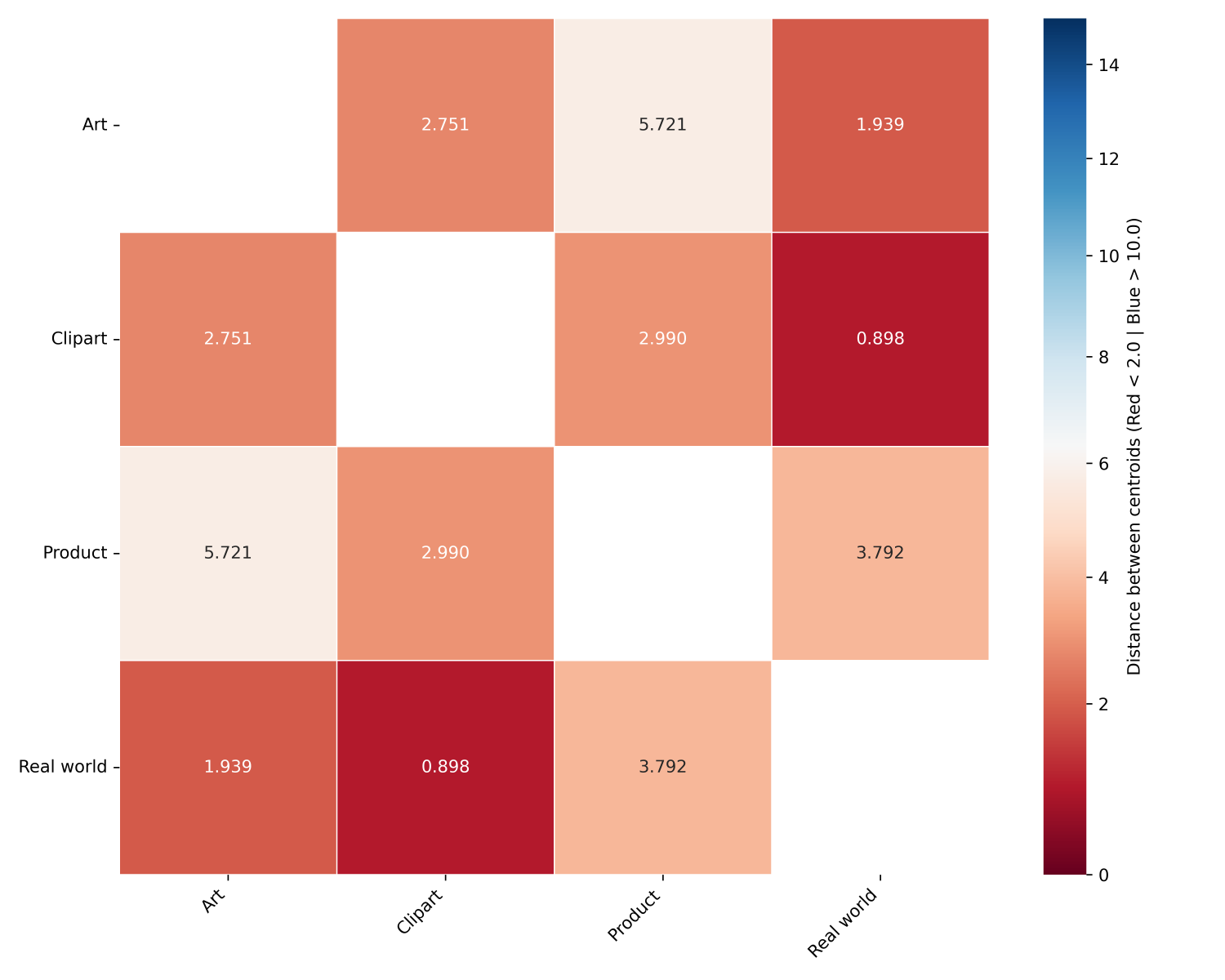}
      \label{fig:officehome_l0_dino} 
  }
  \subfloat[PACS]{
      \includegraphics[width=0.33\textwidth]{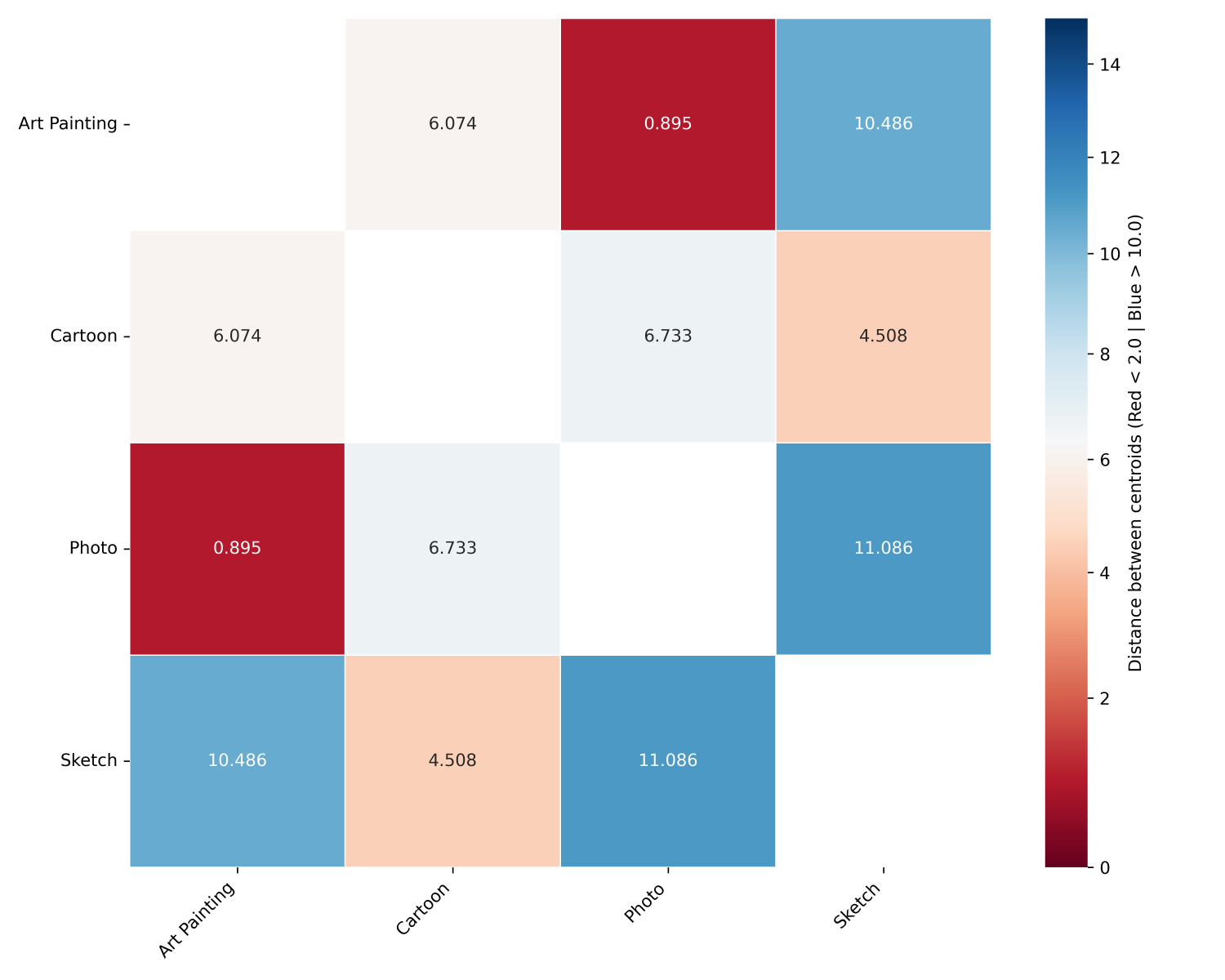}
      \label{fig:pacs_l0_dino} 
  }
  
  \vspace{1em}
  
  \subfloat[ImageNet-C-1]{
      \includegraphics[width=0.33\textwidth]{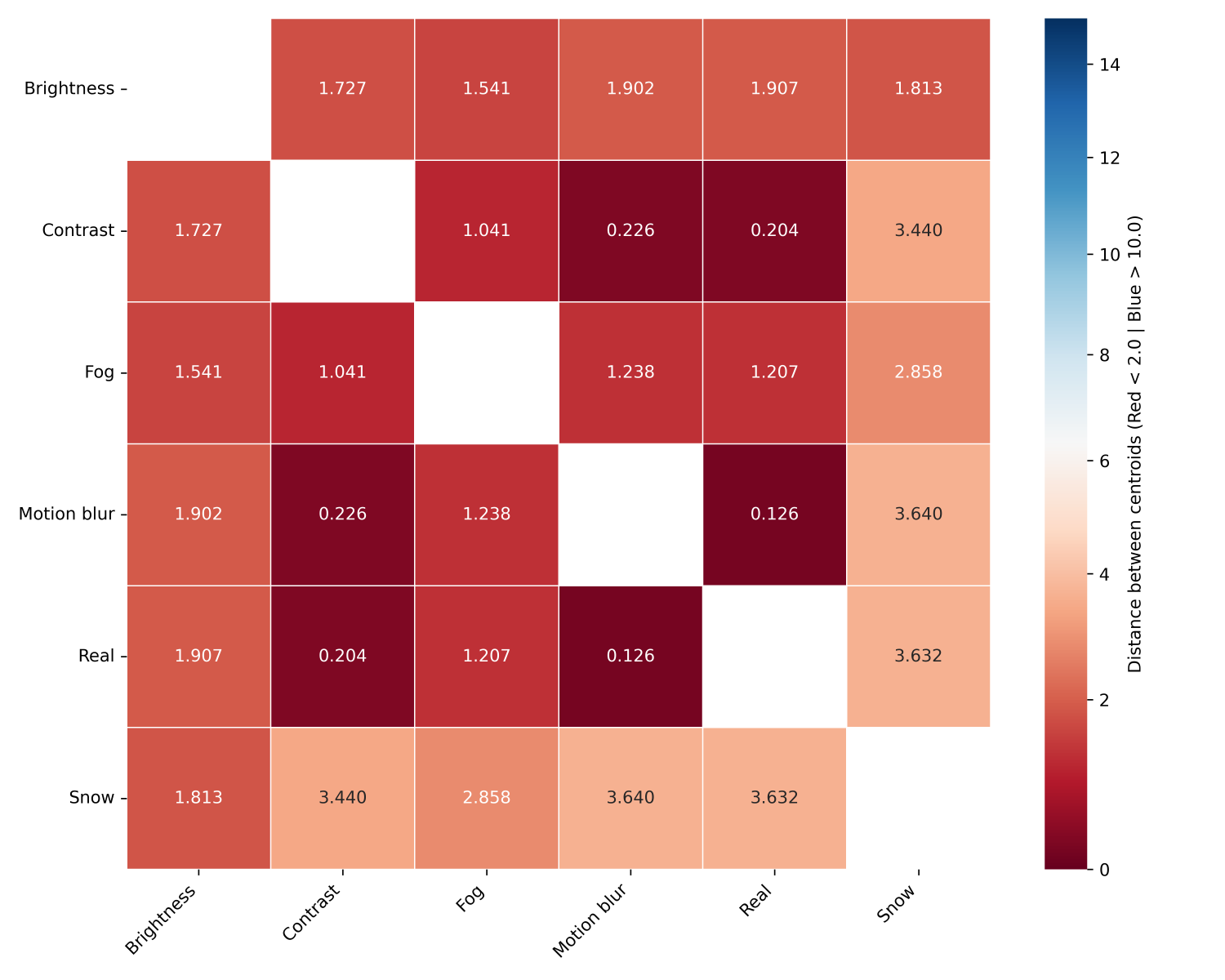}
      \label{fig:imagenet_c1_l0_dino} 
  }  
  \subfloat[ImageNet-C-3]{
      \includegraphics[width=0.33\textwidth]{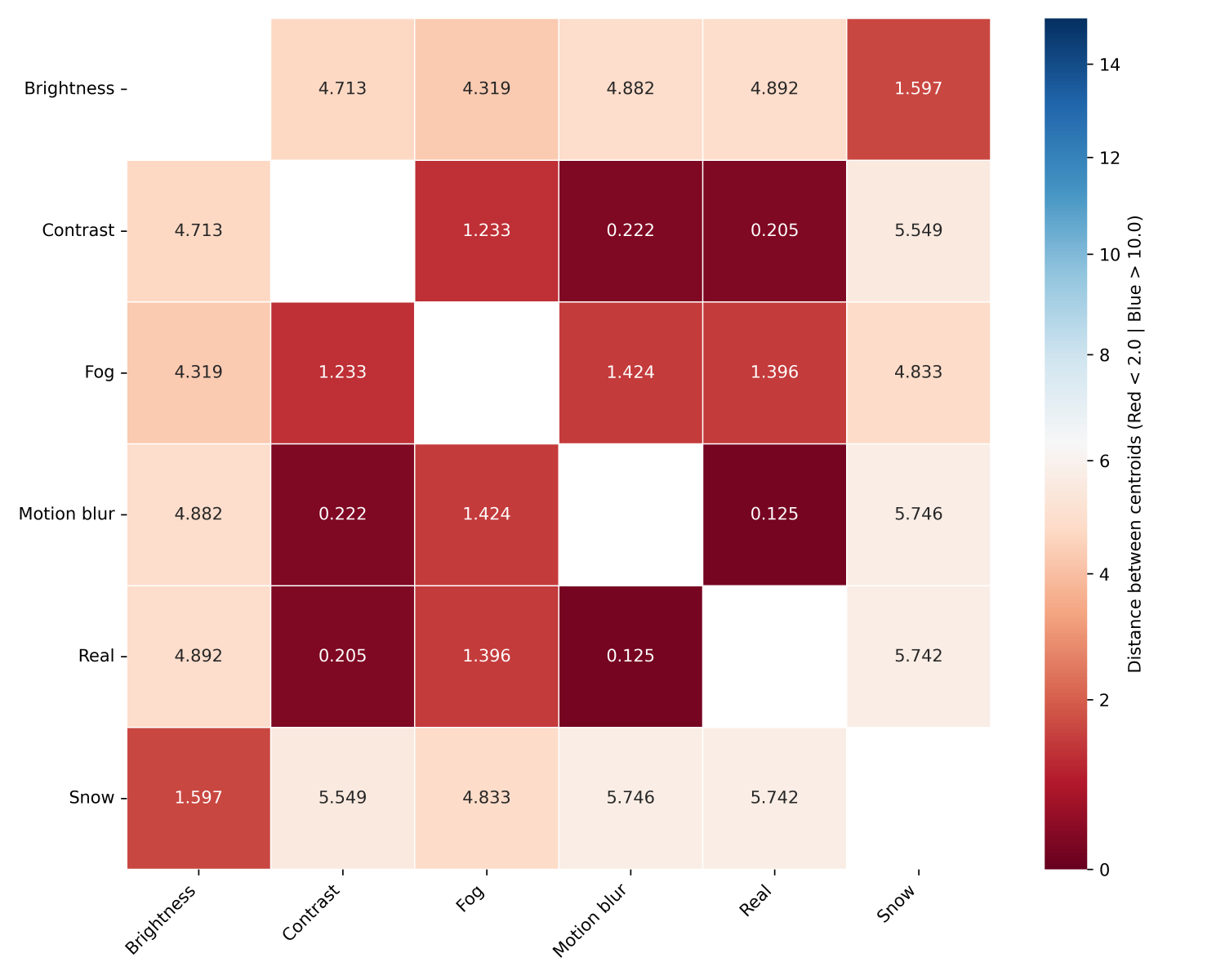}
      \label{fig:imagenet_c3_l0_dino} 
  }
  \subfloat[ImageNet-C-5]{
      \includegraphics[width=0.33\textwidth]{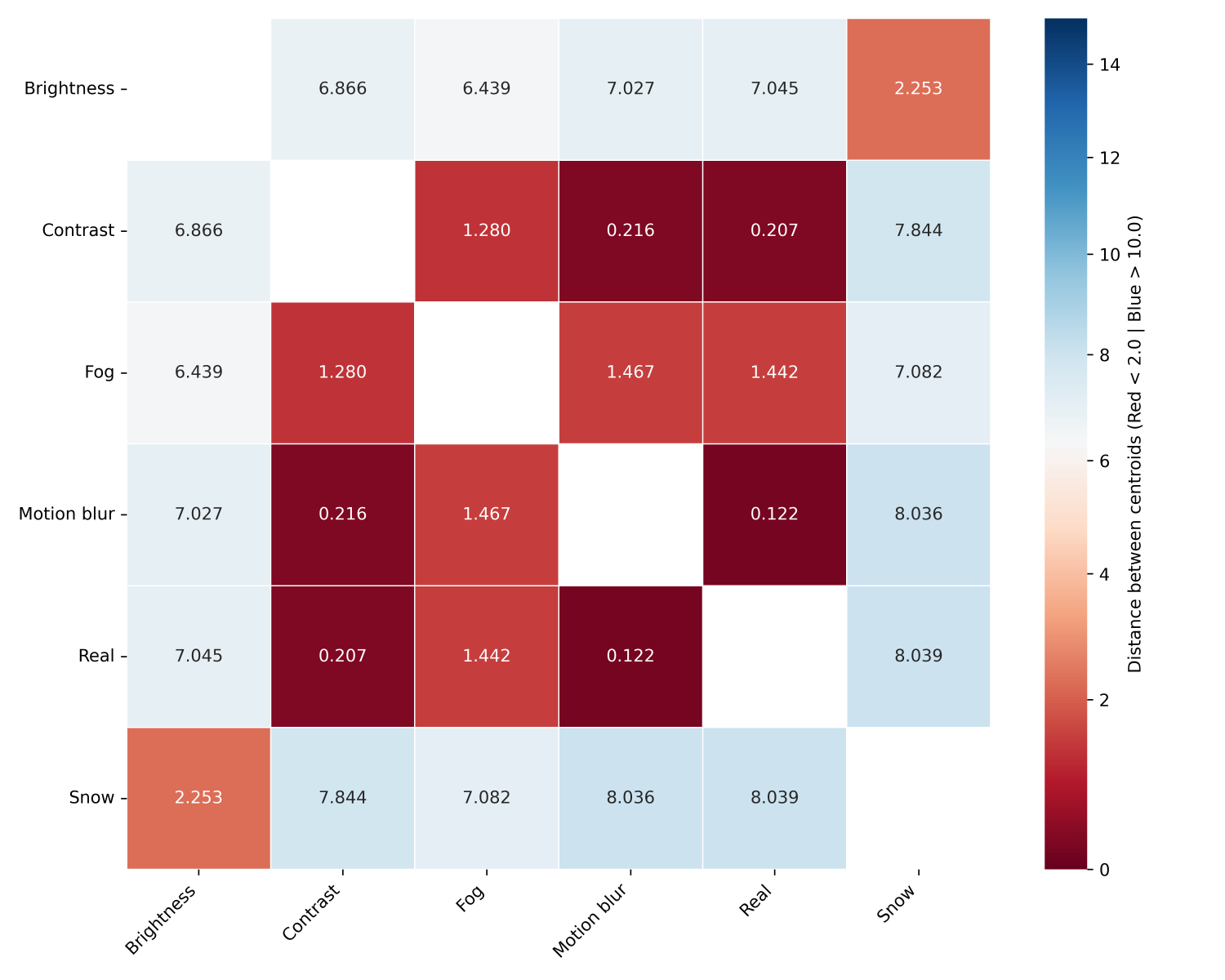}
      \label{fig:imagenet_c5_l0_dino} 
  }
  
    \caption{\textbf{Visualization of Layer 0 centroid distances across datasets for the DINOv3 model.} Distance matrices illustrate the spatial separation between domains right at the input embedding level. Warmer hues denote smaller centroid distances between domains, whereas cooler hues denote larger separations. At Layer 0, the domain distance distributions of DomainNet and OfficeHome closely mirror the ImageNet-C-1 and ImageNet-C-3 corruption severities. Conversely, PACS exhibits greater initial domain separation once more, closely resembling the distributions seen in ImageNet-C-3 and ImageNet-C-5.}
  \label{fig:layer0_vis_dino} 
\end{figure}

An analysis of the inter-feature distances in the first layer of DINOv3, presented in Figure~\ref{fig:layer0_vis_dino}, corroborates our earlier findings in the CLIP feature space. We observe that DomainNet and OfficeHome features exhibit behavior consistent with the lower ImageNet-C corruption severities (1 and 3). In contrast, PACS exhibits a broader variance, resembling the distributions associated with severities 3 and 5. This pattern suggests that domain separation in PACS is inherently challenging regardless of the underlying architecture. Specifically, it consistently exhibits the ambiguous, intermediate spatial overlaps (analogous to Severity 3) observed in the CLIP setting---a regime in which our method demonstrably struggles.

Critically, the scope of the separability condition is diagnosable in advance from Layer 0 centroid distances alone, suggesting a practical pre-screening step: compute inter-domain separation at the input embedding level before selecting a transferability metric.

\section{Ablation Studies}
\label{section:ablation}

\subsection{Feature Representation and Standardization}
\label{sec:ablation_features}

We conduct an extensive ablation study to isolate the impact of specific geometric features and representation choices within RADAR. Table~\ref{tab:combined_features} presents the results of these geometric design choices across vision and text modalities. 

\begin{table}[ht!]
    \centering
    \caption{Feature ablations for vision (DomainNet) and text (EuroEval) representations. Both modalities exhibit a strong preference for combining Distance and Angle features, though text representations show heavy reliance on angular displacement, and only text benefits from layer-wise standardization.}
    \label{tab:combined_features}
    
    \begin{subtable}[t]{\textwidth}
        \caption{Ablation study on the vision models CLIP \cite{radford2021learning} and DINOv3 \cite{simoni2025dinov3}. We observe both angles and distances are crucial.}
        \label{tab:img_feat_baseline}
        \resizebox{\textwidth}{!}{%
        \begin{tabular}{l cccccccc ccc}
        \toprule
         & \multicolumn{8}{c}{\textbf{Design choices}} & \multicolumn{3}{c}{\textbf{Mean Correlation Improvement (\%pt.)} $\downarrow$} \\ 
         \cmidrule(lr){2-9} \cmidrule(l){10-12} 
         & \textbf{Distance} & \textbf{Angle} & \textbf{Standardize} & \textbf{Sampling} & \textbf{Replace.} & \textbf{Space} & \textbf{Covariance} & \textbf{Algorithm} & \textbf{CLIP} & \textbf{DINOv3} & \textbf{Total} \\ 
        \midrule
        \multirow{4}{*}{\textbf{Features}} 
         & \xmark & \cmark & \xmark & Uniform & \xmark & Euclidean & Diag & GMM + KL & -2.66           & -5.62          & -8.28 \\
         & \cmark & \xmark & \xmark & Uniform & \xmark & Euclidean & Diag & GMM + KL & 1.34            & -7.96          & -6.62 \\
         & \cmark & \cmark & \xmark & Uniform & \xmark & Euclidean & Diag & GMM + KL & \textbf{-10.25} & {\ul -7.48}    & \textbf{-17.73} \\
         & \cmark & \cmark & \cmark & Uniform & \xmark & Euclidean & Diag & GMM + KL & {\ul -8.85 }    & \textbf{-8.31} & {\ul -17.16} \\ 
        \bottomrule
        \end{tabular}%
        }
    \end{subtable}
    
    \vspace{1.5em} 
    
    \begin{subtable}[t]{\textwidth}
        \centering
        \caption{Ablation study on the text models Qwen3-Embedding \cite{zhang2025qwen3} and EmbeddingGemma \cite{vera2025embeddinggemma}. We observe standardization provides a massive boost to text modalities, unlike vision models.}
        \label{tab:lang_feat_baseline}
        \resizebox{\textwidth}{!}{%
        \begin{tabular}{l cccccccc ccc}
        \toprule
         & \multicolumn{8}{c}{\textbf{Design choices}} & \multicolumn{3}{c}{\textbf{Mean Correlation Improvement (\%pt.)} $\downarrow$} \\ 
         \cmidrule(lr){2-9} \cmidrule(l){10-12} 
         & \textbf{Distance} & \textbf{Angle} & \textbf{Standardize} & \textbf{Sampling} & \textbf{Replace.} & \textbf{Space} & \textbf{Covariance} & \textbf{Algorithm} & \textbf{Qwen3-Embedding} & \textbf{EmbeddingGemma} & \textbf{Total} \\ 
        \midrule
        \multirow{4}{*}{\textbf{Features}} 
         & \xmark & \cmark & \xmark & Uniform & \xmark & Euclidean & Diag & GMM + KL & {\ul -5.68}    & -19.42          & -25.10 \\
         & \cmark & \xmark & \xmark & Uniform & \xmark & Euclidean & Diag & GMM + KL & 7.42           & -12.10          & -4.68 \\
         & \cmark & \cmark & \xmark & Uniform & \xmark & Euclidean & Diag & GMM + KL & -4.24          & {\ul -29.29}    & {\ul -33.53} \\
         & \cmark & \cmark & \cmark & Uniform & \xmark & Euclidean & Diag & GMM + KL & \textbf{-6.71} & \textbf{-30.75} & \textbf{-37.46} \\ 
        \bottomrule
        \end{tabular}%
        }
    \end{subtable}
\end{table}

We observe a universal synergy between distance and angular features for both modalities. Omitting either the angles or relative distances leads to severe performance degradation, with distance-only features even negatively impacting the CLIP and Qwen3-Embedding models. Notably, text representations exhibit a disproportionate reliance on angular features, yielding a total score of -25.10 alone compared to -4.68 for distance. This aligns with the widespread use of cosine-based contrastive objectives in language modeling, which inherently prioritizes angular topology. 

Furthermore, we uncover a modality-specific divergence regarding feature standardization. While vision models suffer slightly from standardizing distance and angle scales, indicating that raw magnitude variance can contain useful transferability signals, text models seem to benefit from it. Standardizing text representations neutralizes magnitude noise, allowing the combined distance and angle metric to achieve the optimal overall score of -37.46, improving the MCI by roughly 4 percentage points over the unstandardized version.
Despite the minor performance drop in vision models, we opt to apply feature standardization to both modalities to maintain a unified, modality-agnostic framework, as the substantial gains in text evaluation far outweigh the slight degradation in vision.

\subsection{Pairwise Sampling Strategies}
\label{sec:ablation_sampling}

Next, we evaluate the pairwise sampling used in our density estimation of angles and relative distances. Table~\ref{tab:combined_sampling} details the ablation comparing uniform sampling, positive (inlier-inlier) pair sampling, negative (inlier-outlier) pair sampling, and a mixture of both.

\begin{table}[ht!]
    \centering
    \caption{Sampling ablations for vision (DomainNet) and text (EuroEval) representations. Both modalities demonstrate that ``Mix'' sampling --- incorporating both positive (inlier-inlier) and negative (inlier-outlier) pairs --- yields the strongest predictive correlation, proving that mapping both class centers and class boundaries is essential.}
    \label{tab:combined_sampling}
    
    \begin{subtable}[t]{\textwidth}
        \caption{Ablation study on the vision models CLIP \cite{radford2021learning} and DINOv3 \cite{simoni2025dinov3}. We observe that Mixed sampling with replacement yields a marginal peak in total performance.}
        \label{tab:img_samp_baseline}
        \resizebox{\textwidth}{!}{%
        \begin{tabular}{l cccccccc ccc}
        \toprule
         & \multicolumn{8}{c}{\textbf{Design choices}} & \multicolumn{3}{c}{\textbf{Mean Correlation Improvement (\%pt.)} $\downarrow$} \\ 
         \cmidrule(lr){2-9} \cmidrule(l){10-12} 
         & \textbf{Distance} & \textbf{Angle} & \textbf{Standardize} & \textbf{Sampling} & \textbf{Replace.} & \textbf{Space} & \textbf{Covariance} & \textbf{Algorithm} & \textbf{CLIP} & \textbf{DINOv3} & \textbf{Total} \\ 
        \midrule
        \multirow{4}{*}{\textbf{Sampling}} 
         & \cmark & \cmark & \cmark & Positive & \xmark & Euclidean & Diag & GMM + KL & {\ul -9.25}     & -8.76           & -18.01 \\
         & \cmark & \cmark & \cmark & Negative & \xmark & Euclidean & Diag & GMM + KL & -8.45           & -8.65           & -17.10 \\
         & \cmark & \cmark & \cmark & Mix      & \xmark & Euclidean & Diag & GMM + KL & \textbf{-9.72}  & {\ul -9.87}     & {\ul -19.59} \\
         & \cmark & \cmark & \cmark & Mix      & \cmark & Euclidean & Diag & GMM + KL & -8.03           & \textbf{-11.58} & \textbf{-19.61} \\ 
        \bottomrule
        \end{tabular}%
        }
    \end{subtable}
    
    \vspace{1.5em} 
    
    \begin{subtable}[t]{\textwidth}
        \centering
        \caption{Ablation study on the text models Qwen3-Embedding \cite{zhang2025qwen3} and EmbeddingGemma \cite{vera2025embeddinggemma}. Mixed sampling significantly outperforms isolated sampling, with ``no replacement'' yielding the best results.}
        \label{tab:lang_samp_baseline}
        \resizebox{\textwidth}{!}{%
        \begin{tabular}{l cccccccc ccc}
        \toprule
         & \multicolumn{8}{c}{\textbf{Design choices}} & \multicolumn{3}{c}{\textbf{Mean Correlation Improvement (\%pt.)} $\downarrow$} \\ 
         \cmidrule(lr){2-9} \cmidrule(l){10-12} 
         & \textbf{Distance} & \textbf{Angle} & \textbf{Standardize} & \textbf{Sampling} & \textbf{Replace.} & \textbf{Space} & \textbf{Covariance} & \textbf{Algorithm} & \textbf{Qwen3-Embedding} & \textbf{EmbeddingGemma} & \textbf{Total} \\ 
        \midrule
        \multirow{4}{*}{\textbf{Sampling}} 
         & \cmark & \cmark & \cmark & Positive & \xmark & Euclidean & Diag & GMM + KL & -7.03          & {\ul -29.51}    & {\ul -36.54 } \\
         & \cmark & \cmark & \cmark & Negative & \xmark & Euclidean & Diag & GMM + KL & -6.85          & -28.19          & -35.04 \\
         & \cmark & \cmark & \cmark & Mix      & \xmark & Euclidean & Diag & GMM + KL & {\ul -7.13}    & \textbf{-30.58} & \textbf{-37.71} \\
         & \cmark & \cmark & \cmark & Mix      & \cmark & Euclidean & Diag & GMM + KL & \textbf{-7.28} & -28.40          & -35.68 \\
        \bottomrule
        \end{tabular}%
        }
    \end{subtable}
\end{table}

The results clearly demonstrate that the ``Mix'' strategy strictly dominates isolated sampling across both vision and text modalities. Relying solely on positive pairs captures only intra-class compactness, whereas relying on negative pairs maps only inter-class separation. By integrating both, RADAR comprehensively captures the topological structure of the representation space---mapping both class centers and decision boundaries---which yields the strongest predictive correlations overall, surpassing uniform sampling.

Furthermore, we evaluate the effect of sampling pairs with versus without replacement. Text modalities exhibit a strict preference for sampling without replacement, achieving an optimal total score of -37.71 compared to -35.68 with replacement, improving the MCI by roughly 2 percentage points. In those models, sampling with replacement risks reducing the geometric diversity of the sampled trajectories. Conversely, vision models exhibit high variance. While DINOv3 benefits from replacement, CLIP's performance degrades significantly. Given these dynamics, we conclude that mixed sampling \textit{without} replacement offers the most stable and universally robust configuration across diverse foundation models.

\subsection{Topological Space Formulation}
\label{sec:ablation_space}

As a subsequent evaluation, we analyze the topological space in which these angles and relative distances are computed. Table~\ref{tab:combined_space} compares the effects of computing the representations in Euclidean, Geodesic, and pseudo-Cartesian spaces prior to density estimation. 

\begin{table}[ht!]
    \centering
    \caption{Topological space ablations for vision (DomainNet) and text (EuroEval) representations. The data reveals a clear cross-modal divide: vision representations achieve peak predictability in pseudo-Cartesian space, whereas text representations heavily favor Euclidean space. However, Euclidean representation seems to be the best middle-ground.}
    \label{tab:combined_space}
    
    \begin{subtable}[t]{\textwidth}
        \caption{Ablation study on the vision models CLIP \cite{radford2021learning} and DINOv3 \cite{simoni2025dinov3}. The Cartesian projection yields the best overall performance, largely driven by DINOv3.}
        \label{tab:img_space_baseline}
        \resizebox{\textwidth}{!}{%
        \begin{tabular}{l cccccccc ccc}
        \toprule
         & \multicolumn{8}{c}{\textbf{Design choices}} & \multicolumn{3}{c}{\textbf{Mean Correlation Improvement (\%pt.)} $\downarrow$} \\ 
         \cmidrule(lr){2-9} \cmidrule(l){10-12} 
         & \textbf{Distance} & \textbf{Angle} & \textbf{Standardize} & \textbf{Sampling} & \textbf{Replace.} & \textbf{Space} & \textbf{Covariance} & \textbf{Algorithm} & \textbf{CLIP} & \textbf{DINOv3} & \textbf{Total} \\ 
        \midrule
        \multirow{3}{*}{\textbf{Space}}
         & \cmark & \cmark & \cmark & Mix      & \xmark & Euclidean & Diag & GMM + KL & \textbf{-9.72} & -9.87           & {\ul -19.59} \\
         & \cmark & \cmark & \cmark & Mix      & \xmark & Geodesic  & Diag & GMM + KL & -6.74          & {\ul -11.37}    & -18.11 \\
         & \cmark & \cmark & \cmark & Mix      & \xmark & Cartesian & Diag & GMM + KL & {\ul -8.52}    & \textbf{-12.99} & \textbf{-21.51} \\ 
        \bottomrule
        \end{tabular}%
        }
    \end{subtable}
    
    \vspace{1.5em} 
    
    \begin{subtable}[t]{\textwidth}
        \centering
        \caption{Ablation study on the text models Qwen3-Embedding \cite{zhang2025qwen3} and EmbeddingGemma \cite{vera2025embeddinggemma}. Both text foundation models strongly degrade when moved away from Euclidean space.}
        \label{tab:lang_space_baseline}
        \resizebox{\textwidth}{!}{%
        \begin{tabular}{l cccccccc ccc}
        \toprule
         & \multicolumn{8}{c}{\textbf{Design choices}} & \multicolumn{3}{c}{\textbf{Mean Correlation Improvement (\%pt.)} $\downarrow$} \\ 
         \cmidrule(lr){2-9} \cmidrule(l){10-12} 
         & \textbf{Distance} & \textbf{Angle} & \textbf{Standardize} & \textbf{Sampling} & \textbf{Replace.} & \textbf{Space} & \textbf{Covariance} & \textbf{Algorithm} & \textbf{Qwen3-Embedding} & \textbf{EmbeddingGemma} & \textbf{Total} \\ 
        \midrule
        \multirow{3}{*}{\textbf{Space}}
         & \cmark & \cmark & \cmark & Mix      & \xmark & Euclidean & Diag & GMM + KL & \textbf{-7.13} & \textbf{-30.58} & \textbf{-37.71} \\
         & \cmark & \cmark & \cmark & Mix      & \xmark & Geodesic  & Diag & GMM + KL & -4.79          & {\ul -25.60}    & -30.39 \\
         & \cmark & \cmark & \cmark & Mix      & \xmark & Cartesian & Diag & GMM + KL & {\ul -5.86}    & -24.31          & -30.17 \\
        \bottomrule
        \end{tabular}%
        }
    \end{subtable}
\end{table}

The data shows a stark cross-modal divide regarding the optimal underlying geometry. Text models exhibit a strict reliance on Euclidean space, achieving an optimal total MCI of -37.71, while projecting text embeddings into Geodesic or Cartesian spaces results in massive performance degradation.

Conversely, vision models achieve their highest total performance in a pseudo-Cartesian space. However, this total is heavily skewed by DINOv3, which uniquely thrives in Cartesian projection. CLIP, much like the text models, actually favors Euclidean space over the alternatives.

When establishing a unified, modality-agnostic configuration for RADAR, the Euclidean space emerges as the clear optimal middle-ground. It serves as the absolute best geometric space for both text models and CLIP, while remaining highly competitive for DINOv3. The severe degradation text models experience in Cartesian and Geodesic spaces far outweighs the isolated gains observed in DINOv3, making Euclidean space the most robust universal standard.

\subsection{GMM Covariance Parameterization}
\label{sec:ablation_covariance}
To evaluate how potential correlations between the angle and relative distance distributions impact density estimation, we ablate the design choice of the GMM covariance matrix. Table~\ref{tab:combined_covariance} presents this ablation for both vision and text representations, comparing Diagonal, Full, Tied, and Spherical parameterizations. 

\begin{table}[ht!]
    \centering
    \caption{GMM covariance ablations for vision (DomainNet) and text (EuroEval) representations. Both modalities show a strict preference for Diagonal covariance parameterization, avoiding the overfitting associated with Full covariance and the over-restriction of Spherical covariance.}
    \label{tab:combined_covariance}
    
    \begin{subtable}[t]{\textwidth}
        \caption{Ablation study on the vision models CLIP \cite{radford2021learning} and DINOv3 \cite{simoni2025dinov3}. Full covariance severely degrades DINOv3's predictive correlation.}
        \label{tab:img_cov_baseline}
        \resizebox{\textwidth}{!}{%
        \begin{tabular}{l cccccccc ccc}
        \toprule
         & \multicolumn{8}{c}{\textbf{Design choices}} & \multicolumn{3}{c}{\textbf{Mean Correlation Improvement (\%pt.)} $\downarrow$} \\ 
         \cmidrule(lr){2-9} \cmidrule(l){10-12} 
         & \textbf{Distance} & \textbf{Angle} & \textbf{Standardize} & \textbf{Sampling} & \textbf{Replace.} & \textbf{Space} & \textbf{Covariance} & \textbf{Algorithm} & \textbf{CLIP} & \textbf{DINOv3} & \textbf{Total} \\ 
        \midrule
        \multirow{4}{*}{\textbf{Covariance}}
         & \cmark & \cmark & \cmark & Mix      & \xmark & Euclidean & Diag & GMM + KL      & \textbf{-9.72} & \textbf{-9.87} & \textbf{-19.59} \\
         & \cmark & \cmark & \cmark & Mix      & \xmark & Euclidean & Full & GMM + KL      & -8.12          & -3.80          & -11.92 \\
         & \cmark & \cmark & \cmark & Mix      & \xmark & Euclidean & Tied & GMM + KL      & {\ul -8.46}    & -6.00          & -14.46 \\
         & \cmark & \cmark & \cmark & Mix      & \xmark & Euclidean & Spherical & GMM + KL & -7.98          & {\ul -8.47}    & {\ul -16.45} \\ 
        \bottomrule
        \end{tabular}%
        }
    \end{subtable}
    
    \vspace{1.5em} 
    
    \begin{subtable}[t]{\textwidth}
        \centering
        \caption{Ablation study on the text models Qwen3-Embedding \cite{zhang2025qwen3} and EmbeddingGemma \cite{vera2025embeddinggemma}. Spherical covariance completely collapses the metric for Qwen3-Embedding.}
        \label{tab:lang_cov_baseline}
        \resizebox{\textwidth}{!}{%
        \begin{tabular}{l cccccccc ccc}
        \toprule
         & \multicolumn{8}{c}{\textbf{Design choices}} & \multicolumn{3}{c}{\textbf{Mean Correlation Improvement (\%pt.)} $\downarrow$} \\ 
         \cmidrule(lr){2-9} \cmidrule(l){10-12} 
         & \textbf{Distance} & \textbf{Angle} & \textbf{Standardize} & \textbf{Sampling} & \textbf{Replace.} & \textbf{Space} & \textbf{Covariance} & \textbf{Algorithm} & \textbf{Qwen3-Embedding} & \textbf{EmbeddingGemma} & \textbf{Total} \\ 
        \midrule
        \multirow{4}{*}{\textbf{Covariance}}
         & \cmark & \cmark & \cmark & Mix      & \xmark & Euclidean & Diag      & GMM + KL & \textbf{-7.13} & \textbf{-30.58} & \textbf{-37.71} \\
         & \cmark & \cmark & \cmark & Mix      & \xmark & Euclidean & Full      & GMM + KL & -3.21          & {\ul-28.34}     & -31.55 \\
         & \cmark & \cmark & \cmark & Mix      & \xmark & Euclidean & Tied      & GMM + KL & {\ul-5.47}     & -26.68          & {\ul-32.15} \\
         & \cmark & \cmark & \cmark & Mix      & \xmark & Euclidean & Spherical & GMM + KL & 0.38           & -22.37          & -21.99 \\ 
        \bottomrule
        \end{tabular}%
        }
    \end{subtable}
\end{table}

Across both modalities, the Diagonal covariance structure strictly outperforms all alternative configurations, yielding the lowest MCI. 

In the vision domain, utilizing a Full covariance matrix severely degrades predictive performance for DINOv3, almost halving its correlation improvement. This degradation indicates that the Full covariance matrix likely leads to overfitting on the sampled geometric trajectories. 

Conversely, in the text domain, over-restricting the model proves equally detrimental. Applying a Spherical covariance matrix, which forces uniform variance across all dimensions, completely collapses RADAR's predictive power on the Qwen3-Embedding model, resulting in a degradation in MCI compared to the baseline. Ultimately, the Diagonal covariance strikes the optimal bias-variance tradeoff, with the Tied configuration emerging as the second-best option. The Diagonal covariance provides sufficient flexibility to capture the independent variances of angular and distance features without the parameter explosion and overfitting of a Full covariance matrix.

\subsection{Optimal Transport and Divergence Metrics}
\label{sec:ablation_optTrans}

Lastly, we ablate the algorithm used to measure the distributional shift between the source and target geometric representations. Table~\ref{tab:combined_algorithm} compares our continuous density estimation approach (GMM + KL) against standard discrete optimal transport algorithms (Sinkhorn, Sliced-Wasserstein Distance) and kernel-based metrics (Maximum Mean Discrepancy). 

\begin{table}[ht!]
    \centering
    \caption{Optimal transport algorithm ablations for vision (DomainNet) and text (EuroEval) representations. While standard algorithms like Sinkhorn perform reasonably well on text embeddings, GMM + KL is the only algorithm that demonstrates robust, generalized predictive correlation across both vision and text modalities.}
    \label{tab:combined_algorithm}
    
    \begin{subtable}[t]{\textwidth}
        \caption{Ablation study on the vision models CLIP \cite{radford2021learning} and DINOv3 \cite{simoni2025dinov3}. Standard optimal transport metrics (SWD, Sinkhorn) collapse, while MMD experiences catastrophic failure.}
        \label{tab:img_alg_baseline}
        \resizebox{\textwidth}{!}{%
        \begin{tabular}{l cccccccc ccc}
        \toprule
         & \multicolumn{8}{c}{\textbf{Design choices}} & \multicolumn{3}{c}{\textbf{Mean Correlation Improvement (\%pt.)} $\downarrow$} \\ 
         \cmidrule(lr){2-9} \cmidrule(l){10-12} 
         & \textbf{Distance} & \textbf{Angle} & \textbf{Standardize} & \textbf{Sampling} & \textbf{Replace.} & \textbf{Space} & \textbf{Covariance} & \textbf{Algorithm} & \textbf{CLIP} & \textbf{DINOv3} & \textbf{Total} \\ 
        \midrule
        \multirow{4}{*}{\textbf{Algorithm}}
         & \cmark & \cmark & \cmark & Mix      & \xmark & Euclidean & Diag & GMM + KL    & \textbf{-9.72} & \textbf{-9.87} & \textbf{-19.59} \\
         & \cmark & \cmark & \cmark & Mix      & \xmark & Euclidean & Diag & SWD         & 3.61           & {\ul -5.07}    & {\ul -1.46} \\
         & \cmark & \cmark & \cmark & Mix      & \xmark & Euclidean & --   & MMD (Gauss) & 62.32          & 61.45          & 123.77 \\
         & \cmark & \cmark & \cmark & Mix      & \xmark & Euclidean & --   & Sinkhorn    & {\ul 2.81}     & -4.24          & -1.43 \\
        \bottomrule
        \end{tabular}%
        }
    \end{subtable}
    
    \vspace{1.5em} 
    
    \begin{subtable}[t]{\textwidth}
        \centering
        \caption{Ablation study on the text models Qwen3-Embedding \cite{zhang2025qwen3} and EmbeddingGemma \cite{vera2025embeddinggemma}. While MMD and Sinkhorn show competitive performance for specific architectures, GMM + KL provides the best overall correlation.}
        \label{tab:lang_alg_baseline}
        \resizebox{\textwidth}{!}{%
        \begin{tabular}{l cccccccc ccc}
        \toprule
         & \multicolumn{8}{c}{\textbf{Design choices}} & \multicolumn{3}{c}{\textbf{Mean Correlation Improvement (\%pt.)} $\downarrow$} \\ 
         \cmidrule(lr){2-9} \cmidrule(l){10-12} 
         & \textbf{Distance} & \textbf{Angle} & \textbf{Standardize} & \textbf{Sampling} & \textbf{Replace.} & \textbf{Space} & \textbf{Covariance} & \textbf{Algorithm} & \textbf{Qwen3-Embedding} & \textbf{EmbeddingGemma} & \textbf{Total} \\ 
        \midrule
        \multirow{4}{*}{\textbf{Algorithm}}
         & \cmark & \cmark & \cmark & Mix      & \xmark & Euclidean & Diag & GMM + KL    & -7.13           & \textbf{-30.58} & \textbf{-37.71} \\
         & \cmark & \cmark & \cmark & Mix      & \xmark & Euclidean & Diag & SWD         & 6.07            & -19.05          & -12.98 \\
         & \cmark & \cmark & \cmark & Mix      & \xmark & Euclidean & --   & MMD (Gauss) & \textbf{-24.04} & -9.04           & -33.08 \\
         & \cmark & \cmark & \cmark & Mix      & \xmark & Euclidean & --   & Sinkhorn    & {\ul -13.86}    & {\ul -23.50}    & {\ul -37.36} \\ 
        \bottomrule
        \end{tabular}%
        }
    \end{subtable}
\end{table}

The data reveals a stark difference in how vision and text modalities respond to these algorithms. In the text domain, discrete and kernel-based methods perform remarkably well, with Sinkhorn and MMD achieving a highly competitive total MCI. However, in the vision domain, these same methods collapse. Sinkhorn actively degrades predictive correlation for CLIP, while MMD experiences a catastrophic failure across both vision architectures. This indicates that operating directly on the empirical point clouds of vision trajectories introduces excessive noise, causing point-to-point matching to fail. 

Fitting a GMM and computing the KL divergence emerges as the only universally robust algorithm. By mapping the discrete samples to a continuous probability density function, the GMM smooths out the empirical noise that derails other algorithms in vision models, while simultaneously capturing the precise structural alignments required to achieve optimal predictive correlation in text models. Consequently, \textit{GMM + KL} is strictly required to maintain a unified, modality-agnostic framework.

\subsection{Ablations on the robustness of hyperparameters}
\label{appendix:further_ablation}
To demonstrate the robustness of our method to various hyperparameter configurations, we conduct a comprehensive hyperparameter ablation study. Specifically, we analyze the impact of the number of sampled pairs $N$, the window radius $\ell$, and the sampling temperature $\tau$ on RADAR's overall performance.

\subsubsection{Ablation on the RADAR Sample Size}
\label{appendix:radar_sample}
In this subsection, we evaluate the stability of the RADAR metric across varying sub-sample sizes. As illustrated in Figure~\ref{fig:ablation_N_main}, we compare the metric's performance at lower sample counts against a high-fidelity baseline. Specifically, we report the Spearman rank correlation $\rho$ between RADAR scores computed with smaller sample sizes $N$ and those computed with a high-fidelity sample count of $N=65536$.

\begin{figure}[ht!]
  \centering
  \subfloat[CLIP]{
      \includegraphics[width=0.47\textwidth]{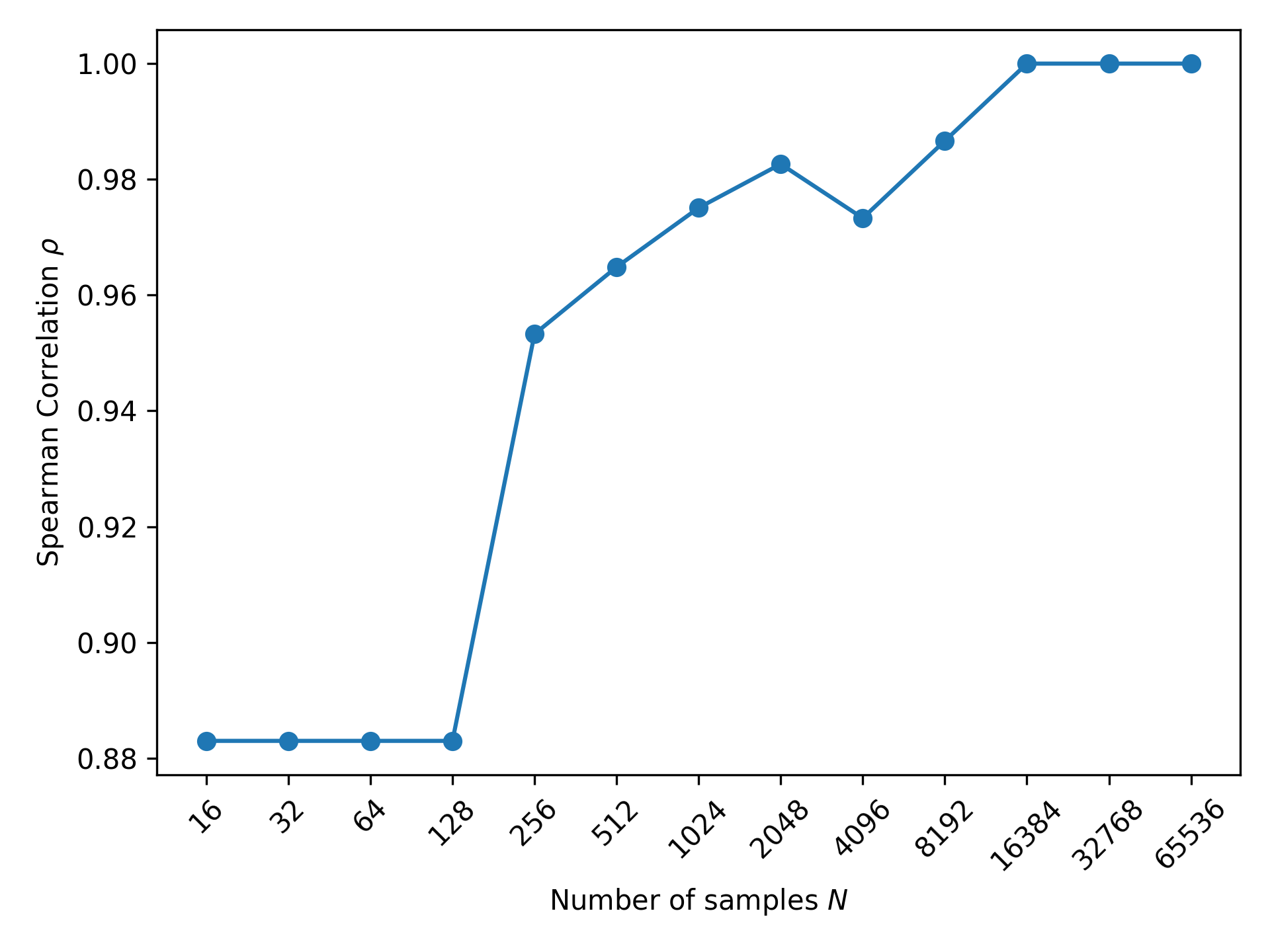}
      \label{fig:ablation_N_clip} 
  }  
  \hfill
  \subfloat[DINOv3]{
      \includegraphics[width=0.47\textwidth]{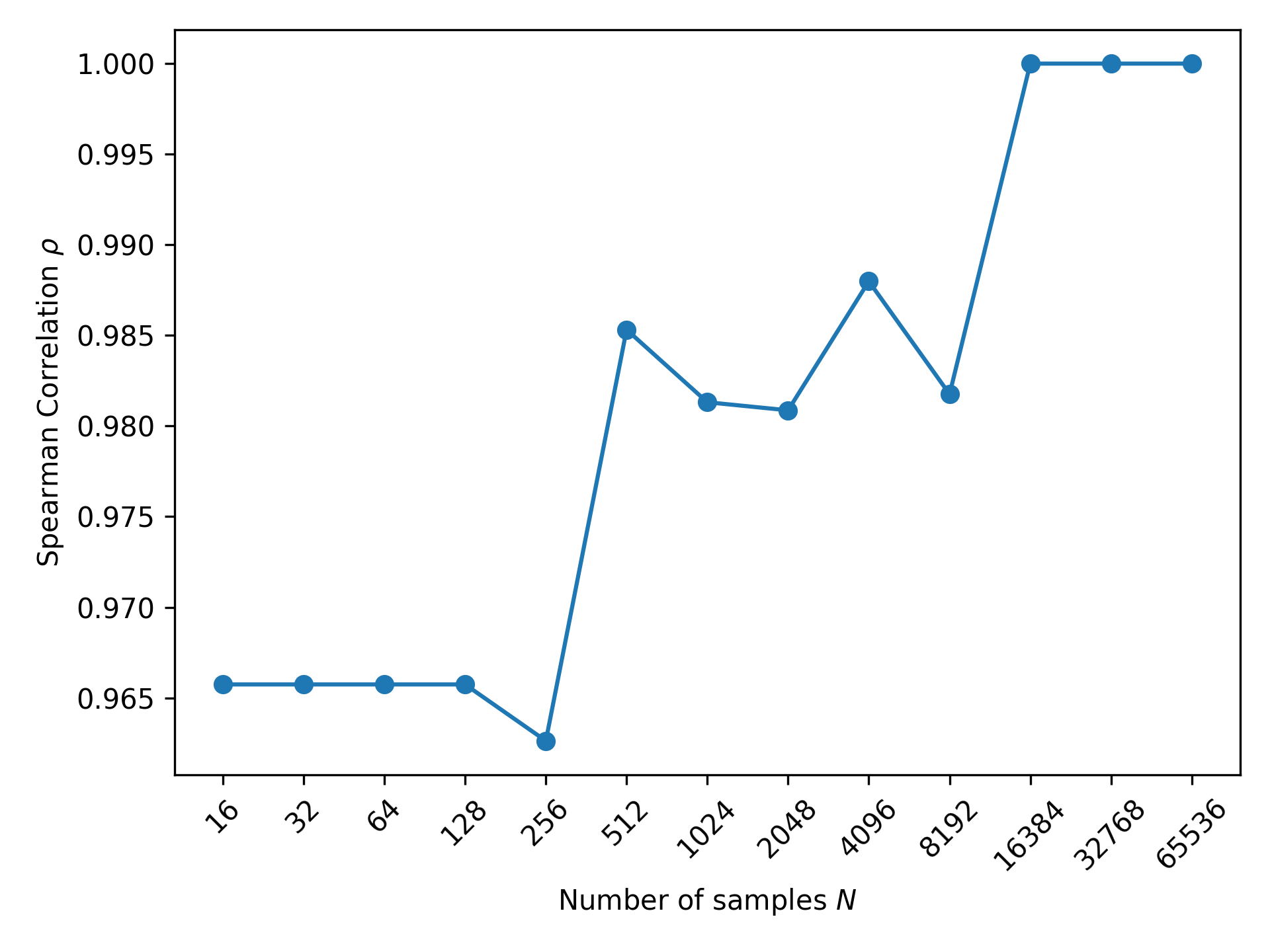}
      \label{fig:ablation_N_dino} 
  }
  \caption{\textbf{Effect of sample size ($N$) on RADAR metric stability.} The plots illustrate the Spearman rank correlation ($\rho$) of RADAR computed at various pair-wise sub-sample sizes ($N \in \{16,32,64,\dots,65536\}$) relative to a high-fidelity baseline computed with $N=65536$ on the DomainNet \cite{peng2019moment} dataset.}
  \label{fig:ablation_N_main} 
\end{figure}

We observe that for sample sizes $N \leq 128$, the CLIP model maintains a Spearman correlation exceeding 0.88, which increases to greater than 0.95 at $N=256$. For the DINOv3 model, RADAR demonstrates even greater sample efficiency, achieving a strong correlation of 0.965 with a drastically reduced sample size of just $N=16$. For both models, increasing $N$ further leads to a steady performance increase, ultimately plateauing at $N=16384$ where RADAR achieves a perfect correlation of 1.0 with the high-fidelity baseline of $N=65536$. For our experiments, we choose $N=32768$ to ensure evaluations fall safely within this stable regime.

\subsubsection{Ablation on the RADAR window size}
\label{appendix:radar_window}
Furthermore, we ablate the window size parameter $\ell$ to determine its impact on the metric's predictive stability. Figure~\ref{fig:ablation_ell_main} demonstrates the Spearman rank correlation across different window radii, measured against a baseline RADAR configuration computed at $\ell=6$.

\begin{figure}[ht!]
  \centering
  \subfloat[CLIP]{
      \includegraphics[width=0.47\textwidth]{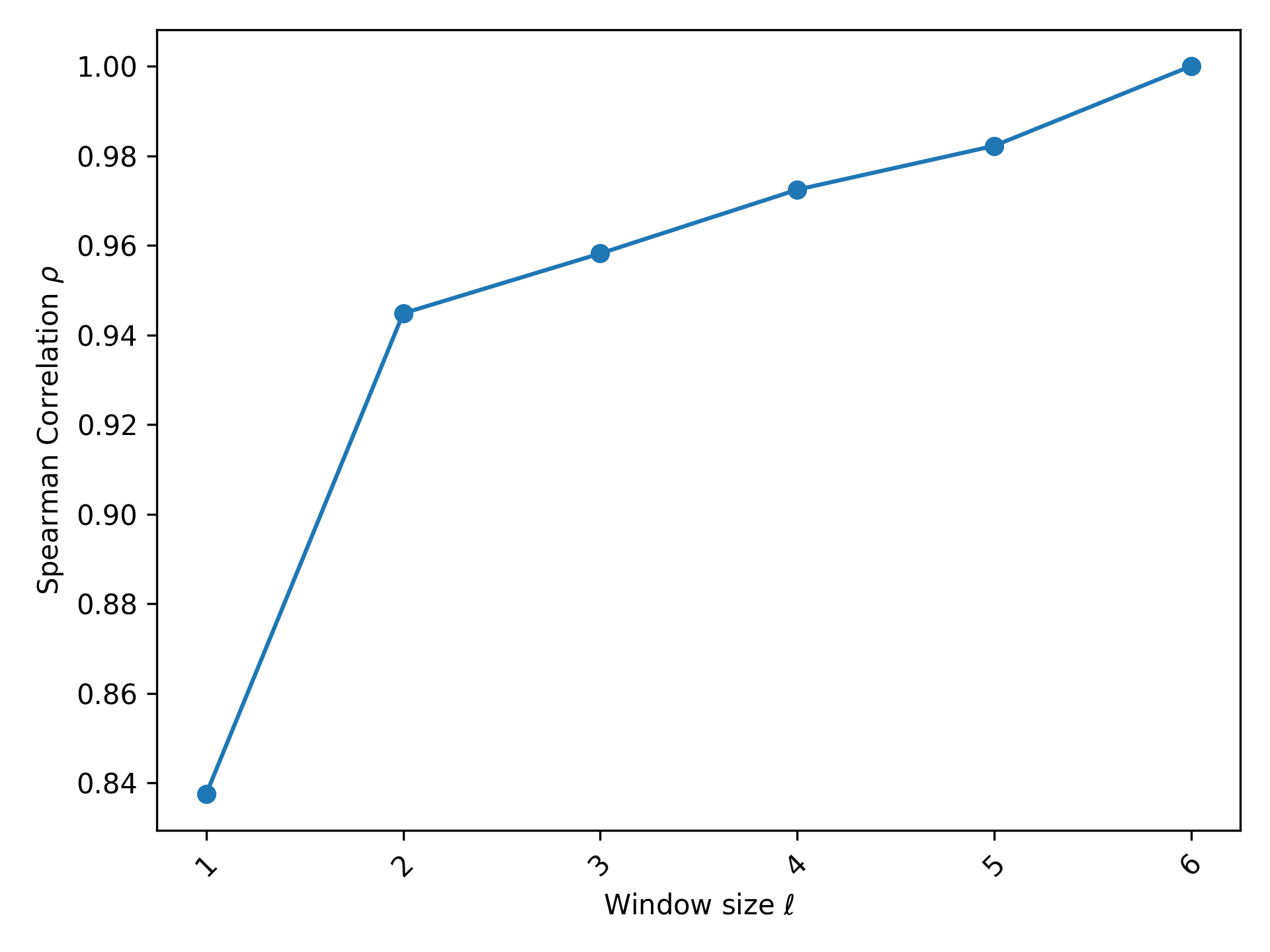}
      \label{fig:ablation_ell_clip} 
  }  
  \hfill
  \subfloat[DINOv3]{
      \includegraphics[width=0.47\textwidth]{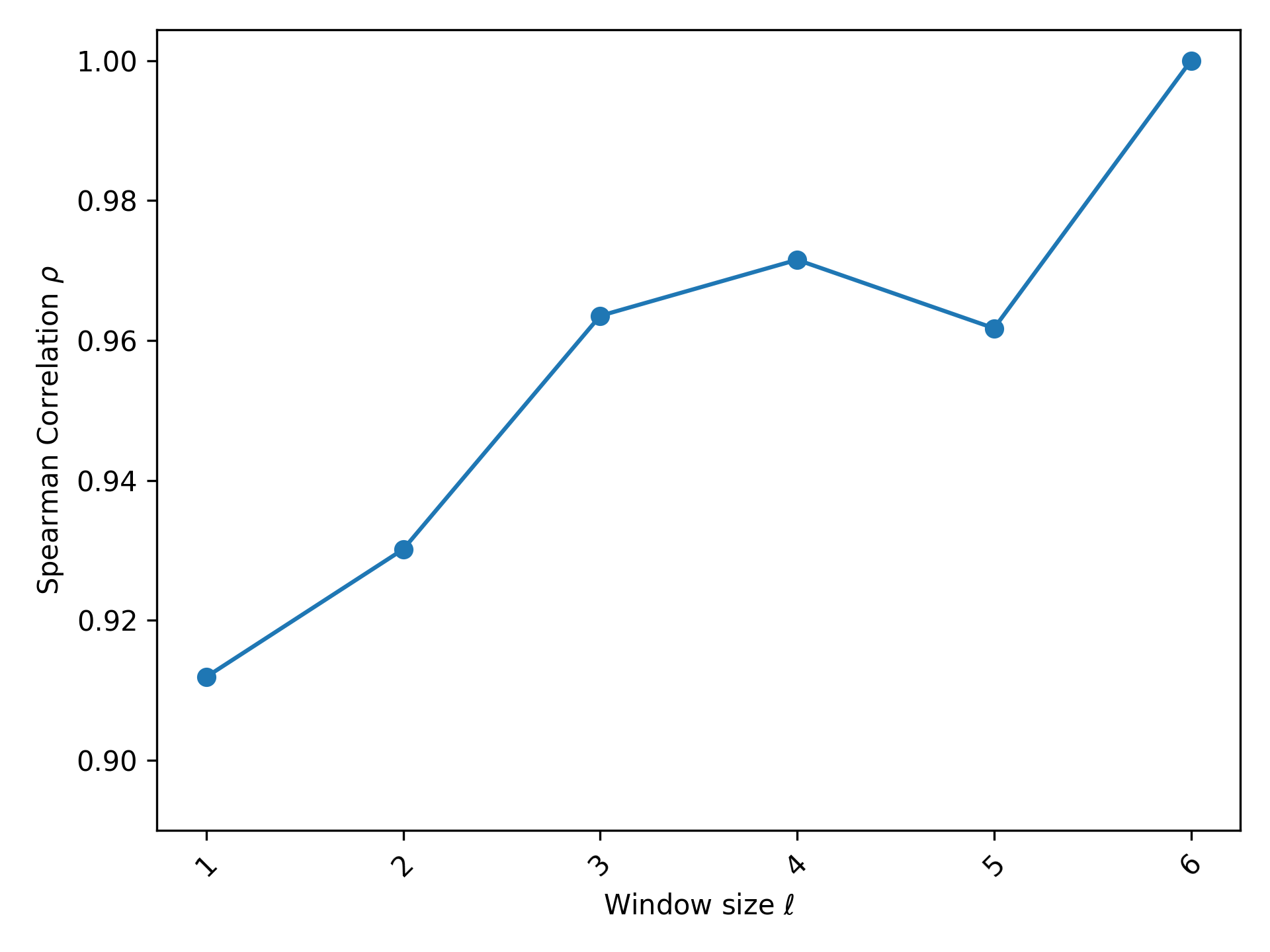}
      \label{fig:ablation_ell_dino} 
  }
  \caption{\textbf{Effect of the window size ($\ell \in \{1,2,3,4,5,6\}$) on RADAR metric stability.} The plots illustrate the Spearman rank correlation ($\rho$) of RADAR computed at various window sizes relative to a high-fidelity baseline computed with $\ell=6$ on the DomainNet \cite{peng2019moment} dataset.}
  \label{fig:ablation_ell_main} 
\end{figure}

Similar to the sample size $N$, we observe a strong correlation for both models even at the minimal window radius of $\ell=1$. The CLIP model yields a correlation of approximately 0.84, while DINOv3 starts at roughly 0.91. As $\ell$ increases, the CLIP model experiences a sharp jump to roughly 0.95 at $\ell=2$. Conversely, DINOv3 sees its steepest ascent at $\ell=3$, reaching a Spearman correlation of approximately 0.97. Overall, performance for both models rises steadily as $\ell$ grows, with CLIP achieving a 0.98 correlation at $\ell=5$. Although DINOv3 experiences a minor drop at $\ell=5$, the metric remains highly correlated overall, demonstrating that longer, dynamic trajectories provide increasingly stable transferability estimates. Therefore, we set $\ell=6$ as our default hyperparameter.

\subsubsection{Ablation on the RADAR weighting temperature}
\label{appendix:radar_temp}

Finally, we ablate the RADAR temperature parameter $\tau$, which controls the smoothness of the inlier-inlier and inlier-outlier sampling distributions, as detailed in Appendix~\ref{appendix:radar_params}. Figure~\ref{fig:ablation_tau_main} demonstrates the Mean Correlation Improvement (MCI) across different temperature settings achieved on the DomainNet \cite{peng2019moment} dataset.

\begin{figure}[ht!]
  \centering
  \subfloat[CLIP]{
      \includegraphics[width=0.47\textwidth]{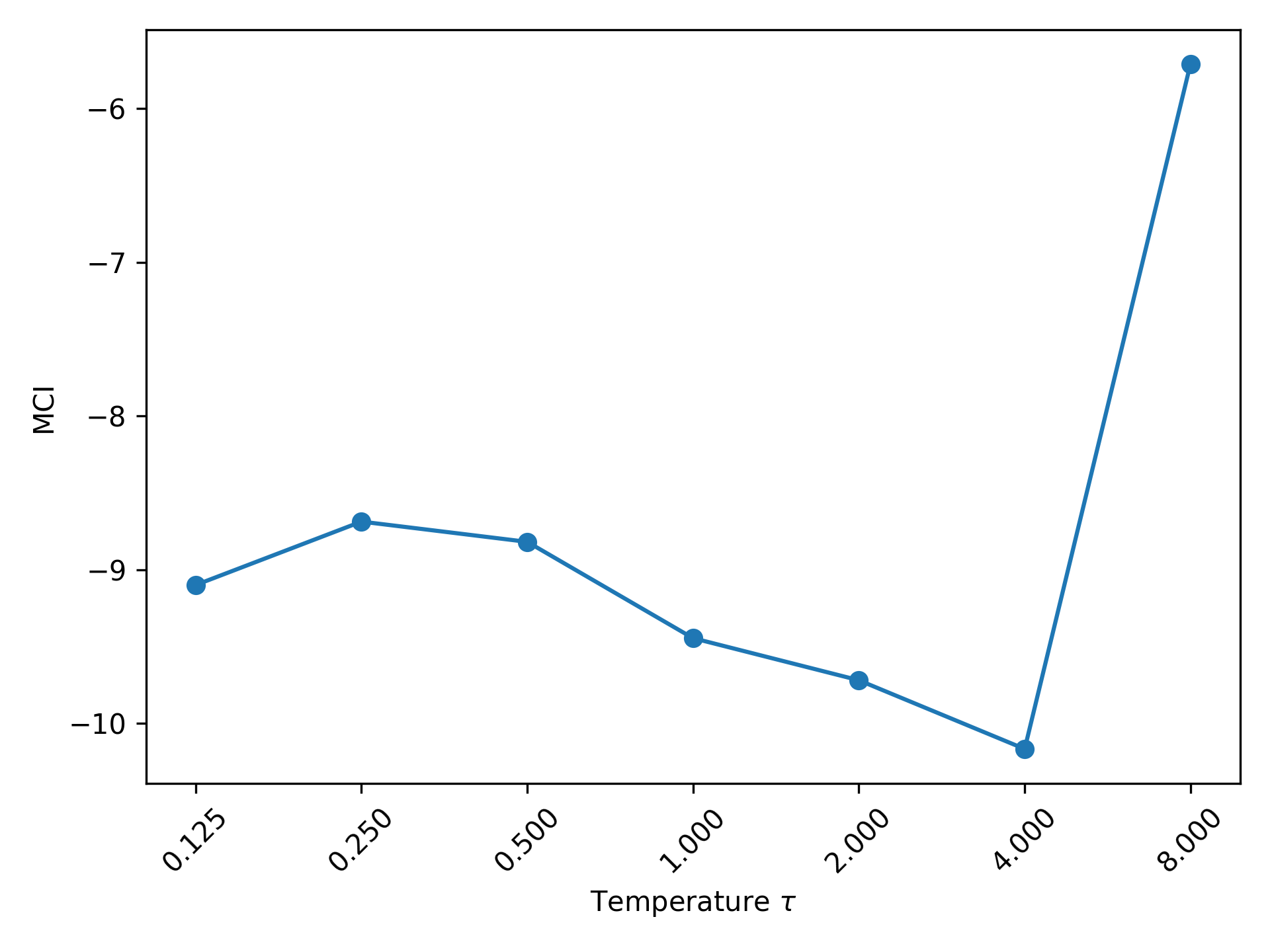}
      \label{fig:ablation_tau_clip} 
  }  
  \hfill
  \subfloat[DINOv3]{
      \includegraphics[width=0.47\textwidth]{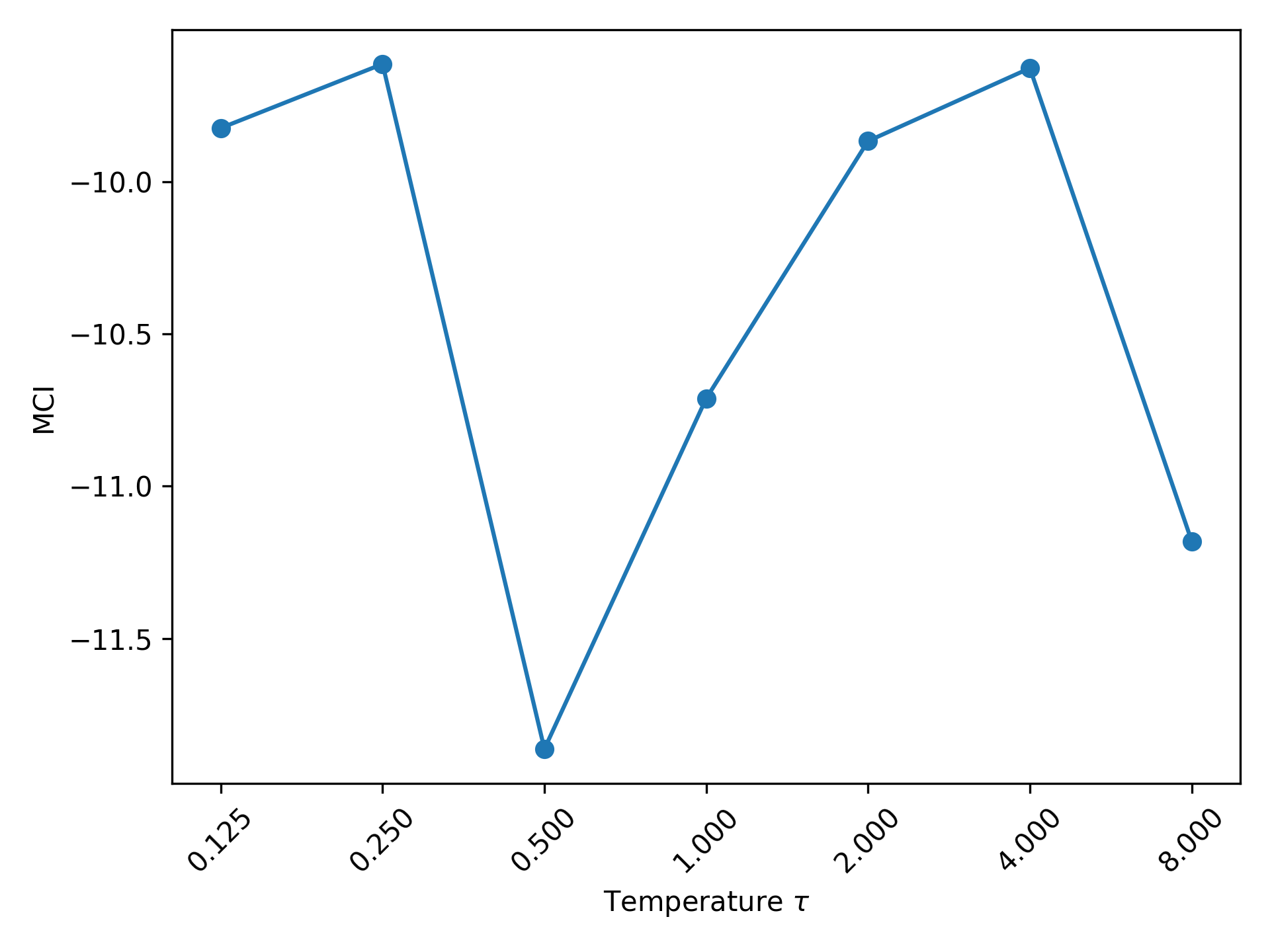}
      \label{fig:ablation_tau_dino} 
  }
  \caption{\textbf{Effect of the temperature $\tau$ ($\tau \in \{0.125,0.25,0.5,1,2,4,8\}$) on RADAR metric stability.} The plots illustrate the Mean Correlation Improvement (MCI) of the RADAR metric computed across various weighting temperatures on the DomainNet \cite{peng2019moment} dataset.}
  \label{fig:ablation_tau_main}
\end{figure}

For the CLIP model, the MCI exhibits a slight initial degradation as $\tau$ increases from 0.125 to 0.250, followed by a steady improvement as $\tau$ ranges from 0.5 to 4.0. However, we observe a catastrophic performance collapse at $\tau=8$, suggesting that oversmoothing the sampling distribution is highly detrimental. In contrast, the DINOv3 model remains remarkably robust across the 0.125--8.0 range, with a total MCI variation of only 2.0 percentage points (pp.). Given that CLIP exhibits a much higher sensitivity---with a more than 4.0 pp. spread over the same range---we prioritize its stability and select $\tau=2.0$ as our default configuration to avoid oversmoothing while maintaining peak performance across both architectures.

\section{Experiment parameters}
\label{appendix:experiment_params}

\subsection{RADAR parameters}
\label{appendix:radar_params}

As a first step of our metric, we apply layer-wise standardization to the geometric features. Specifically, we sample $2^{20}$ representational pairs to compute the baseline statistics, which are then used to normalize the angles and relative distances to a zero mean and unit variance. This standardization is critical as it ensures that the angular and distance components are evaluated on the same scale, preventing either feature from disproportionately dominating the Gaussian Mixture Model during density estimation.

Pairwise sampling is conducted using the weighted strategy detailed in Section~\ref{method:sampling}. To reduce the variance of the sampling weights and encourage broader exploration, we apply temperature scaling:
\begin{equation}
    w_{reg}(x) = \exp\left(-\frac{w(x)}{\tau}\right),
\end{equation}
where the temperature hyperparameter $\tau$ is set to $2.0$. Finally, to guarantee that every pair maintains a non-zero probability of selection, the resulting weights $w_{reg}(x)$ are strictly bounded within the interval $[0.1, 1.0]$.

For the density estimation step in RADAR, we employ a Gaussian Mixture Model (GMM) with strictly controlled parameters to ensure reproducibility. The number of mixture components, $K$, is dynamically configured as $K = \min(2C, 50)$, where $C$ is the number of unique classes in the target dataset. Allocating twice as many components as classes serves a critical regularization purpose: it provides the GMM with sufficient structural flexibility to model multi-modal distributions within individual classes, effectively capturing both inlier-inlier and inlier-outlier pairs for computing the angles and relative distances. Initialization is performed via the $k$-means algorithm with a single run and capped at 100 Expectation-Maximization (EM) iterations. To maintain computational stability and enforce positive semi-definite covariance matrices, we apply a diagonal regularization constant of $1 \times 10^{-3}$. Finally, the covariance type remains configurable per experiment, and the random state is fixed.

For the remaining RADAR hyperparameters, we sample $N=32768$ pairs and employ a window radius of $\ell=6$. This configuration provides sufficient resolution to comprehensively capture the representations' multi-layer geometric trajectories. Importantly, RADAR exhibits strong robustness to these specific choices. A detailed ablation study confirming this stability is provided in Appendix~\ref{appendix:further_ablation}.

\subsection{MLP probe parameters}
\label{appendix:mlp_params}
To evaluate the quality of the learned representations, we train a Multi-Layer Perceptron (MLP) probe. The probe architecture consists of an initial linear layer that projects the input representations from the model's hidden dimension to an intermediate size of 512, followed by a ReLU activation function. To mitigate overfitting, we apply dropout regularization with a probability of $p=0.1$. Finally, a second linear layer projects the intermediate features to the target number of classes.

The MLP probe is trained for a maximum of 100 epochs using a batch size of 256. We optimize the network using the AdamW optimizer with a learning rate of 0.01 and a weight decay of $1 \times 10^{-4}$. To mitigate overfitting, we employ early stopping with a patience of 10 epochs. To account for class imbalance during training, we apply a class-weighted Cross-Entropy loss. Specifically, the weight for each class $c$, denoted as $w_c$, is computed to be inversely proportional to its sample count $N_c$. To maintain the overall magnitude of the loss, these weights are normalized such that their sum equals the total number of classes $C$. Formally, the weight for class $c$ is defined as:
\begin{equation}
    w_c = C \frac{N_c^{-1}}{\sum_{j=1}^{C} N_j^{-1}}.
\end{equation}
To ensure the robustness of our evaluation and account for initialization variance, we train the MLP probes over $r=10$ distinct random seeds for each source-target pair. Across all runs, we observe rapid convergence, with early stopping consistently triggered before epoch 40.

\section{Datasets Used}
\label{appendix:datasets}

To evaluate the cross-modal applicability of RADAR, we utilize diverse datasets across both vision (three benchmarks) and text (two benchmarks). While EuroEval \cite{nielsen2025encoder} and PACS \cite{li2017deeperbroaderartierdomain} are used in their entirety, we strategically subsample OfficeHome \cite{venkateswara2017deep}, DomainNet \cite{peng2019moment}, and Amazon Reviews \cite{muennighoff2023mteb} to ensure tractable evaluation. Specifically, we restrict the size of DomainNet and OfficeHome by adopting sampling protocols standard in the domain adaptation literature \cite{zhang2022survey}. Furthermore, we uniformly downsample the Amazon Reviews dataset from its original 256,000 samples per language to create a balanced, computationally efficient subset. Finally, to explicitly assess the metric's resilience to natural corruptions, we incorporate a subset of the ImageNet-C \cite{hendrycks2019robustness} dataset evaluated at varying severity levels (1, 3, and 5).

\subsection{Image classification datasets}

\noindent \textbf{DomainNet \cite{peng2019moment} Dataset Details:}
\begin{itemize}
    \item \textbf{Dataset Splits:} Train: 15,874 \textbar{} Validation: 3,175 \textbar{} Test: 6,841
    \item \textbf{Domain Distribution (Train):} Clipart (1,889), Infograph (2,038), Painting (2,375), Quickdraw (3,500), Real (4,146), Sketch (1,926)
    \item \textbf{Class Distribution (Train):} Strawberry (1,524), Circle (1,172), Bread (1,609), Feather (1,620), Teapot (1,575), Whale (1,756), Windmill (1,649), Zebra (1,584), Tiger (1,732), Bird (1,653)
\end{itemize}

\noindent \textbf{PACS \cite{li2017deeperbroaderartierdomain} Dataset Details:}
\begin{itemize}
    \item \textbf{Dataset Splits:} Train: 8,977 \textbar{} Validation: 1,014 \textbar{} Test: 9,991
    \item \textbf{Domain Distribution (Train):} Art Painting (1,840), Cartoon (2,107), Photo (1,499), Sketch (3,531)
    \item \textbf{Class Distribution (Train):} Dog (1,555), Elephant (1,486), Giraffe (1,407), Guitar (1,000), Horse (1,384), House (846), Person (1,299)
\end{itemize}

\noindent \textbf{OfficeHome \cite{venkateswara2017deep} Dataset Details:}
\begin{itemize}
    \item \textbf{Dataset Splits:} Train: 6,031 \textbar{} Validation: 606 \textbar{} Test: 604
    \item \textbf{Domain Distribution (Train):} Art (1,034), Clipart (1,709), Product (1,720), Real World (1,568)
    \item \textbf{Class Distribution (Train):} Alarm Clock (299), Backpack (295), Bike (317), Bottle (338), Candles (330), Chair (363), Computer (303), Couch (268), Flipflops (270), Flowers (355), Folder (266), Helmet (298), Keyboard (291), Laptop (316), Monitor (310), Scissors (295), Sneakers (276), Speaker (290), TV (268), Telephone (283)
\end{itemize}

\subsection{Sentiment analysis datasets}

\noindent \textbf{EuroEval \cite{nielsen2025encoder} Dataset Details:}
\begin{itemize}
    \item \textbf{Dataset Splits:} Train: 8,029 \textbar{} Validation: 2,048 \textbar{} Test: 13,312
    \item \textbf{Domain Distribution (Train):} Czech (1,024), English (1,024), German (1,024), Hungarian (1,024), Italian (1,024), Polish (1,024), Slovak (1,024), Spanish (861)
    \item \textbf{Class Distribution (Train):} Negative (2,853), Neutral (2,360), Positive (2,816)
\end{itemize}

\noindent \textbf{Amazon Reviews \cite{muennighoff2023mteb} Dataset Details:}
\begin{itemize}
    \item \textbf{Dataset Splits:} Train: 196,596 \textbar{} Validation: 12,276 \textbar{} Test: 12,276
    \item \textbf{Domain Distribution (Train):} Chinese (32,766), English (32,766), French (32,766), German (32,766), Japanese (32,766), Spanish (32,766)
    \item \textbf{Class Distribution (Train):} Negative (65,532), Neutral (65,532), Positive (65,532)
\end{itemize}

\subsection{Synthetic datasets}
\noindent \textbf{Imagenet-C \cite{hendrycks2019robustness}}
\begin{itemize}
    \item \textbf{Dataset Splits:} Train: 16,140 \textbar{} Validation: 3,228 \textbar{} Test: 10,818
    \item \textbf{Domain Distribution (Train):} Brightness (2,699), Contrast (2,696), Fog (2,691), Motion Blur (2,687), Real (2,684), Snow (2,683)
    \item \textbf{Class Distribution (Train):} rock beauty (872), Mexican hairless (679), digital watch (800), Japanese spaniel (694), toy terrier (774), black-and-tan coonhound (658), English foxhound (678), space heater (903), otterhound (664), wire-haired fox terrier (879), Sealyham terrier (842), Petri dish (865), sweatshirt (837), hair spray (882), hatchet (801), throne (878), Irish water spaniel (872), power drill (817), letter opener (887), affenpinscher (858)
\end{itemize}

\section{Baseline Implementations}
\label{appendix:baselines}
We implement the majority of our baseline methods by directly adapting the codebase from the Transfer Learning Library \cite{jiang2022transferability, tllib}. For the s-OTDD baseline, we utilize the official implementation provided by the original authors, with minor modifications to interface properly with our standardized experimental framework. 

Additionally, for the computation of the Sinkhorn and MMD (Maximum Mean Discrepancy) distances, we leverage the highly optimized \texttt{geomloss} library \cite{feydy2019fast}.

\section{Detailed statistics --- raw Spearman rank correlation}
\label{appendix:raw_spearman}

Table~\ref{tab:appendix_raw} details the raw mean Spearman correlation ($\rho$) across the vision and text modalities, complementing the MCI results presented in the main text. As in our primary evaluation, we adjust the sign of each metric such that a \textbf{negative} Spearman rank correlation with target accuracies indicates optimal predictive performance.

\begin{table}[ht!]
    \centering
    \begin{subtable}[t]{0.4\textwidth}
        \centering 
        
        \subcaption{Results on vision benchmarks.}
        \resizebox{\textwidth}{!}{%
\begin{tabular}{lcccccc}
\toprule
 & \multicolumn{6}{c}{\textbf{Mean Spearman Correlation (\%pt.)} $\downarrow$} \\ \cmidrule(l){2-7} 
 & \multicolumn{2}{c}{\textbf{DomainNet} \cite{peng2019moment}} & \multicolumn{2}{c}{\textbf{OfficeHome} \cite{venkateswara2017deep}} & \multicolumn{2}{c}{\textbf{PACS} \cite{li2017deeperbroaderartierdomain}} \\ \cmidrule(lr){2-3} \cmidrule(lr){4-5} \cmidrule(l){6-7} 
\textbf{Method} & \textbf{CLIP} & \textbf{DINOv3} & \textbf{CLIP} & \textbf{DINOv3} & \textbf{CLIP} & \textbf{DINOv3} \\ 
\midrule
LEEP \cite{nguyen2020leep}              & {\ul-51.65}     & {\ul-31.16}     & 0.16            & -3.42           & -4.09           & 12.00 \\
H-Score \cite{bao2019information}       & -42.10          & -15.47          & 13.45           & 14.63           & 0.27            & {\ul-4.62} \\
Reg. H-Score \cite{ibrahim2022newer}    & -37.00          & -17.50          & 12.80           & 19.61           & -1.56           & -3.23 \\
LogME \cite{you2021logme}               & -30.86          & -14.70          & 0.27            & 11.33           & 0.59            & \textbf{-26.41} \\
NCE \cite{tran2019transferability}      & -47.91          & -30.46          &  -2.53          & -3.42           & -3.07           & 11.78 \\
TransRate \cite{huang2022frustratingly} & -40.73          & -21.90          & {\ul -7.10}     & {\ul -10.47}    & \textbf{-14.90} & 18.29 \\
S-OTDD \cite{nguyen2025lightspeed}      & -16.01          & -13.29          & 15.55           & \textbf{-32.03} & {\ul-13.37}     & 8.26 \\ 
\midrule
RADAR (OUR)                             & \textbf{-52.74} & \textbf{-44.47} & \textbf{-11.94} & -3.80           & 10.27           & 15.44 \\ 
\bottomrule
\end{tabular}%
}
        
        \label{tab:results_raw_vision} 
    \end{subtable}%
    \hfill 
    \begin{subtable}[t]{0.55\textwidth}
        \centering
        \subcaption{Results on text benchmarks.}
        \resizebox{\textwidth}{!}{%
\begin{tabular}{lcccc}
\toprule
 & \multicolumn{4}{c}{\textbf{Mean Spearman Correlation (\%pt.)} $\downarrow$} \\ \cmidrule(l){2-5} 
 & \multicolumn{2}{c}{\textbf{EuroEval} \cite{nielsen2025encoder}} & \multicolumn{2}{c}{\textbf{Amazon Reviews} \cite{muennighoff2023mteb}} \\ \cmidrule(lr){2-3} \cmidrule(l){4-5} 
\textbf{Method} & \textbf{Qwen3-Embedding} & \textbf{EmbeddingGemma} & \textbf{Qwen3-Embedding} & \textbf{EmbeddingGemma} \\ 
\midrule
LEEP \cite{nguyen2020leep}              & {\ul-4.61}     & -1.66           & -17.29          & {\ul-11.27} \\
H-Score \cite{bao2019information}       & 8.24           & 6.15            & -25.89          & -9.84 \\
Reg. H-Score \cite{ibrahim2022newer}    & 5.10           & 4.85            & {\ul -25.96}    & -9.83 \\
LogME \cite{you2021logme}               & 21.90          & 15.21           & 3.53            & 7.72 \\
NCE \cite{tran2019transferability}      & -1.05          & {\ul-3.72}      & \textbf{-26.64} & -7.47 \\
TransRate \cite{huang2022frustratingly} & 2.10           & -1.50           & -6.58           & 6.29 \\
S-OTDD \cite{nguyen2025lightspeed}      & 5.16           & 13.11           & -3.40           & 6.50 \\ 
\midrule
RADAR (OUR)                             & \textbf{-8.16} & \textbf{-36.18} & -15.83          & \textbf{-11.96} \\ 
\bottomrule
\end{tabular}%
}
        
        \label{tab:results_raw_lang} 
    \end{subtable}
    \caption{Raw Mean Spearman Rank correlation (\%pt., $\downarrow$). Bold indicates the best performance, underline indicates second-best performance.}
    \label{tab:appendix_raw} 
\end{table}

Interestingly, excluding PACS, we observe that RADAR is the only metric that successfully maintains a negative Spearman correlation across all evaluated vision and text datasets. While NCE and LEEP also achieve consistently negative rank correlations within the text domain experiments, RADAR demonstrates exceptionally strong predictive performance on specific benchmarks; most notably, it outperforms all baselines by a substantial margin on EuroEval.

\subsection{Statistical significance}
In this subsection, we report mean and standard deviation of RADAR's Spearman rank correlation across $r=10$ random seeds for both backbone architectures on DomainNet \cite{peng2019moment} and EuroEval \cite{nielsen2025encoder} in Table~\ref{tab:raw_var}. This directly quantifies the variance introduced by RADAR's stochastic pair sampling and confirms the stability of the reported results.

\begin{table}[ht!]
    \caption{Mean, standard deviation, minimum and maximum of RADAR's 
    Spearman rank correlation across $r=10$ subsampling seeds on 
    DomainNet \cite{peng2019moment} and EuroEval \cite{nielsen2025encoder}.}
    \label{tab:raw_var}
    \centering
    \resizebox{0.7\textwidth}{!}{%
        \begin{tabular}{llccccc}
        \toprule
        \textbf{Dataset} & \textbf{Backbone} & 
        \textbf{Mean $\rho$ (\%pt., $\downarrow$)} & &
        \textbf{Std Dev} & \textbf{Min} & \textbf{Max} \\
        \midrule
        \multirow{2}{*}{DomainNet} 
            & CLIP      & -51.75 & $\pm$ & 0.94 & -53.25 & -50.57 \\
            & DINOv3    & -44.60 & $\pm$ & 0.91 & -46.09 & -43.61 \\
        \midrule
        \multirow{2}{*}{EuroEval}  
            & Qwen3-Embedding   & -8.37   & $\pm$ & 0.64 & -9.62  & -7.52 \\
            & EmbeddingGemma    & -34.92  & $\pm$ & 1.55 & -36.80 & -32.29 \\
        \bottomrule
        \end{tabular}
    }
\end{table}

On DomainNet, RADAR exhibits strong stability with respect to the randomness of empirical pair sampling and GMM fitting. The standard deviations for the Spearman correlation ($\rho$) across all $r=10$ independent seeds are consistently low for both vision architectures: $\pm0.94$ percentage points for CLIP and $\pm0.91$ for DINOv3, with narrow empirical spreads of $[-53.25, -50.57]$ and $[-46.09, -43.61]$, respectively.

RADAR exhibits broadly stable behaviour across both text architectures 
on EuroEval, though with some variation between models. The standard deviations are comparable to those observed in the vision 
domain for Qwen3-Embedding ($\pm0.64$ percentage points), while 
EmbeddingGemma exhibits notably higher variance ($\pm1.55$ percentage 
points), reflected in its wider empirical spread 
of $[-36.88, -32.29]$ compared to $[-9.62, -7.52]$ for Qwen3-Embedding.

These results suggest that RADAR's geometric trajectory descriptor consistently produces stable divergence estimates across independent subsampling runs, and that the stratified pair sampling strategy introduced in Section~\ref{sec:method} provides sufficient coverage of the domain manifold without exhaustive enumeration.

\textit{Note:} The mean $\rho$ reported in Table~\ref{tab:raw_var} differs slightly from that reported in Table~\ref{tab:appendix_raw}, as the two  tables are computed over independent sets of subsampling seeds. Given the sampling variance of each model, this discrepancy is expected and reflects seed-to-seed variability rather than an inconsistency in the experimental setup. Nevertheless, the difference remains within 1 standard deviation reported in Table~\ref{tab:raw_var}.

\section{Broader Impacts}
\label{appendix:impact}
Our work aims to yield positive societal impacts by enabling more data-efficient training paradigms. By strategically filtering training data, our method significantly reduces the computational overhead --- and consequently, the environmental footprint, such as electricity and cooling water consumption --- associated with training large-scale models. However, we acknowledge the dual-use nature of our research. As RADAR can be used to predict negative transfer, malicious actors could theoretically invert our metric's objective to intentionally degrade model performance, induce catastrophic forgetting, or systematically amplify harmful biases within foundation models. Mitigating this risk would require access controls on the metric's deployment and awareness of adversarial inversion when applying transferability metrics in safety-critical settings.

\section{Computational cost analysis}
\label{appendix:cost_analysis}

In Table~\ref{tab:comp_cost}, we report the wall-clock time required to compute each metric across all domain pairs in the DomainNet \cite{peng2019moment} dataset. Because RADAR relies on GMMs---which currently lack highly optimized GPU implementations---and several baseline metrics require a model trained on source data, we evaluate all methods under both CPU-only and GPU-accelerated configurations. For metrics necessitating a source model, we include the training time of an auxiliary MLP classifier head to ensure a fair comparison. Baseline metrics are evaluated using the standardized implementations from the Transfer Learning Library \cite{jiang2022transferability, tllib}, with the exception of s-OTDD \cite{nguyen2025lightspeed}, for which we utilize the official implementation. Due to implementation-specific compatibility constraints, s-OTDD \cite{nguyen2025lightspeed} results are reported for CPU execution only. Regarding our own method, we omit the CPU timing for the Sinkhorn-based variant as its primary utility lies in GPU acceleration; its GPU performance is included to demonstrate the latent optimization potential of our metric.

\begin{table}[ht!]
\centering
\caption{Computational cost analysis comparing RADAR to baseline metrics. Time represents the wall-clock time in seconds ($\downarrow$) to evaluate all the domain pairs over the whole network, averaged over $r=10$ runs. Evaluated on an AMD EPYC 9124 CPU and a single NVIDIA RTX 6000 Ada Generation GPU.}
\label{tab:comp_cost}
\resizebox{\textwidth}{!}{
\begin{tabular}{lccccccc}
\toprule
\textbf{Hardware} & \textbf{LEEP} & \textbf{LogME} & \textbf{NCE} & \textbf{TransRate} & \textbf{s-OTDD} & \textbf{RADAR (Ours)} & \textbf{RADAR (Sinkhorn)} \\ 
\midrule
\textbf{CPU Only} & 894.62 & 2803.74 & 916.64          & 1193.73 & >3600 & \textbf{644.56} & --\\ 
\textbf{GPU}      & 208.01 & 1411.28 & \textbf{204.05} & 378.09  & N/A   & 454.20          & 253.53\\ 
\bottomrule
\end{tabular}
}
\end{table}

Remarkably, RADAR is the most efficient metric in CPU-only environments, outperforming all baselines including LEEP \cite{nguyen2020leep}, NCE \cite{tran2019transferability} and s-OTDD \cite{nguyen2025lightspeed}. While our approach involves training a GMM, strategic subsampling of the data allows us to bypass the longer training times associated with the classifier heads required by many traditional benchmarks. Although the layer-wise trajectory analysis of our vanilla method introduces a slight overhead on the GPU, the RADAR (Sinkhorn) variant significantly mitigates this, achieving runtimes comparable to the fastest statistical benchmarks. These results demonstrate that RADAR offers a superior balance of predictive power and practical scalability, with significant room for further hardware-specific acceleration.

\section{Data Processing Inequality and Multi-Layer Advantage}
\label{appendix:dpi}

A natural question is whether the multi-layer trajectory descriptor used by RADAR provides any advantage over a single-layer divergence computed at the network's final representation. The data processing inequality (DPI) for total variation gives a precise answer: divergences present at intermediate layers can only be \emph{suppressed}, never amplified, by subsequent deterministic layer maps. This justifies extracting features from intermediate transformer blocks rather than relying on the final-layer representation alone.

\begin{assumption}[Frozen constant-width transformer]
\label{ass:transformer}
Let $\Phi : \mathcal{X} \to \mathbb{R}^{H}$ be a frozen, pre-trained foundation model with $L$ residual blocks of constant ambient dimension $H$. The layer-$l$ feature map $h^{(l)} : \mathcal{X} \to \mathbb{R}^{H}$ satisfies $h^{(0)}(x) = E(x)$ and $h^{(l)}(x) = T_{l-1}(h^{(l-1)}(x))$, where each $T_{l} : \mathbb{R}^{H} \to \mathbb{R}^{H}$ is deterministic and measurable, and features are uniformly bounded: $\|h^{(l)}(x)\| \le B$ for all $l, x$. The pushforward feature distribution at layer $l$ for domain $D$ is $\mu_{D}^{(l)} := (h^{(l)})_{\#}\, P_{D}^{X}$.
\end{assumption}

\begin{proposition}[Layer-wise total variation is non-increasing]
\label{prop:dpi}
Under Assumption~\ref{ass:transformer}, for any two domains $A, B$ and any pair of layer indices $l \le l'$ within the network,
\begin{equation}
    d_{\mathrm{TV}}\!\left(\mu_{A}^{(l)},\, \mu_{B}^{(l)}\right)
\;\;\ge\;\;
d_{\mathrm{TV}}\!\left(\mu_{A}^{(l')},\, \mu_{B}^{(l')}\right).
\end{equation}
\end{proposition}

\begin{proof}
By Assumption~\ref{ass:transformer}, the composition
\begin{equation}
    T_{l:l'} \;:=\; T_{l'-1} \circ T_{l'-2} \circ \cdots \circ T_{l}
\;:\; \mathbb{R}^{H} \to \mathbb{R}^{H}
\end{equation}
is a deterministic measurable map satisfying $\mu_{D}^{(l')} = (T_{l:l'})_{\#}\, \mu_{D}^{(l)}$ for $D \in \{A, B\}$. The data processing inequality for total variation states that for any measurable map $f$ and any pair of probability measures $\mu, \nu$ on a common space,
\begin{equation}
    d_{\mathrm{TV}}(f_{\#}\mu,\, f_{\#}\nu) \;\le\; d_{\mathrm{TV}}(\mu, \nu).
\end{equation}
Applying this with $f = T_{l:l'}$, $\mu = \mu_{A}^{(l)}$, and $\nu = \mu_{B}^{(l)}$ yields the claim.
\end{proof}

\begin{corollary}[Final-layer representations underestimate domain divergence]
\label{cor:multi-layer}
For any intermediate layer $l < L$,
\begin{equation}
    d_{\mathrm{TV}}\!\left(\mu_{A}^{(l)},\, \mu_{B}^{(l)}\right)
\;\;\ge\;\;
d_{\mathrm{TV}}\!\left(\mu_{A}^{(L)},\, \mu_{B}^{(L)}\right).
\end{equation}
Consequently, any transferability metric that operates exclusively on single-layer features (e.g., centroid distance, LogME, TransRate computed at the network output) systematically underestimates the domain divergence that may be visible at any earlier layer.
\end{corollary}

\begin{remark}[Implications for transferability estimation]
\label{rem:multi-layer}
Proposition~\ref{prop:dpi} formalizes an observation made empirically by recent work on foundation-model representations: the most discriminative features for downstream tasks are not always located at the network's output~\cite{bolya2025perception}. Final-layer normalization and the inductive bias toward task-aligned features can suppress domain-specific structure that remains observable at intermediate depths.

This has two consequences for the design of transferability metrics:
\begin{enumerate}
    \item \textbf{Single-layer baselines are biased.} Metrics that compare only final-layer feature distributions (such as $l_2$ centroid distance or s-OTDD) inherit a one-sided bias: they may report small divergence even when the underlying domains are meaningfully distinct at earlier layers.
    \item \textbf{Multi-layer extraction is justified.} Multi-layer extraction captures the full divergence profile across depth, recovering divergence signals that are monotonically suppressed as depth increases and would be invisible to any single-layer metric.
\end{enumerate}
The trajectory description leveraged by RADAR utilizes this insight by integrating geometric information across a window of radius $\ell=6$ centered at $l$ (from $l-6 \rightarrow \cdots \rightarrow l \rightarrow \cdots \rightarrow l+6$). The empirical justification for the specific window size $\ell = 6$ is provided by the ablation in Appendix~\ref{appendix:radar_window} and Figure~\ref{fig:ablation_ell_main}, which demonstrates that Spearman correlation increases monotonically with $\ell$ across both vision architectures.
\end{remark}

\begin{remark}[Scope of the inequality]
\label{rem:dpi-scope}
Proposition~\ref{prop:dpi} bounds total variation, not the full RADAR KL divergence, and concerns marginal feature distributions at single layers, not the joint trajectory law. It establishes that single-layer final-only metrics are biased, but does not by itself establish that the specific $(\theta^{(l)}, d^{(l)})$ trajectory descriptor is optimal among multi-layer alternatives. The empirical case for the geometric descriptor over alternative multi-layer summaries is made through the ablations in Section~\ref{section:ablation}.
\end{remark}

\section{Proof of the Displacement Triangle}
\label{appendix:disp_triangle}

\begin{proposition}[Triangle closure]
\label{prop:triangle-closure}
For all $x, x' \in \mathcal{X}$ and all $l \in \{0, 1, \ldots, L-1\}$,
\begin{equation}
    v_{\mathrm{sep}}^{(l)}(x, x') \;+\; v_{\mathrm{detour}}^{(l)}(x, x') \;=\; v_{\mathrm{traj}}^{(l)}(x).
\end{equation}
\end{proposition}

\begin{proof}
Direct expansion of Definitions~(\ref{eq:sep})--(\ref{eq:traj}):
\begin{align}
v_{\mathrm{sep}}^{(l)}(x, x') + v_{\mathrm{detour}}^{(l)}(x, x')
&= \bigl[h^{(l)}(x') - h^{(l)}(x)\bigr] + \bigl[h^{(l+1)}(x) - h^{(l)}(x')\bigr] \\
&= h^{(l+1)}(x) - h^{(l)}(x) \\
&= v_{\mathrm{traj}}^{(l)}(x).
\end{align}
The intermediate term $h^{(l)}(x')$ cancels, leaving the direct displacement from layer $l$ to layer $l+1$ at anchor $x$.
\end{proof}

\section{Base models}
\label{appendix:base_models}
We describe the exact base model configurations used for feature extraction and our main results. 

For image classification:
\begin{itemize}
    \item \textbf{DINOv3 \cite{simoni2025dinov3}}: facebook/dinov3-vits16-pretrain-lvd1689m \cite{dino_url}
    \item \textbf{CLIP \cite{radford2021learning}}: openai/clip-vit-base-patch32 \cite{clip_url}
\end{itemize}

For sentiment analysis:
\begin{itemize}
    \item \textbf{Qwen3-Embedding \cite{zhang2025qwen3}}: Qwen/Qwen3-Embedding-0.6B \cite{qwen_url}
    \item \textbf{EmbeddingGemma \cite{vera2025embeddinggemma}}: google/embeddinggemma-300m \cite{google_url}
\end{itemize}

\section{Declaration of LLM usage}
\label{appendix:llm_usage}
We declare that LLMs were used for spell-checking, editing, formatting, and paraphrasing. The proofs in Appendix~\ref{appendix:dpi} were additionally proposed and reviewed with LLM assistance and independently verified by the authors. LLMs were also used for programming aid, providing initial, skeleton implementations that were subsequently refined by the authors.

\end{document}